%% file: main-aistats.tex
\documentclass[twoside]{article}

\input{2024-AISTATS/header/aistats-header.tex}

\input{2024-AISTATS/header/aistats-def.tex}

\input{2024-AISTATS/header/aistats-letters.tex}

\usepackage[accepted]{aistats2025}
\usepackage[round]{natbib}

\begin{document}

\runningtitle{Refined Analysis of Federated Averaging and Federated Richardson-Romberg}

\runningauthor{P. Mangold, A. Durmus, A. Dieuleveut, S. Samsonov, E. Moulines}

\twocolumn[

\aistatstitle{Refined Analysis of Constant Step Size Federated Averaging and Federated Richardson-Romberg Extrapolation}

\aistatsauthor{Paul Mangold \textsuperscript{1} \And %
Alain Durmus \textsuperscript{1} \And %
Aymeric Dieuleveut \textsuperscript{1}}
\aistatsauthor{Sergey Samsonov \textsuperscript{2} \And %
Eric Moulines \textsuperscript{1,3}}
\aistatsaddress{
\textsuperscript{1} CMAP, CNRS, École polytechnique, Institut Polytechnique de Paris, 91120
Palaiseau, France \\ 
\textsuperscript{2} HSE University, Russia 
\qquad \qquad \qquad \qquad
\textsuperscript{3} MBZUAI
} 
]

\begin{abstract}
In this paper, we present a novel analysis of \FedAvg with constant step size, relying on the Markov property of the underlying process. We demonstrate that the global iterates of the algorithm converge to a stationary distribution and analyze its resulting bias and variance relative to the problem's solution.
We provide a first-order bias expansion in both homogeneous and heterogeneous settings. Interestingly, this bias decomposes into two distinct components: one that depends solely on stochastic gradient noise and another on client heterogeneity.
Finally, we introduce a new algorithm based on the Richardson-Romberg extrapolation technique to mitigate this bias.
\end{abstract}

\section{INTRODUCTION}
\label{sec:intro}
\input{2024-AISTATS/src/aistats-introduction}

\section{PRELIMINARIES}
\label{sec:prelim}
\input{2024-AISTATS/src/aistats-assumptions}

\section{RELATED WORK}
\label{sec:related-work}
\input{2024-AISTATS/src/aistats-related-work}

\section{DETERMINISTIC FEDAVG}
\label{sec:deterministic-fedavg}
\input{2024-AISTATS/src/aistats-deterministic}

\section{STOCHASTIC FEDAVG}
\label{sec:stochastic-fedavg}
\input{2024-AISTATS/src/aistats-stochastic}

\section{RICHARDSON-ROMBERG FOR FEDERATED AVERAGING}
\label{sec:richardson-romberg}
\input{2024-AISTATS/src/aistats-richardson}

\section{NUMERICAL EXPERIMENTS}
\label{sec:expe}
\input{2024-AISTATS/src/aistats-expe}

\section{CONCLUSION}
\label{sec:conclu}
\input{2024-AISTATS/src/aistats-conclusion}

\section*{ACKNOWLEDGEMENTS}
The work of P. Mangold has been supported by Technology Innovation Institute (TII), project Fed2Learn. The work of Aymeric Dieuleveut is supported by Hi!Paris FLAG chair, and this work has benefited from French State aid managed by the Agence Nationale de la Recherche (ANR) under France 2030 program with the reference ANR-23-PEIA-005 (REDEEM project).
This work is supported by Hi! PARIS and ANR/France 2030 program (ANR-23-IACL-0005)
The work of E. Moulines has been partly funded by the European Union (ERC-2022-SYG-OCEAN-101071601). Views and opinions expressed are however those of the author(s) only and do not necessarily reflect those of the European Union or the European Research Council Executive Agency. Neither the European Union nor the granting authority can be held responsible for them.
The work of S. Samsonov was prepared within the framework of the HSE University Basic Research Program.

\renewcommand{\refname}{REFERENCES}
\bibliographystyle{plainnat}
\bibliography{references.bib}

\clearpage 

\appendix
\onecolumn

\aistatstitle{Supplementary Materials}

\section{Refined Analysis of \FedAvg}
\label{sec:refined-fedavg-det}

\input{2024-AISTATS/src/aistats-app-analysis-fedavg-det}

\section{Analysis of Stochastic \FedAvg}
\input{2024-AISTATS/src/aistats-app-analysis-fedavg-sto}

\section{Analysis of Federated Richardson-Romberg Extrapolation}
\input{2024-AISTATS/src/aistats-app-richardson-romberg}

\section{Technical Lemma on Matrix Products}
\input{2024-AISTATS/src/aistats-technical-lemmas}

\end{document}

%% file: 2024-AISTATS/header/aistats-header.tex
\usepackage{amsmath}
\usepackage{amssymb}
\usepackage{amsthm}

\usepackage{xargs}
\usepackage{xspace}
\usepackage{xparse}

\usepackage{xcolor}

\usepackage[textwidth=2cm,textsize=tiny]{todonotes}

\usepackage{booktabs}
\usepackage[flushleft]{threeparttable}

\usepackage{pifont}

\usepackage{xfrac}

\usepackage{subcaption}

\usepackage{algorithm}
\usepackage{algorithmic}

\usepackage{comment}

\usepackage{thmtools}
\usepackage{thm-restate}

\usepackage{regexpatch}
\makeatletter
\xpatchcmd\thmt@restatable{%
\csname #2\@xa\endcsname\ifx\@nx#1\@nx\else[{#1}]\fi
}{%
\ifthmt@thisistheone
\csname #2\@xa\endcsname\ifx\@nx#1\@nx\else[{#1}]\fi
\else
\csname #2\@xa\endcsname[{Restated}]
\fi}{}{}
\makeatother

\definecolor{amaranth}{rgb}{0.8, 0, 0.34}
\usepackage[colorlinks=true,allcolors=amaranth]{hyperref}

\usepackage{cleveref}

\usepackage{enumitem}

\usepackage{tabularx} %
\usepackage{array} %

\newcolumntype{C}{>{\centering\arraybackslash}X}

%% file: 2024-AISTATS/header/aistats-def.tex
\def\FedAvg{\textnormal{\textsc{FedAvg}}\xspace}
\def\FedAvgBold{{\bfseries\textsc{FedAvg}}\xspace}
\def\FedAvgwDG{\textnormal{\textsc{FedAvg-D}}\xspace}

\def\Scaffold{\textsc{Scaffold}\xspace}
\def\SGD{\textsc{SGD}\xspace}

\def\titan
\usepackage{lmodern}%
\usepackage{slantsc}%

\def\step{\gamma}
\newcommandx{\mhstep}[1]{\gamma_{#1}}

\def\nrounds{T}
\def\nagent{N}
\def\nlupdates{H}

\newcommandx{\restew}[2][1=c]{\mathcal{R}^{#1}_{#2}}
\newcommandx{\resteintw}[2][1=c]{\textsf{r}^{#1}_{#2}}
\newcommandx{\reste}[3][1=c]{\mathcal{R}^{#1}_{#2}(#3)}
\newcommandx{\resteint}[3][1=c]{\resteintw[#1]{#2}(#3)}
\def\Msmoothcstvar{\tau}
\def\Msmoothcstvarbis{\widetilde{\tau}}
\def\smoothcstvar{k}
\def\invopcov{\mathbf{A}}

\def\paramw{\theta}

\def\varparamw{\vartheta}
\newcommandx{\param}[1][1=]{\paramw_{#1}}
\newcommandx{\varparam}[1][1=]{\varparamw}
\newcommandx{\sparam}[1][1=]{\tilde{\paramw}}
\newcommandx{\svarparam}[1][1=]{\tilde{\varparamw}}
\def\paramlim{\theta^{\star}}
\def\cvar{\xi}

\newcommandx{\randState}[2][1=,2=]{Z_{#1}^{#2}}
\newcommandx{\randStatebis}[2][1=,2=]{Y_{#1}^{#2}}
\newcommandx{\globRandState}[1][1=]{Z_{#1}}

\def\markovkernel{\kappa}
\newcommandx{\brset}[1]{\mcB(\rset^{#1})}
\def\covfunc{\mathcal{C}}

\def\strcvx{\mu}
\def\lip{L}
\def\explip{L}

\newcommandx{\pseudogradw}[1][1=]{\msg^{(#1)}}
\newcommandx{\pseudograd}[2][1=]{\pseudogradw[#1](#2)}
\newcommandx{\spseudogradw}[2][1=]{\msg_{#2}^{(#1)}}
\newcommandx{\spseudograd}[3][1=]{\spseudogradw[#1]{#3}(#2)}

\newcommandx{\pseudogradws}[1][1=]{\msG^{(#1)}}
\newcommandx{\pseudograds}[3][1=]{\pseudogradws[#1](#2; #3)}

\newcommandx{\dfedavgop}[1][1=]{\mathsf{T}^{(#1)}}
\newcommandx{\dfedavgopc}[1][1=]{\mathsf{T}^{(#1)}_c}
\newcommandx{\sdfedavgop}[2][1=]{\mathsf{\widetilde T}_{#1}^{(#2)}}
\newcommandx{\hsdfedavgop}[2][1=]{\mathsf{\widetilde V}_{#1}^{(#2)}}

\newcommand{\statdist}[1]{\pi^{(#1)}}
\newcommand{\statdistlim}[1]{\bar{\paramw}_{\textnormal{sto}}^{(#1)}}
\newcommand{\statdistlimv}[1]{\bar{\vartheta}_{\textnormal{sto}}^{(#1)}}
\def\biasHetero{\mathrm{b}_{\mathrm{h}}}
\def\biasSto{\mathrm{b}_{\mathrm{s}}}
\newcommand{\covstatdistlim}[1]{\bar{\boldsymbol{\Sigma}}_{\textnormal{sto}}^{(#1)}}
\newcommand{\detlim}[1]{\bar{\paramw}_{\textnormal{det}}^{(#1)}}

\newcommandx{\heterboundw}[1][1=1]{\zeta_{\star,#1}}
\def\heterboundgrad{\heterboundw[1]}
\def\heterbound{\heterboundw[2]}

\newcommandx{\hetcof}[1][1=]{\chi_\textnormal{het}^{#1}}

\def\fw{f}
\newcommandx{\nfw}[1][1=c]{f_{#1}}
\newcommandx{\nFw}[1][1=c]{F_{#1}}
\newcommandx{\nFww}[2][1=c]{\nFw[#1]^{#2}}
\newcommandx{\nf}[2][1=c]{\nfw[#1](#2)}
\newcommandx{\nfs}[3][1=c]{\nFw[#1]^{#3}(#2)}

\def\gfww{\nabla \fw}
\newcommandx{\gfw}[1]{\nabla \fw(#1)}
\newcommandx{\gnfw}[1][1=c]{\nabla \nfw[#1]}
\newcommandx{\gnf}[2][1=c]{\nabla \nfw[#1](#2)}

\newcommandx{\gnfs}[3][1=c]{\nabla \nFw[#1]^{#3}(#2)}
\newcommandx{\gnfws}[2][1=c]{\nabla \nFw[#1]^{#2}}
\newcommandx{\gfs}[3][1=c]{\nabla \fw[#1]^{#3}(#2)}

\newcommandx{\hnfw}[1][1=c]{\nabla^2 \nfw[#1]}
\newcommandx{\hnf}[2][1=c]{\nabla^2 \nfw[#1](#2)}
\newcommandx{\hf}[1]{\nabla^2 \fw(#1)}
\newcommandx{\hnfs}[3][1=c]{\nabla^2 \nFw[#1]^{#3}(#2)}

\newcommandx{\hhnfw}[1][1=c]{\nabla^3 \nfw[#1]}
\newcommandx{\hhnf}[2][1=c]{\nabla^3 \nfw[#1](#2)}
\newcommandx{\hhf}[1]{\nabla^3 \fw(#1)}

\newcommandx{\remainder}[2][1=c]{\bar{D}_{(#1)}(#2)}
\newcommandx{\remainders}[3][1=c]{D^{(#1)}_{#3}(#2)}
\newcommandx{\avgremainder}[1]{\bar{D}(#1)}
\newcommandx{\avgremainders}[2]{D_{#2}(#1)}

\newcommandx{\inthess}[2][1=c]{\bar{D}_{#1}^{(#2)}}
\newcommandx{\inthesss}[3][1=c]{D^{(#1)}_{#3}_{#2}}

\newcommandx{\contract}[2][1=c]{\Gamma_{#1}^{#2}}
\newcommandx{\fullcontract}[1]{\Gamma^{#1}}
\newcommandx{\tcontract}[2][1=c]{\widetilde{\Gamma}_{#1}^{#2}}
\newcommandx{\accmat}[2][1=c]{C^{(#1)}_{#2}}
\newcommandx{\taccmat}[2][1=c]{\widetilde{C}^{(#1)}_{#2}}
\newcommandx{\updatefuncnoise}[2][1=c]{\varepsilon_{#1}^{#2}}
\newcommandx{\roundfuncnoise}[2][1=c]{\mathcal{E}^{(#1)}_{#2}}
\newcommandx{\fullfuncnoise}[1]{\mathcal{E}_{#1}}
\newcommandx{\funcnoise}[2][2=c]{\omega^{#2}(#1)}
\newcommandx{\locfuncnoise}[2][2=c]{\varepsilon^{#2}(#1)}

\newcommandx{\roundbias}[1]{b_\biasSto{#1}}
\newcommandx{\fullbias}[1]{\mcB_{#1}}
\newcommandx{\fullfluct}[1]{\mcF \!\ell_{#1}}

\def\detcontractw{\digamma}
\newcommandx{\detcontract}[2][1=c]{\detcontractw_{#1}^{#2}}
\newcommandx{\detfullcontract}[1]{\detcontractw^{#1}}

\newcommandx{\flparam}[2][2=]{\tilde{\param}_{#1}^{#2}}
\newcommandx{\trparam}[1]{\hat{\param}_{#1}}
\newcommandx{\flcvar}[2][1=c]{\widetilde{\cvar}^{(#1)}_{#2}}
\newcommandx{\trcvar}[2][1=c]{\hat{\cvar}^{(#1)}_{#2}}
\newcommandx{\locflparam}[2][1=c]{\widetilde{\param}^{(#1)}_{#2}}
\newcommandx{\loctrparam}[2][1=c]{\hat{\param}^{(#1)}_{#2}}

\newcommandx{\flpp}[1]{\Delta^{\param,\param}_{#1}}
\newcommandx{\flpc}[2][1=c]{\Delta^{\param,\cvar^{(#1)}}_{#2}}
\newcommandx{\flcp}[2][1=c]{\Delta^{\cvar^{(#1)},\param}_{#2}}
\newcommandx{\flcc}[3][1=c,2=c]{\Delta^{\cvar^{(#1)},\cvar^{(#2)}}_{#3}}

\def\noiseVec{\mathcal{E}}

\def\oneVec{\mathbf{1}}

\newcommandx{\VarNoise}[1]{\Sigma_{\noiseVec}^{(#1)}}

\def\gennoiseround{\varepsilon}

\newcommandx{\noiseround}[2][1=]{\gennoiseround^{(#1)}_{#2}}

\newcommandx{\nfuncbw}[1][1=]{\funcbw^{{#1}}}
\newcommandx{\nfuncAw}[1][1=]{\funcAw^{{#1}}}
\newcommandx{\nfuncb}[2][1=]{\nfuncbw[#1](#2)}
\newcommandx{\nfuncA}[2][1=]{\nfuncAw[#1](#2)}

\def\funnoiseomega{\omega}
\def\funcAw{\mathbf{A}}

\newcommandx{\funcnoiselim}[2][2=]{\funnoiseomega^{#2}(#1)}

\newcommand{\gf}[1]{\nabla f(#1)}

\def\rset{\mathbb{R}}
\def\nset{\mathbb{N}}

\newcommandx{\firstsentence}[1]{\textcolor{purple}{#1}}

\newcommandx{\paul}[1]{\todo[color=purple!20]{PM: #1}}
\newcommandx{\alain}[1]{\todo[color=purple!20]{AD: #1}}
\newcommandx{\alaini}[1]{\todo[color=purple!20,inline]{AD: #1}}
\newcommandx{\aymeric}[1]{{\color{orange}{AD: }#1}}
\newcommandx{\som}[1]{\todo[color=green!20]{SS: #1}}
\newcommandx{\pauli}[1]{\todo[inline,color=purple!20]{PM: #1}}
\newcommandx{\todoi}[1]{\todo[inline,color=red!50]{TODO: #1}}

\def\eqsp{\enspace}
\def\rmd{\mathrm{d}}

\def\Id{\mathrm{Id}}

\newcommand{\indi}[1]{\mathbf{1}_{#1}}

\def\PE{\mathbb{E}}
\newcommandx{\CPE}[2]{\PE \left[ #1 ~\middle|~ #2 \right]}

\newcommandx{\norm}[2][2=]{ \lVert #1 \rVert_{#2} }
\newcommandx{\bnorm}[2][2=]{ \Big\lVert #1 \Big\rVert_{#2} }
\newcommandx{\pscal}[2]{ \langle #1 , #2 \rangle }
\newcommandx{\bpscal}[2]{ \Big\langle #1 , #2 \Big\rangle }
\newcommandx{\abs}[1]{ | #1 | }
\newcommandx{\babs}[1]{ \Big| #1 \Big| }

\def\wasserstein{\mathbf{W}_2}

\newtheorem{theorem}{Theorem}%
\newtheorem{lemma}{Lemma}%
\newtheorem{proposition}{Proposition}%
\newtheorem{remark}{Remark}%

\DeclareMathOperator*{\argmin}{\arg\min}

\newcommandx{\nbarA}[1][1=]{\bar{A}_{#1}}
\newcommand{\eqdef}{\overset{\Delta}{=}}

\newcommandx{\xic}[1][1=c]{\xi_{#1}}
\renewcommand{\iint}[2]{\{#1,\ldots,#2\}}
\def\nsets{\nset^*}

\newcommandx{\const}[1][1=]{\mathrm{C}_{#1}}

\newcommand{\expe}[2][]{\ifthenelse{\equal{#1}{}}{\mathbb{E}\left[ #2  \right]}{\mathbb{E}_{#1}\left[ #2  \right]}}
\newcommand{\expeLigne}[2][]{\ifthenelse{\equal{#1}{}}{\mathbb{E}[ #2 ]}{\mathbb{E}_{#1}[ #2  ]}}
\newcommand{\expeLine}[2][]{\ifthenelse{\equal{#1}{}}{\mathbb{E}[ #2 ]}{\mathbb{E}_{#1}[ #2  ]}}

\newcommandx{\normLigne}[2][1=]{\ensuremath{\Vert #2 \Vert^{#1}}}

\def\mrl{\mathrm{L}}

%% file: 2024-AISTATS/header/aistats-letters.tex
\ExplSyntaxOn

\NewDocumentCommand{\definealphabet}{mmmm}
 {%
  \int_step_inline:nnn { `#3 } { `#4 }
   {
    \cs_new_protected:cpx { #1 \char_generate:nn { ##1 }{ 11 } }
     {
      \exp_not:N #2 { \char_generate:nn { ##1 } { 11 } }
     }
   }
 }

\ExplSyntaxOff

\definealphabet{bb}{\mathbb}{A}{Z}
\definealphabet{mc}{\mathcal}{A}{Z}
\definealphabet{mc}{\mathcal}{a}{z}
\definealphabet{mf}{\mathfrak}{A}{Z}
\definealphabet{mf}{\mathfrak}{a}{z}
\definealphabet{ms}{\mathsf}{A}{Z}
\definealphabet{ms}{\mathsf}{a}{z}

\newcommand*{\ie}{i.e.\@\xspace}

\makeatletter
\newcommand*{\etc}{%
    \@ifnextchar{.}%
        {etc}%
        {etc.\@\xspace}%
}
\makeatother

\newtheorem{assum}{\textbf{A}\hspace{-2pt}}
\crefname{assum}{A\hspace{-2pt}}{A\hspace{-2pt}}

%% file: 2024-AISTATS/src/aistats-introduction.tex
Federated averaging (\FedAvg) \citep{mcmahan2017communication}
 has become a cornerstone of federated
 learning.
 It allows multiple clients to collaborate on a shared optimization problem without having to exchange their local data directly. While \FedAvg\ has proven practical efficiency in many federated learning scenarios, its convergence can be significantly affected by the heterogeneity of clients. In fact, \FedAvg\ performs several local updates to speed up the training process and reduce communication costs. However, this leads to the \emph{local drift} phenomenon \citep{karimireddy2020scaffold}: as the number of local steps increases, each client tends to converge to an optimum that matches its local data, rather than the global optimum of the entire coalition, leading to biases in the resulting conclusions.

Several methods have been proposed to mitigate the bias of \FedAvg caused by the heterogeneity across clients. These approaches typically fall into two categories: control variates-based methods \citep{karimireddy2020scaffold, mishchenko2022proxskip, malinovsky2022variance} and primal-dual proximal approaches \citep{sadiev2022communication, grudzien2023can}. These techniques allow for more local steps while complying with lower bounds on the number of communications required for federated learning \citep{arjevani2015communication}.

Recently, it was found that \FedAvg\ suffers from a second type of bias known as \emph{iterate bias}. 
This bias appeared in multiple analyse of federated averaging
\citet{khaled2020tighter,glasgow2022sharp,wang2024Unreasonable}, as an additional term that scales with the variance of the gradients and the number of local steps.
This bias arises from the use of local stochastic gradients, similar to what was observed in previous work on \SGD \citep{pflug1986stochastic, dieuleveut2020bridging}. In this paper, we propose a new analysis of \FedAvg\ for strongly convex and smooth local objective functions.
Our analysis gives new insights on \FedAvg's convergence and bias.
It also allows us to design a simple mechanism that reduces the algorithm's bias. Our main contributions are as follows:

\begin{enumerate}[label=$\bullet$]
\item First, we propose a refined analysis of \FedAvg, with any number of local step, in the deterministic setting, where the local gradients are exact. We recall that, in the presence of client heterogeneity, \FedAvg\ suffers from a bias: it does not converge to the global optimum, but rather to another point that lies in its neighborhood. %
Then, we derive an exact first-order expansion in $O(\gamma H)$ of this bias, where $\step$ is the step size and $\nlupdates$ the number of local updates. %
\item We then extend this analysis to \FedAvg\ with \textit{stochastic} gradients. 
We highlight the Markov property of \FedAvg's iterates, showing similarity with \SGD, as studied by \cite{dieuleveut2020bridging}.
Leveraging this property, we show that, for any number of local steps, \FedAvg's iterates sequence admits a unique stationary distribution and converges exponentially fast in the second-order Wasserstein distance. This allows us to provide a sharp analysis of \FedAvg, establishing an explicit first-order expansion of its bias in $O(\gamma H)$. We show that the bias can be decomposed into two terms: one depending solely on the covariance of the stochastic gradients, and one depending solely on client heterogeneity. The scaling of these terms is influenced by both \emph{gradient} and \emph{Hessian} dissimilarity, extending existing results. %

\item We propose a novel approach for mitigating bias, addressing both heterogeneity and stochastic noise using the Richardson-Romberg extrapolation procedure. In contrast to \Scaffold, this method does not use control variates, and thus does not incur additional memory cost at the client level.
To the best of our knowledge, this is the first method capable of reducing the stochastic bias inherent in \FedAvg. We validate this approach numerically, demonstrating that it can outperform existing bias-correction techniques, such as \Scaffold, particularly in scenarios where gradient variance is substantial.
\end{enumerate}

\paragraph{Notation.}
In this paper, we denote by $\pscal{\cdot}{\cdot}$ the euclidean dot product, and $\norm{\cdot}$ the associated norm.
Vectors are column vectors, we denote $\Id$ the identity matrix, and $\oneVec_{n}$ the vector of size $n$ filled with $1$'s. For a three times differentiable function $f$ and $i \in \{1,2,3\}$ we denote $\nabla^i f$ its $i$-th order derivatives.
For a sequence of matrices $M_1, \dots M_k$, we denote the product by $\prod_{\ell=1}^k M_\ell = M_k M_{k-1} \cdots M_1$.
For two matrices $A, B$, we denote $A \otimes B$ the linear operator $M \mapsto A M B$, where $A, B$ and $M$ are matrices of compatible sizes.
Furthermore, we denote $M^{\otimes k}$ the $k^{\textnormal{th}}$ tensor power of a tensor $M$.
Let $\brset{d}$ be the Borel $\sigma$-field of $\rset^d$. 
For two probability measures $\lambda, \nu$ over $\rset^d$ with finite second moment, we define the second-order Wasserstein distance as $\wasserstein^2(\lambda, \nu) = \inf_{\xi \in \Pi(\lambda, \nu)} \int \norm{ \theta - \varparam }^2 \xi(\rmd \theta, \rmd \varparam)$, where $\Pi(\lambda, \nu)$ is the set of probability measures on $\rset^d \times \rset^d$ such that $\xi(\msA \times \rset^d) = \lambda(\msA)$ and $\xi(\rset^d \times \msA) = \nu(\msA)$ for all $\msA \in \brset{d}$.

%% file: 2024-AISTATS/src/aistats-assumptions.tex
\begin{algorithm}[tb]
\caption{\FedAvg}
\label{alg:fedavg}
\textbf{Input}: step size $\step > 0$, initial $\param[0] \in \rset^d$, number of rounds $\nrounds > 0$, number of clients $\nagent > 0$, number of local steps $\nlupdates > 0$ 
\begin{algorithmic}[1] %
\FOR{$t=0$ to $\nrounds-1$}
\FOR{$c=1$ to $\nagent$}
\STATE Initialize $\param[c,t]^{0} = \param[t]$
\FOR{$h=0$ to $\nlupdates-1$}
\STATE Receive random state $\randState[c,t][h+1]$
\STATE Set $\param[c,t]^{h+1} = \param[c,t]^{h} - \step \gnfs{\param[c,t]^{h}}{\randState[c,t][h+1]}$
\ENDFOR
\ENDFOR
\STATE
Average:
$\param[t+1] = \tfrac{1}{\nagent} \sum\nolimits_{c=1}^{\nagent} \param[c,t]^{\nlupdates}$
\ENDFOR
\STATE \textbf{Return: } $\param[T]$    
\end{algorithmic}
\end{algorithm}

\paragraph{Federated Averaging.}
We study the federated stochastic optimization problem
\begin{align}
    \label{pb:smooth-fl}
    \paramlim \in \argmin_{\paramw \in \rset^d} 
        \fw(\paramw) = \frac{1}{\nagent} \sum_{c=1}^\nagent \nf[c]{\paramw} \eqsp,
\end{align}
where for each $c \in \iint{1}{\nagent}$, $\nf[c]{\paramw} = \PE[ \nfs[c]{\paramw}{ \randState[c][] } ]$, with $\randState[c][]$ a random variable with distribution $\xic$,
taking values in a measurable set $(\msZ,\mcZ)$, and $(z,\theta)\mapsto \nFww[c]{z}(\theta)$ are measurable functions. To solve \eqref{pb:smooth-fl}, we consider $\nagent$ clients indexed by $c \in {1, \dots, \nagent}$, and assume that each client $c$ has access to its own function $\nfw[c]$ through stochastic sampling of $\nFww[c]{\randState[c][]}$.
In this case, \FedAvg\ solves the problem~\eqref{pb:smooth-fl} by performing local stochastic gradient updates on each client. These local iterations are sent at regular intervals to a central server, which aggregates them by calculating the average and sends this updated estimate back to the clients. The clients then restart their local updates based on this new estimate. Starting from a common initial point $\param [0]$ shared by all clients and the server, in each round $t \in\nsets$ the server sends its current estimate $\param[t]$ to each client $c \in {1, \dots, \nagent}$. Then each client $c$ starts with this updated value and sets $\param[c,t]^{0} = \param[t]$, and performs $\nlupdates \in\nsets$ local updates: for $ h \in \{0, \dots, \nlupdates-1\}$,
\begin{equation*}
    {\param[c,t]^{h+1}
     = 
      \param[c,t]^{h}
      - \step \gnfs[c]{\param[c,t]^{h}}{\randState[c,t][h]}} 
      \eqsp, 
\end{equation*}
where $\gamma >0$ is a common step size shared by the clients, and $\{\smash{\randState[\tilde{c},\tilde{t}][\tilde{h}]} \, : \, \tilde{c} \in \iint{1}{\nagent}, \tilde{h} \in \iint{0}{\nlupdates-1}, \tilde{t} \in \nset\}$ are independent random variables, so that for each $\tilde{c} \in\iint{1}{\nagent}$, $\tilde{h} \in \iint{0}{\nlupdates-1}$ and $\tilde{t}\in\nset$, $\smash{\randState[\tilde{c},\tilde{t}][\tilde{h}]}$ has distribution $\xic$. Once the local updates are complete, each client sends its last iteration $\smash{\param[c,t]^{H}}$ to the central server, which updates the global parameters as
\begin{equation}
  \label{eq:def_global_estim}
  \param[t+1] = \frac{1}{\nagent} \sum_{c= 1}^N     \param[c,t]^{H} \eqsp.
\end{equation}
We give the pseudocode of \FedAvg in \Cref{alg:fedavg}.
The main challenge with this algorithm is that using local updates introduces bias when the clients' local functions are heterogeneous, a phenomenon that we formally characterize in \Cref{sec:deterministic-fedavg} and \Cref{sec:stochastic-fedavg}.

\paragraph{Assumptions.}
Throughout this paper, we consider the following assumptions.
\begin{assum}[Regularity]
\label{assum:local_functions}
For every $c \in \{1, \dots, \nagent\}$, the function $f_{c}$ is three times differentiable. In addition, suppose that for every $c \in \{1, \dots, \nagent\}$:
\begin{enumerate}[label=(\alph*),leftmargin=*]
\item \label{assum:item_strong_convex}
The function $f_c$ is $\mu$-strongly convex with $\mu > 0$, that is $ \hnf[c]{\paramw}
      \succcurlyeq
      \mu \Id$.
Moreover, for all $z \in \msZ$, the function $\nFww[c]{z}$ if convex.
\item \label{assum:smoothness} There exists a constant $\explip > 0$ such that, for all $z \in \msZ$, the function $\nFww[c]{z}$ is $\explip$-smooth. 
In particular, for all $\paramw, \varparam \in \rset^d$, it holds that
  \begin{align*}
 &  \norm{ \gnfs[c]{\paramw}{Z_{c}} - \gnfs[c]{\varparam}{Z_{c} }}^2 
    \le\\
    & \qquad \qquad 
    \explip \pscal{ \paramw - \varparam}{ \gnfs[c]{\paramw}{z} - \gnfs[c]{\varparam}{z} } 
      \eqsp.
  \end{align*}
\item \label{assum:smoothness_c} For all $\param \in \rset^d$, it holds that $ \hnf[c]{\paramw} \preccurlyeq \explip \Id$.
\item \label{assum:bounded-third-lip} The third derivative of $\nfw[c]$ is uniformly bounded.
\end{enumerate}
\end{assum}

Note that under \Cref{assum:local_functions}, 
  $\nagent^{-1} \sum_{c=1}^\nagent \nfw[c]$ is $\mu$-strongly convex and therefore has a unique
  minimizer $\paramlim$, and the operator $\Id \otimes \hf{\paramlim} + \hf{\paramlim} \otimes \Id$ is invertible.

\begin{assum}[Heterogeneity Measure] 
\label{assum:heterogeneity}
  There exist $\heterboundgrad, \heterbound > 0$ such that for any $c \in \{1, \dots, \nagent\}$, with $\paramlim$ as in \eqref{pb:smooth-fl},
  \begin{align*}
      \frac{1}{\nagent} \sum_{c=1}^\nagent 
      \norm{ \nabla^i \nf[c]{\paramlim} - \nabla^i \fw(\paramlim) }^2
    & \le \heterboundw[i]^2
    \,
    \text{ for } i \in \{1 , 2\}
    \eqsp.
  \end{align*}
  where we recall that $\gf{\paramlim} = 0$.
\end{assum}
Note that when the solution of \eqref{pb:smooth-fl} is unique, which is notably the case under \Cref{assum:local_functions}, this assumption also holds.

%% file: 2024-AISTATS/src/aistats-related-work.tex
\paragraph{Analysis of Federated Averaging.}
\FedAvg\ was first introduced by \citet{mcmahan2017communication}. Since then, numerous analyses have been developed. Initial studies primarily relied on assumptions of homogeneity \citep{stich2019local,wang2018cooperative,haddadpour2019convergence,yu2019parallel,wang2018cooperative,li2019communication}. 
Several works have proposed to study \FedAvg a fixed-point method by \citet{malinovskiy2020local,wang2021local}, and multiple works have shown convergence of \FedAvg with deterministic gradients to a biased point, whose distance to the solution depends on the number of local steps and heterogeneity levels \citep{malinovskiy2020local,charles2021convergence,pathak2020fedsplit}, with an explicit characterization of the bias in the quadratic case.
Over time, various heterogeneity measures have been proposed to derive upper bounds on the error of \FedAvg. Among the most common assumptions is \emph{bounded gradient dissimilarity} \citep{yu2019linear,khaled2020tighter,karimireddy2020scaffold,reddi2021adaptive,zindari2023convergence,crawshaw2024federated}. Other measures include second-order similarity \citep{arjevani2015communication,khaled2020tighter}, relaxed first-order heterogeneity \citep{glasgow2022sharp}, and average drift at the optimum \citep{wang2024Unreasonable,patel2023still}. It has also been demonstrated that \FedAvg\ can achieve linear speed-up in the number of clients \citep{yang2021achieving, qu2021federated}.

\paragraph{Correcting Heterogeneity Bias.}
A first approach for addressing heterogeneity is based on control variates, pioneered by the \Scaffold algorithm \citep{karimireddy2020scaffold}. \citet{mishchenko2022proxskip} later demonstrated that \Scaffold effectively accelerates training, and since then, other control variates schemes have been developed  \citep{condat2022randprox, malinovsky2022variance, condat2022provably, grudzien2023can, mangold2024scafflsa}. In addition, a class of algorithms relying on dual-primal approaches has been proposed to address heterogeneity \citep{sadiev2022communication, grudzien2023can}. 
While both approaches allow for more local training steps and effectively correct heterogeneity bias, they do not address the bias caused by stochasticity when using fixed steps ize.

\paragraph{Stochastic Bias.}
Even in the single-client setting, \SGD with fixed step size have been shown to exhibit bias \citep{lan2012optimal, defossez2015averaged, dieuleveut2016nonparametric, chee2017convergence}. \citet{dieuleveut2020bridging} proposed framing \SGD iterates with a constant step size as a Markov chain, drawing connections to established results in stochastic processes \citep{pflug1986stochastic}. Stochastic bias has also been observed in the analysis of federated learning methods. For instance, \citet{khaled2020tighter} identified this bias in their bounds on client drift, and similar observations were made in the convergence analyses of \citet{glasgow2022sharp, wang2024Unreasonable}, which compared \SGD's iterates to those of deterministic gradient descent. In this work, we investigate the iterate bias of \FedAvg, demonstrating that the stationary distribution of \SGD's iterates is inherently biased.
\paragraph{Richardson-Romberg.}
The Richardson-Romberg extrapolation technique, originally introduced by \citet{richardson1911ix}, is a classical method in numerical analysis. This approach has been widely applied across various fields, including time-varying autoregressive processes \citep{moulines2005recursive}, data science \citep{bach2021effectiveness}, and many others \citep{stoer2013introduction}. Specifically, it has been utilized in the context of \SGD by \citet{dieuleveut2020bridging} and \citet{sheshukova2024nonasymptotic}. In this work, we extend these ideas to the federated learning setting, demonstrating that this form of extrapolation effectively mitigates both heterogeneity and stochastic bias.

%% file: 2024-AISTATS/src/aistats-deterministic.tex
In this section, we present a new analysis of \FedAvg with deterministic gradients (\FedAvgwDG), where $\nFww[c]{z} = \nfw[c]$ for all $c \in \iint{1}{\nagent}$ and $z \in\msZ$. This analysis highlights the core philosophy of the method developed in this paper. Unlike previous analyses, we demonstrate that \FedAvgwDG converges to a point $\detlim{\step, \nlupdates}$ that differs from the optimal solution $\paramlim$. We then provide an explicit expression for the distance between these two points, allowing us to establish tight upper bounds on the bias of \FedAvgwDG.

In the \FedAvgwDG setting, we use the formulation of \FedAvgwDG using fixed-point methods \citep{malinovskiy2020local}.
We thus define the local updates of the client $c$ by induction, starting from the point $\paramw \in \rset^{d}$: 
\begin{equation*}
  \dfedavgop[\gamma, h+1]_c(\paramw)
  \eqdef   
  (\Id -\gamma \nabla f^{(c)})(\dfedavgop[\gamma, h]_c(\paramw)) ~, \eqsp \dfedavgop[\gamma, 0]_c(\paramw) \eqdef \paramw~,
\end{equation*}
where $h \in \iint{0}{\nlupdates-1}$. 
The global updates from \eqref{eq:def_global_estim} can thus be rewritten as
\begin{equation*}
\dfedavgop[\gamma, \nlupdates](\paramw) \eqdef  \frac{1}{\nagent} \sum_{c=1}^\nagent
  \dfedavgop[\gamma, \nlupdates]_c(\paramw)
  \eqsp,
\end{equation*}
or, equivalently, we can write $\dfedavgop[\gamma, H](\theta) = \paramw - \step \pseudograd[\step, \nlupdates]{\paramw}$, with the pseudo-gradient
\begin{align*}
  \pseudograd[\step, \nlupdates]{\paramw}
  & \eqdef
    \frac{1}{\nagent} \sum_{c=1}^\nagent \sum_{h=0}^{\nlupdates-1}  \gnf[c]{\dfedavgop[\gamma, h]_c(\paramw)}
    \eqsp.
\end{align*}
First, we show that \FedAvgwDG with deterministic updates converges to a fixed point of $\dfedavgop[\gamma, H]$. %
\begin{restatable}[Stationary Point of \FedAvgwDG]{proposition}{statpointfedavgdet}
\label{prop:convergence-fedavg-to-point}
Assume \Cref{assum:local_functions}. Then for all $\nlupdates > 0$ and $\step \le 1/\explip$, \FedAvgwDG converges to a unique point $\detlim{\step, \nlupdates}$ that satisfies $\dfedavgop[ \gamma,\nlupdates](\detlim{\step, \nlupdates}) = \detlim{\step, \nlupdates}$ and $\pseudogradw[\step, \nlupdates](\detlim{\step, \nlupdates}) = 0$. 
Moreover, the iterates of \FedAvgwDG satisfy
\begin{align*}
\norm{ \param[t] - \detlim{\step, \nlupdates}}^2
& \le
(1 - \step \strcvx)^{\nlupdates t} \norm{ \param[0] - \detlim{\step, \nlupdates} }^2
\eqsp.
\end{align*}
\end{restatable}
We note that similar results have been derived by \citet{malinovskiy2020local,pathak2020fedsplit,charles2021convergence}, using the fact that local updates are contractive.
Nonetheless, we provide a proof of this statement in \Cref{sec:proof-thm-expand-deterministic-prelim-prop} for completeness.
This result shows that taking a larger number of local updates $\nlupdates$ effectively speeds up the process, although this can also move the limit point $\detlim{\step,\nlupdates}$ away from the solution $\paramlim$.

To characterize this stationary point, we derive an explicit expression for the bias $\detlim{\step, \nlupdates} - \paramlim$ of \FedAvg. We define the matrices, for $h \in \iint{1}{\nlupdates}$,
\begin{equation*}
\inthess[c]{\step, h} \eqdef \int_{0}^1 \hnf[c]{ u \, \dfedavgop[\gamma, h]_c(\detlim{\step, h}) + (1 - u) \paramlim} \rmd u \eqsp.
\end{equation*}
We also define the following matrix products, that allow expressing the update of the error when starting from the point $\detlim{\step, \nlupdates}$
\begin{align}
\label{eq:def-contract-det}
\detcontract[c]{\star,h+1:\nlupdates}
\!\eqdef \!\!
\!\prod_{\ell=h+1}^{\nlupdates-1} \! \!\!\left(\Id - \step \inthess[c]{\step,\ell}\right)
,
~
\detfullcontract{\star} \!\eqdef\! 
\frac{1}{\nagent} \sum_{c=1}^\nagent \detcontract[c]{\star}
\!\eqsp,
\end{align}
where $\detcontract[c]{\star} = \detcontract[c]{\star,1:\nlupdates}$.
We now provide an expression and an upper bound on the bias of \FedAvgwDG.

\begin{restatable}[Bias of \FedAvgwDG]{proposition}{biasdetfedavgwdg}
\label{prop:bias-det-fedavg}
Assume \Cref{assum:local_functions} and  \Cref{assum:heterogeneity}. Then for all $\nlupdates > 0$ and $\step \le 1/\lip$, we have
\begin{align*}
\detlim{\step, \nlupdates} - \paramlim
& = 
\frac{1}{\nagent}
\sum_{c=1}^\nagent
\sum_{h=1}^\nlupdates
\Upsilon^{(\step, h)}_c \gnf[c]{\paramlim}
\eqsp,
\end{align*}
where $\Upsilon^{(\step, h)}_c = (\Id - \detfullcontract{\star})^{-1} \detcontract[c]{\star,h+1:\nlupdates}$ and $\detcontract[c]{\star},\detfullcontract{\star}$ are defined in \eqref{eq:def-contract-det}.
Furthermore, if $\step \strcvx \nlupdates \le 1$, then
\begin{align*}
\norm{ \detlim{\step, \nlupdates} - \paramlim }
& \le \gamma (\nlupdates-1) \const[1] \eqsp, \, \text{ with } \const[1] \eqdef 
\lip  \heterboundgrad/{\mu} 
\eqsp.
\end{align*}
\end{restatable}
We prove \Cref{prop:bias-det-fedavg} in \Cref{sec:proof-thm-expand-deterministic-prelim-prop}, using the fact that $\dfedavgop[ \gamma,\nlupdates](\detlim{\step, \nlupdates}) = \detlim{\step, \nlupdates}$ from \Cref{prop:convergence-fedavg-to-point}.
Importantly, when $\nlupdates=1$, the bias of \FedAvg completely vanishes, recovering the fact that gradient descent converges. Based on \Cref{prop:bias-det-fedavg}, we further propose a first-order expansion of the bias of \FedAvgwDG. This highlights that (i) the bias of \FedAvgwDG solely depends on heterogeneity, and (ii) the convergence bound derived in \Cref{prop:bias-det-fedavg} is sharp for small values of the product $\gamma H$. 
\begin{theorem}[First-Order Bias of \FedAvgwDG]
\label{cor:first-order-det}
Assume \Cref{assum:local_functions} and \Cref{assum:heterogeneity}. Then for all $\nlupdates > 0$ and $\step \le 1/\lip$ such that $\step \strcvx \nlupdates \le 1$, we have
\begin{align*}
& \detlim{\step, \nlupdates} - \paramlim
 \! = \!
\frac{\step(\nlupdates\!\!-\!1)}{2}   \biasHetero + O( \step^2 \nlupdates^2 )
\eqsp,
\end{align*}
where the heterogeneity bias $\biasHetero$ is given by
\begin{equation*}
\!\biasHetero
\!\eqdef\!
\frac{1}{\nagent} 
\!\sum_{c=1}^\nagent  
\hf{\paramlim}^{-1} 
(\hnf[c]{\paramlim} -  \hf{\paramlim}) \gnf[c]{\paramlim} ~, 
\end{equation*}
and
the explicit expression of the reminder term $O(\gamma^2 H^2)$ is given in \Cref{sec:proof-thm-expand-deterministic}.
\end{theorem}
The proof of \Cref{cor:first-order-det} is given in \Cref{sec:proof-thm-expand-deterministic}.
This statement shows that the scale of $\detlim{\step, \nlupdates} - \paramlim$ depends on the scale of local gradients at $\paramlim$, but \emph{also on the difference of Hessians at the solution}.

Furthermore, as a byproduct of \Cref{prop:convergence-fedavg-to-point,prop:bias-det-fedavg}, we obtain the following corollary, establishing the convergence of \FedAvgwDG to a neighborhood of $\paramlim$.
\begin{restatable}[Convergence Rate of Deterministic \FedAvgwDG]{corollary}{corollaryconvergenceratefedavgwdg}
\label{cor:convergence-det-fedavg}
Assume \Cref{assum:local_functions} and  \Cref{assum:heterogeneity}.
Let $\nlupdates > 0$ and $\step \le 1/\lip$ such that $\step \strcvx \nlupdates \le 1$. Then the global iterates of \FedAvgwDG satisfy
\begin{align*}
\norm{ \param[t] - \paramlim }^2 
& \le 
2 (1 - \step \strcvx)^{\nlupdates t} \norm{ \param[0] - \detlim{\step, \nlupdates} }^2
\\
& \qquad \qquad \qquad \qquad
+ 2 \step^2 (\nlupdates-1)^2 \const[1]^2
\,.
\end{align*}
\end{restatable}
We prove this Corollary in \Cref{sec:proof-thm-expand-deterministic-prelim-prop}.
This result shows that the iterates of \FedAvgwDG converge linearly to a neighborhood of the solution $\paramlim$.
The radius of this neighborhood is determined by the level of heterogeneity among the clients, quantified by $\heterboundgrad$, and the number of local steps $\nlupdates$. %

%% file: 2024-AISTATS/src/aistats-stochastic.tex
\begin{table*}[t]
\setlength\tabcolsep{0pt}
\def\arraystretch{1.2}
    \centering
    \begin{tabular*}{\linewidth}{@{\extracolsep{\fill}}lcc}
    \toprule
       Assumption  &  Stochastic Bias & Heterogeneity Bias \\
       \midrule
       Deterministic (Thm. \ref{cor:first-order-det}) 
       & 
       N/A
       & 
       $\frac{\step (\nlupdates-1)}{2\nagent} \hf{\paramlim}^{-1} \! \sum_{c=1}^\nagent (\hnf[c]{\paramlim} \!-\! \hf{\paramlim}) \gnf[c]{\paramlim}$  \\
       \midrule
       Quadratic (Thm. \ref{thm:bias-quadratic-sto}) 
       &
       0
       &
       $\frac{\step (\nlupdates-1)}{2\nagent} \hf{\paramlim}^{-1} \! \! \sum_{c=1}^\nagent (\hnf[c]{\paramlim} \!-\! \hf{\paramlim}) \gnf[c]{\paramlim}$
       \\
       Homogeneous (Thm. \ref{thm:exp-homogeneous}) 
       &
       $- \frac{\step}{2 \nagent} \hf{\paramlim}^{-1} \hhf{\paramlim} \invopcov \covfunc(\paramlim)$
       & 0 \\
       Heterogeneous (Thm. \ref{thm:bias-var-heterogeneous}) &
       $- \frac{\step}{2 \nagent} \hf{\paramlim}^{-1} \hhf{\paramlim} \invopcov \covfunc(\paramlim)$
       &
       $\frac{\step (\nlupdates-1)}{2\nagent} \hf{\paramlim}^{-1} \! \! \sum_{c=1}^\nagent (\hnf[c]{\paramlim} \!-\! \hf{\paramlim}) \gnf[c]{\paramlim}$\\
    \bottomrule
    \end{tabular*}
    \caption{Summary of our main results. Each row indicates, for one of our four possible setups, which biases \FedAvg suffers from, and the leading term in the expansion of the bias value for small values of $\gamma \nlupdates$.}
    \label{table:results}
\end{table*}

In this section, we present our main findings, including the first-order expansion of the bias in \FedAvg when using stochastic gradients. We demonstrate that \FedAvg is affected by \emph{two types of bias}: one due to \emph{heterogeneity} and the other one due to \emph{stochasticity}. Our analysis is structured into three scenarios, with progressive complexity.
\begin{itemize}[leftmargin=*,noitemsep,topsep=0pt]
    \item First, when the functions $\nfw[c]$ are quadratic, we show that, similar to the single-client setting, there is no stochastic bias, but only a bias due to heterogeneity.
    \item  Second, assuming homogeneous functions, we show that the bias in \FedAvg still arises due to the use of stochastic gradients, demonstrating that \FedAvg is biased even when functions are homogeneous.
    \item Finally, in the general heterogeneous case, we show that both sources of bias are observed, and that the overall bias of \FedAvg is the sum of the biases observed in the two previous settings.
\end{itemize}
A summary of our results can be found in \Cref{table:results}. For our analysis, we introduce the following assumption, which provides an upper bound on the variance of the stochastic gradient. This bound is expressed as the variance at the solution $\paramlim$, along with an additional polynomial term.
For all $z \in \msZ$ and $\paramw \in \rset^d$, we denote the centered stochastic gradient by
\begin{align}
  \label{eq:def-epsilon}
      \updatefuncnoise[c]{z}(\paramw) \eqdef \gnfs[c]{\paramw}{z} - \gnf[c]{\paramw}
      \eqsp,
  \end{align}
and we assume that its moments satisfy a form of smoothness.
\begin{assum}[Gradient's Variance]
\label{assum:smooth-var}
There exist constants $\Msmoothcstvar, \smoothcstvar \ge 0$ such that for any $\paramw \in \rset^d$, $p \in \{1,2,3\}$, and $c\in \iint{1}{\nagent}$, it holds with a random variable $\randState[c][]$ with distribution  $\xic$ and $\updatefuncnoise[c]{z}(\paramw)$ as in \eqref{eq:def-epsilon}, that
\begin{align*}
    \PE^{1/p}\big[ \norm{ \updatefuncnoise[c]{\randState[c][]}(\paramw) }^{2p} \big] 
    & \le 
    \Msmoothcstvar^2 \left\{ 1 + \norm{ \paramw - \paramlim }^{\smoothcstvar} \right\} 
    \eqsp.
\end{align*}
In particular, we have $\norm{ \PE[ \updatefuncnoise[c]{\randState[c][]}(\paramlim)^{\otimes 2} ]  } \le \Msmoothcstvar^2$.
\end{assum}

\subsection{\FedAvgBold as a Markov Chain}
\paragraph{{\FedAvgBold} Generating Operators.} Now we extend the methodology described in the deterministic case to \FedAvg with stochastic gradients. 
For a vector $\randState[1:\nagent][1:\nlupdates]= \{ \randState[\tilde{c}][\tilde{h}] : \tilde{c} \in \iint{1}{\nagent}, \tilde{h} \in \iint{1}{\nlupdates}\}$,
 and any $c\in \iint{1}{\nagent}$, we recursively define  $\sdfedavgop[c]{\step,h}(\paramw; \randState[c][1:h])$ as an operator generating the local updates of \FedAvg starting form $\theta$. That is, we set $\sdfedavgop[c]{\step,0} = \Id$, and for $h \ge 0 $, we define
\begin{align*}
\sdfedavgop[c]{\step,h\!+\!1}(\paramw; \randState[c][1:h\!+\!1])
\!\eqdef\! \Big(\Id \!-\! \step \gnfws[c]{\randState[c][h\!+\!1]}\Big)\!\Big(\sdfedavgop[c]{\step,h}(\paramw; \randState[c][1:h])\Big) 
\!\!\eqsp.
\end{align*}
We then define $\sdfedavgop{\step,\nlupdates}(\paramw; \randState[1:\nagent][1:\nlupdates])$, an operator generating \FedAvg's global updates. That is, for $\paramw \in \rset^d$, we let
\begin{align}
  \label{eq:def-dfedavgop-sto}
  \sdfedavgop{\step,H}\left(\paramw; \randState[1:\nagent][1:\nlupdates]\right) 
  \eqdef  \frac{1}{\nagent} \sum_{c=1}^{\nagent} \sdfedavgop[c]{\step,\nlupdates}(\paramw; \randState[c][1:\nlupdates])
  \eqsp.
\end{align}
Note that \eqref{eq:def-dfedavgop-sto} can also be written as 
$\sdfedavgop{\step,H}\left(\paramw; \randState[1:\nagent][1:\nlupdates]\right) = \paramw - \step \pseudograds[\step,\nlupdates]{\paramw}{\randState[1:\nagent][1:\nlupdates]}$,
where
\begin{align*}
  \pseudograds[\step,\nlupdates]{\paramw}{\randState[1:\nagent][1:\nlupdates]}
   \!\eqdef\!
    \frac{1}{\nagent} \sum_{c=1}^\nagent\! \sum_{h=0}^{\nlupdates-1}  \gnfs[c]{\sdfedavgop[c]{\step,h}(\paramw;\randState[c][1:h])}{\randState[c][h+1]}
    ~.
\end{align*}
With the notations above, we have that the iterates defined in  \eqref{eq:def_global_estim} can be written, for any $t \geq 0$, as
\begin{align}
\label{eq:globalupdate}
    \param[t+1] =  \sdfedavgop{\step,H}\left(\param[t]; \randState[1:\nagent,t][1:\nlupdates]\right) \eqsp, 
\end{align}
with $\randState[1:\nagent,t][1:\nlupdates]$ the random states at global iteration $t$.
We now study the properties of the sequence $\{\param[t]\}_{t \in \nset}$.

\paragraph{Properties of $\{\param[t]\}_{t \in \nset}$ as a Markov chain.}
Equation~\eqref{eq:globalupdate} shows that \FedAvg's global iterates define a time-homogeneous Markov chain with the corresponding Markov kernel $\markovkernel$ on $(\rset^d, \brset{d})$ defined as 
\begin{align*}
\markovkernel(\paramw, \msB) \eqdef 
\PE[\indi{\msB}(\sdfedavgop{\step,H}(\paramw,\randState[1:\nagent][1:\nlupdates]))]
\eqsp,
\end{align*}
for all $\msB \in \brset{d}$ and $\theta \in \rset^{d}$.
Next we define, for $t \geq 1$, the iterates of $\markovkernel$ as $\markovkernel^1 = \markovkernel$, and, with $\msB \in \brset{d}, \theta \in \rset^{d}$, 
\begin{align*}
  \markovkernel^{t+1}(\paramw, \msB)
  \eqdef
  \int \markovkernel^t(\paramw, \rmd \varparam) 
  \markovkernel(\varparam, \msB) \eqsp.
\end{align*}
For any probability  measure $\rho$ on $\brset{d}$ and $t \in \nsets$, $\rho \kappa^t$ is the distribution of the iterates $\param[t]$ of \FedAvg when started from $\param[0] \sim \rho$.
We now show that the iterates of \FedAvg converge to a unique stationary distribution, giving the counterpart of \Cref{prop:convergence-fedavg-to-point} to the stochastic regime.
\begin{restatable}[Convergence of \FedAvg]{proposition}{convstatdistfedavg}
\label{prop:conv-stat-dist}
Assume \Cref{assum:local_functions} and let $\step \le 1/\explip$. Then the iterates of \FedAvg converge to a unique stationary distribution $\statdist{\step,\nlupdates}$, admitting a finite second moment. Furthermore, for any initial distribution $\rho$ and $t \in \nsets$,
  \begin{align*}
    \wasserstein^2(\rho \markovkernel^{t}, 
    \statdist{\step, \nlupdates})
    & \le
      (1 - \step \strcvx)^{\nlupdates t}
      \wasserstein^2(\rho, \statdist{\step, \nlupdates})
      \eqsp.
  \end{align*}
\end{restatable}
The proof is postponed to~\Cref{sec:proof-stat-dist}. \Cref{prop:conv-stat-dist} shows that the Markov kernel $\kappa$ is geometrically ergodic in $2$-Wasserstein distance. Moreover, the distribution of $\param[t]$ converges to the limiting distribution $\statdist{\step, \nlupdates}$ at a linear rate $(1 - \step\mu)$, for a step size $\step$, with the exponent given by the number of \textit{effective} steps $\nlupdates \times t$. 
As with the deterministic algorithm, a larger number of local steps $\nlupdates$ speeds up the convergence, but leads to additional bias.

Under the conditions of \Cref{prop:conv-stat-dist} we define the mean and the covariance matrix of the parameters under the invariant distribution $\statdist{\step,\nlupdates}$, that is,
\begin{equation}
\label{eq:def-expec-var}
\begin{aligned}
\statdistlim{\step,\nlupdates} &\eqdef \int \varparam \statdist{\step, \nlupdates}(\rmd \varparam)\eqsp, \\
 \covstatdistlim{\step,\nlupdates} & \eqdef \int \{\varparam - \paramlim\}^{\otimes 2} \statdist{\step, \nlupdates}(\rmd \varparam)  \eqsp.
\end{aligned}
\end{equation}

\textbf{Convergence to a neighborhood of the limit.}
Under the following assumption that gradient's variance is uniformly bounded, we can characterize the convergence of \FedAvg to a neighborhood of $\statdistlim{\step,\nlupdates}$.
\begin{assum}[Bounded Variance]
\label{assum:unif-bound}
There exists $\Msmoothcstvarbis > 0$ such that, for any $\param \in \rset^d$, 
$\PE [ \norm{ \gnfs[c]{\param}{\randState} - \gnf[c]{\param} }^2 ]
\le 
\Msmoothcstvarbis^2$.
\end{assum}
We stress that we only require this assumption to study the convergence towards a reference point that is not the solution $\paramlim$. 
In such cases, it is be necessary to bound the variance around any reference point, like in \Cref{assum:unif-bound}.
The following theorem gives the convergence rate of \FedAvg towards a neighborhood of $\statdistlim{\step,\nlupdates}$. 
\begin{restatable}[Convergence to a neighborhood of $\statdistlim{\step,\nlupdates}$]{proposition}{convrateneighborhoodfedavg}
\label{prop:conv-neighborhood-stat-dist}
Assume \Cref{assum:local_functions}, \Cref{assum:smooth-var}, and \Cref{assum:unif-bound}. Let $\step \le 1/(8\explip)$ and $\step \strcvx \nlupdates \le 1$. Then for any $t \in \nsets$, the iterates $\param[t]$ of \FedAvg satisfy
  \begin{align*}
    & \PE[ \norm{ \param[t] - \statdistlim{\step,\nlupdates} }^2 ]
    \le
    (1 - \step \strcvx)^{\nlupdates t} \psi_0
    + \frac{4 \step}{\strcvx} \Msmoothcstvarbis^2
      \eqsp,
  \end{align*}
  where $\psi_0 = 4 \norm{ \param[0] - \paramlim }^2
+ \frac{24 \nlupdates^2 \step^2 \explip^2 \heterboundgrad^2}{\strcvx^2} 
+ \frac{32 \step}{\strcvx} \Msmoothcstvar^{2}$.
\end{restatable}
The proof is postponed to \Cref{sec:conv-neighborhood-proof}.
In this rate, heterogeneity does not appear. However, the reference point $\statdistlim{\step,\nlupdates}$ may differ from the global solution $\paramlim$.

\subsection{Bias of \FedAvgBold}
In the remainder of this section, we derive expansions in $\gamma$ and $\gamma H$ for the bias $\statdistlim{\step, \nlupdates} - \paramlim$ and $ \covstatdistlim{\step,\nlupdates}$.
To this end, we define for $c \in \iint{1}{\nagent}$ the matrices
\begin{align}
\label{eq:def-contract-sto-mat}
\contract[c]{\star} \eqdef \left(\Id - \step \hnf[c]{\paramlim}\right)^\nlupdates
~,
~
\fullcontract{\star} \eqdef
\frac{1}{\nagent} \sum_{c=1}^\nagent \contract[c]{\star}\eqsp.
\end{align}
Note that $\contract[c]{\star}$ and $\fullcontract{\star}$ are analogous to the matrices introduced in \eqref{eq:def-contract-det}, but, contrarily to \eqref{eq:def-contract-det}, we use the Hessian of $\nfw[c]$ at~$\paramlim$. We also define the following operator $\invopcov$ and matrix $\covfunc(\paramlim)$, that will appear in our analysis of bias and variance of the parameters in the stationary distribution $\statdist{\step, \nlupdates}$,
\begin{equation}
\label{eq:def-a-C}
\begin{aligned}
\invopcov 
&\eqdef (\Id \otimes \hf{\paramlim} + \hf{\paramlim} \otimes \Id)^{-1} \eqsp,
\\
\covfunc(\paramlim) 
&\eqdef \PE\Big[ \frac{1}{\nagent} \sum_{c=1}^{\nagent} \updatefuncnoise[1]{1}(\paramlim)^{\otimes 2} \Big]
\eqsp.
\end{aligned}
\end{equation}

\paragraph{Quadratic Functions.}
When the functions $\nfw[c]$ are quadratic, we show that \FedAvg's bias only comes from heterogeneity.
\begin{assum}
\label{ass:quadratic}
Assume that for $c \in \{1, \dots, \nagent\}$ it holds
\begin{align*}
\!
\nf[c]{\paramw} 
= \tfrac{1}{2} \norm{ (\nbarA[c])^{1/2} (\paramw - \paramlim_{c}) }^2
\eqsp,
\end{align*}
where $\nbarA[c] \in \rset^{d \times d}$ is a positive semi-definite matrix, and $\paramlim_{c} \in \rset^d$. 
\end{assum}
Note that $\paramlim$ generally differ from $\frac{1}{\nagent} \sum_{c=1}^{N}\paramlim_{c}$ when not all the $\paramlim_{c}$'s or the $\nbarA[c]$'s are equal.
\begin{restatable}[Bias of \FedAvg, Quadratic Functions]{theorem}{biasfedavgquadratic}
\label{thm:bias-quadratic-sto}
Assume \Cref{assum:local_functions}, \Cref{assum:heterogeneity}, \Cref{assum:smooth-var}, \Cref{ass:quadratic}, and $\step \le 1/L$. Then, using notations from \eqref{eq:def-contract-sto-mat}, the bias of \FedAvg is given by
\begin{equation*}
    \statdistlim{\step, \nlupdates}
    - \paramlim = \frac{1}{\nagent} \sum_{c=1}^\nagent
      (\Id - \fullcontract{\star})^{-1} (\Id - \contract[c]{\star}) ( \paramlim - \paramlim_{c} ) \eqsp.
\end{equation*}
Furthermore, when $\step \strcvx \nlupdates \le 1$, it holds that
\begin{align*}
      \norm{ \statdistlim{\step, \nlupdates} - \paramlim }
    \le {\step( \nlupdates -1) \heterbound \heterboundgrad}/{\strcvx}
    \eqsp,
\end{align*}
and the following expansion holds, using notations from \eqref{eq:def-expec-var},
\begin{align*}
\statdistlim{\step, \nlupdates} - \paramlim
& = 
\frac{\step(\nlupdates\!\!-\!1)}{2} \biasHetero
+ O( \step^2 \nlupdates^2 )
\eqsp,
\\
 \covstatdistlim{\step,\nlupdates}
&=
\frac{\step}{\nagent}
\invopcov \covfunc(\paramlim) 
+ O(\step^{2} \nlupdates^2 + \step^{2} \nlupdates )
\eqsp,
\end{align*}
where $\invopcov$ and $\covfunc(\paramlim)$ are defined in \eqref{eq:def-a-C} and the heterogeneity bias $\biasHetero$ is given in \Cref{cor:first-order-det}.
\end{restatable}
The proof is given in \Cref{sec:proof-quadratic}.
This result shows that in quadratic problems the bias of \FedAvg is \emph{solely driven by heterogeneity}. Moreover, it is bounded above by the product of gradient heterogeneity and Hessian heterogeneity: there is no bias if either of these terms is zero.
This refines previous bounds in the quadratic setting \citep{wang2024Unreasonable, mangold2024scafflsa}.
Moreover, we confirm that there is no bias when $\nlupdates=1$, i.e., when only a single local step is performed. It is also shown that the variance of the stationary distribution of \FedAvg scales with $\frac{1}{\nagent}$, up to higher order terms, which ensures a linear speedup with the number of clients — a crucial feature for federated learning.

\paragraph{Homogeneous Functions.}
When the functions $\nfw[c]$ are homogeneous, we demonstrate that \FedAvg remains biased, with the bias arising solely from the stochasticity of the gradients. Namely, we consider the following assumption.
\begin{assum}[Homogeneity]
\label{assum:homogeneity}
The problem \eqref{pb:smooth-fl} is homogeneous, that is, the functions are equal $\nfw[c] = \fw$ and $\nFww[c]{z} = \nFww[]{z}$, and the distributions $\xic$ are identical for all $c \in \{1, \dots, \nagent\}$ and $z \in \msZ$.
\end{assum}
Under this assumption, the following theorem holds.
\begin{restatable}[Bias of \FedAvg, Homogeneous]{theorem}{biasfedavghomogeneous}
\label{thm:exp-homogeneous}
Assume \Cref{assum:local_functions}, \Cref{assum:smooth-var} and \Cref{assum:homogeneity}.
Let $\step \le 1/(9\lip)$ such that $\step\strcvx\nlupdates \le 1$, then the bias and variance of \FedAvg, as per \eqref{eq:def-expec-var}, under the stationary distribution $\statdist{\step,\nlupdates}$ are
\begin{align*}
 ~ \statdistlim{\step, \nlupdates} - \paramlim
&=
\frac{\step}{2 \nagent} \biasSto
+ O(\step^2 \nlupdates + \step^{3/2})
\eqsp,
\\
  \covstatdistlim{\step,\nlupdates}
&=
\frac{\step}{\nagent}
\invopcov \covfunc(\paramlim)
+ O(\step^{2} \nlupdates + \step^{3/2} )
\eqsp,
\end{align*}
where $\invopcov$ and $\covfunc(\paramlim)$ are defined in \eqref{eq:def-a-C}, and 
the stochasticity bias $\biasSto$ is given by
\begin{equation*}
\biasSto \eqdef    - \hf{\paramlim}^{-1}
\hhf{\paramlim} 
\invopcov 
\covfunc(\paramlim)\eqsp .
\end{equation*}
\end{restatable}
The proof of \Cref{thm:exp-homogeneous} is given in \Cref{sec:proof-homogeneous}.
\Cref{thm:exp-homogeneous} shows that \FedAvg is biased whenever the function $f$ is not quadratic.
This bias is proportional to the third-order derivative of $\fw$ and the variance of the gradients at the solution.
Crucially, this bias exists even if the clients are homogeneous.
It is very similar to the bias of \textsc{SGD} given in \citet{dieuleveut2020bridging} for $\nagent=1$ and results from the fact that the third derivative of $\nfw$ is non-zero.
Remarkably, \Cref{thm:exp-homogeneous} guarantees that as long as $\gamma \nlupdates$ is small enough, both the bias and the variance of \FedAvg\ decrease inversely proportional to the number of clients $\nagent$, leading to the desired linear speed-up property.

It is worth noting that the bias of \FedAvg\ in homogeneous settings was previously identified as \emph{iterate bias}. \citet{khaled2020tighter,wang2024Unreasonable} showed that this iterate bias scales with a uniform bound on the gradient variance, and \citet{glasgow2022sharp} provided a refined upper bound using constraints on the third-order derivative of $\fw$. Our paper goes beyond these results and provides a precise first-order expansion of the bias. Importantly, our estimate scales with the variance at $\paramlim$ and does not require a uniform bound on the gradient variance.

\paragraph{Heterogeneous Functions.}
Finally, we present the bias of \FedAvg in the general case, encompassing non-quadratic and heterogeneous functions.
\begin{restatable}[Bias of \FedAvg, Heterogeneous]{theorem}{biasfedavgheterogeneous}
\label{thm:bias-var-heterogeneous}
Assume \Cref{assum:local_functions},  \Cref{assum:heterogeneity} and \Cref{assum:smooth-var}. 
Let $\step \le 1/(45\lip)$ such that $\step \strcvx \nlupdates \le 1$, then the bias and variance of \FedAvg, as defined in \eqref{eq:def-expec-var}, are
\begin{align*}
\statdistlim{\step, \nlupdates} \!-\! \paramlim
& \! = \!\frac{\step}{2 \nagent}
\biasSto \!+\! \frac{\step(\nlupdates\!\!-\!1)}{2}  
\biasHetero
\!+\! O(\step^{2} \nlupdates^2 \!+\! \step^{{3}/{2}} \nlupdates)
\eqsp,
\\
 \covstatdistlim{\step,\nlupdates}
& \! =
\frac{\step}{\nagent}
\invopcov \covfunc(\paramlim)
+ O(\step^{2} \nlupdates^2 \!+\! \step^{{3}/{2}} \nlupdates)
\eqsp,
\end{align*}
where $\invopcov$ and $\covfunc(\paramlim)$ are defined in \eqref{eq:def-a-C}, and $\biasHetero $ and $ \biasSto$ are defined in \Cref{thm:bias-quadratic-sto,thm:exp-homogeneous} respectively.
\end{restatable}
The proof of \Cref{thm:bias-var-heterogeneous} is given in \Cref{sec:proof-heterogeneous}.
This result shows that the bias of \FedAvg with heterogeneous clients consists of two terms: one due to heterogeneity, which exactly matches the bias of \FedAvg in quadratic settings, and one due to stochasticity, which exactly matches the bias of \FedAvg for homogeneous functions.
Again, in this result, we show that when $\nlupdates$ is of order $O(1/\nagent)$, \FedAvg achieves the linear speed-up with respect to the number of clients $\nagent$.

%% file: 2024-AISTATS/src/aistats-richardson.tex
In this section, we apply the Richardson-Romberg extrapolation method to \FedAvg in the context of stochastic gradients and heterogeneous clients. This approach builds upon the bias expression derived from \Cref{thm:bias-quadratic-sto,thm:exp-homogeneous,thm:bias-var-heterogeneous} to define new estimators, that are built by running \FedAvg twice, using different step sizes, and combining the resulting iterates.
In the following, for $t \in \iint{0}{\nrounds}$, we denote $\paramw^{(\step, \nlupdates)}_{t}$ the iterates of \FedAvg with parameters $\step$ and $\nlupdates$, and $\param[t]^{(2\step, \nlupdates)}$ the iterates with parameters $2\step$ and $\nlupdates$. 

\textbf{Richardson-Romberg Extrapolation.}
Using the sequences of iterates $\paramw^{(\step, \nlupdates)}_{t}$ and $\param[t]^{(2\step, \nlupdates)}$, we define the federated Richardson-Romberg iterates as
\begin{align*}
\vartheta_t^{(\step,H)}
& \eqdef
2 \param[t]^{(\step,\nlupdates)} - \param[t]^{(2\step,\nlupdates)}
\eqsp.
\end{align*}
We stress that computing these iterates does not induce additional memory overhead for the clients. 
However, it requires running \FedAvg twice, multiplying the number of communications by two.
We now show that this procedure reduces \FedAvg's bias, leading to a diminished communication complexity.
This method is thus very well suited for use cases where devices have limited computational resources.
\begin{restatable}[Richardson-Romberg]{theorem}{richardsonrombergconvergenceiterates}
\label{thm:RR-non-avg}
Assume \Cref{assum:local_functions},  \Cref{assum:heterogeneity}, \Cref{assum:smooth-var}, and \Cref{assum:unif-bound}.
Let $\step \le 1/(45\lip)$ and $\step \strcvx \nlupdates \le 1$, then the bias of the Richardson-Romberg estimates is
\begin{align*}
\statdistlimv{\step,H} - \paramlim
=
O(\step^2 \nlupdates^2 \!+\! \step^{3/2} \nlupdates)
\eqsp,
\end{align*}
where $\smash{\statdistlimv{\step,H}} \eqdef 2 \statdistlim{\step, \nlupdates} - \statdistlim{2\step, \nlupdates}$. 
Additionally, for any $\epsilon > 0$, it holds that $\PE[\norm{ \vartheta_t^{(\step,H)} - \paramlim}^2 ] = O (\epsilon^2)$ when $\step = O(\epsilon^2)$, $\nlupdates = O(1/\epsilon^{4/3})$, with a number of communications at least
\begin{align*}
\nrounds = O\Big( \frac{1}{\epsilon^{2/3}} \log\Big( \frac{1}{\epsilon} \Big) \Big)
\eqsp.
\end{align*}
\end{restatable}
We prove this Theorem in \Cref{sec:proof-cv-rr}.
This theorem shows that federated Richardson-Romberg extrapolation effectively reduces the bias of \FedAvg. 
As a consequence, to reach a given precision, its communication complexity is reduced, in its leading factor, by a power $2/3$ compared to \FedAvg.
Note that in \Cref{thm:RR-non-avg}, we only aim to show that the communication complexity has reduced dependency on the desired precision~$\epsilon$.
Thus, we do not study its dependency on the problem's constants $\strcvx$ and $\lip$.
To derive more precise results, one needs to give a precise upper bound on the remainder in \Cref{thm:bias-var-heterogeneous}.
Deriving such bounds is an interesting direction for future work.

\begin{figure*}[t]
    \centering
    
     \begin{subfigure}[b]{0.24\linewidth}
         \centering
         \includegraphics[width=\textwidth]{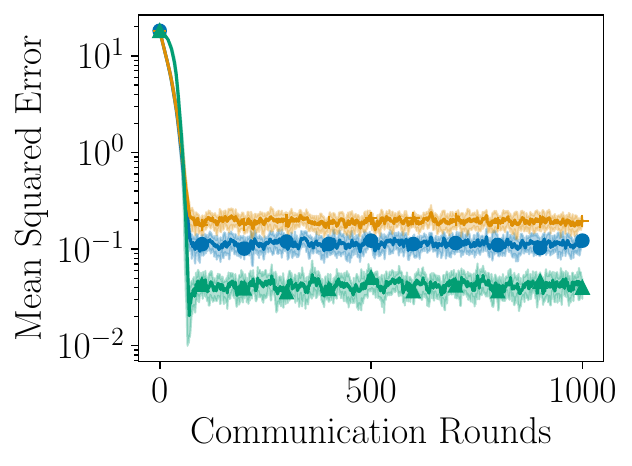}
         \caption{Iterates -- $H=10$}
     \label{10in}
     \end{subfigure}
     \begin{subfigure}[b]{0.235\linewidth}
         \centering
         \includegraphics[width=\textwidth]{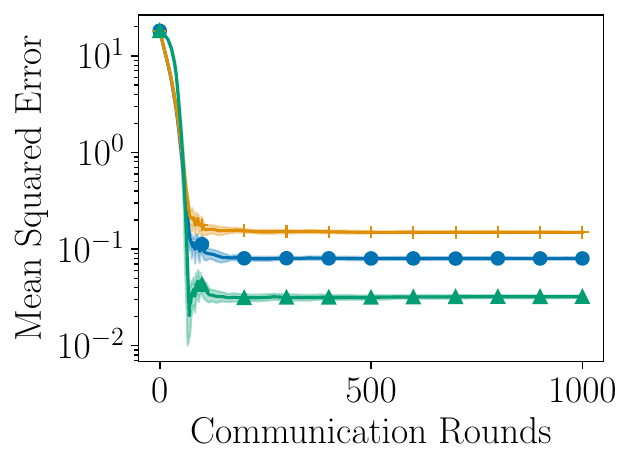}
         \caption{Averaged -- $H=10$}
     \label{10an}
     \end{subfigure}
     \begin{subfigure}[b]{0.24\linewidth}
         \centering
         \includegraphics[width=\textwidth]{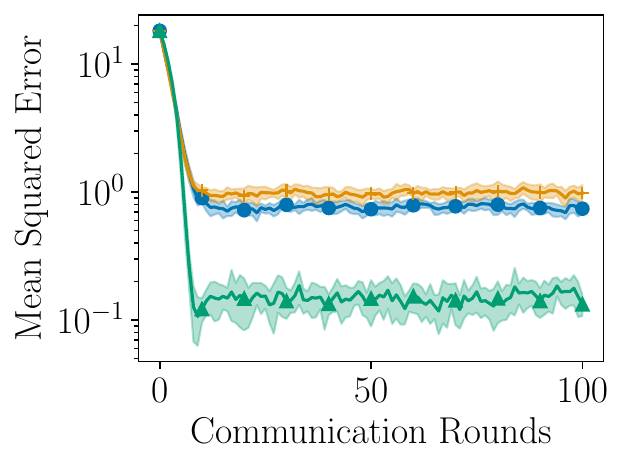}
         \caption{Iterates -- $H=100$}
     \label{100in}
     \end{subfigure}
     \begin{subfigure}[b]{0.235\linewidth}
         \centering
         \includegraphics[width=\textwidth]{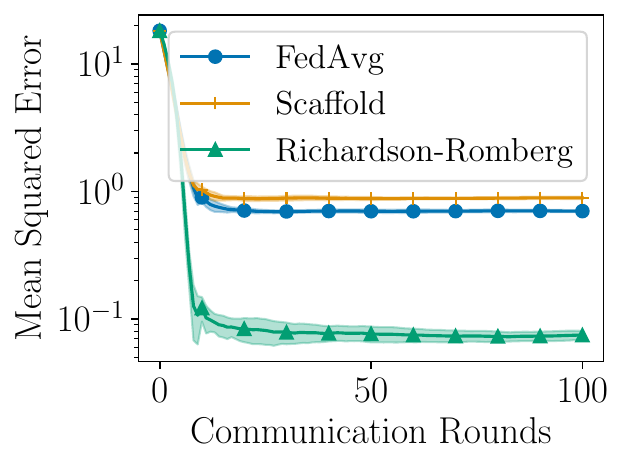}
         \caption{Averaged -- $H=100$}
     \label{100an}
     \end{subfigure}
     
     \begin{subfigure}[b]{0.24\linewidth}
         \centering
         \includegraphics[width=\textwidth]{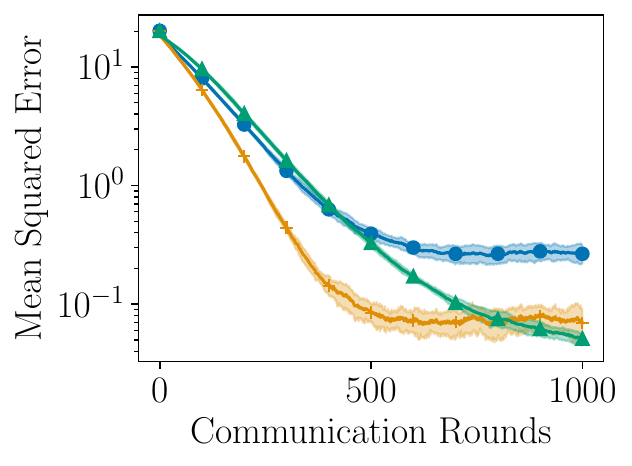}
         \caption{Iterates -- $H=10$}
     \label{10ih}
     \end{subfigure}
     \begin{subfigure}[b]{0.235\linewidth}
         \centering
         \includegraphics[width=\textwidth]{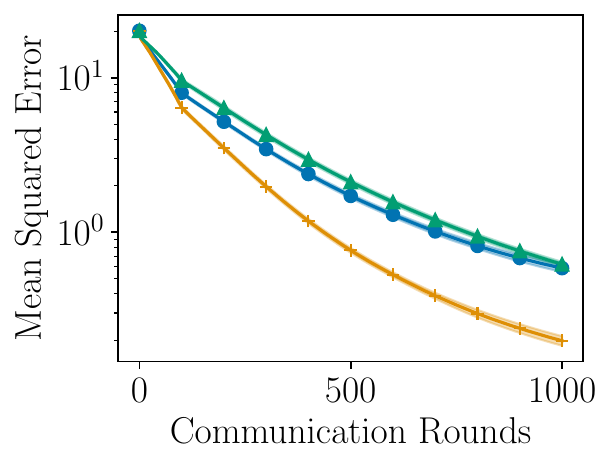}
         \caption{Averaged -- $H=10$}
     \label{10ah}
     \end{subfigure}
     \begin{subfigure}[b]{0.24\linewidth}
         \includegraphics[width=\textwidth]{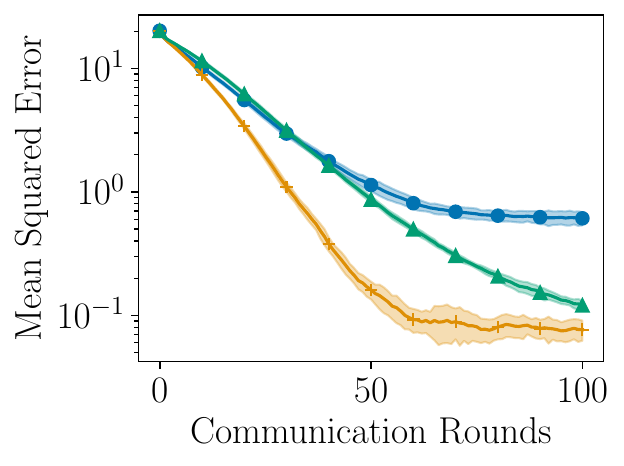}
         \centering
         \caption{Iterates -- $H=100$}
     \label{100ih}
     \end{subfigure}
     \begin{subfigure}[b]{0.235\linewidth}
         \centering
         \includegraphics[width=\textwidth]{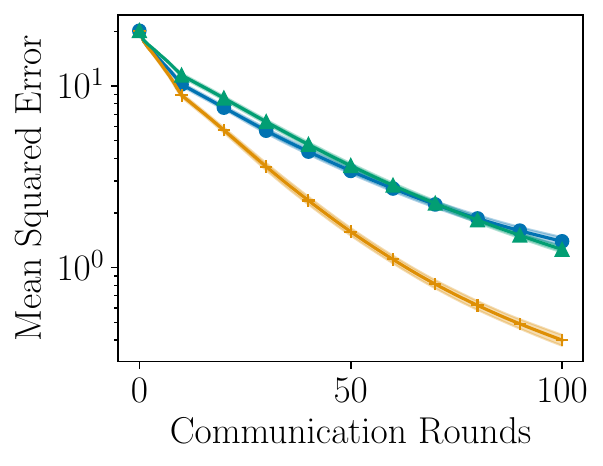}
         \caption{Average -- $H=100$}
     \label{100ah}
     \end{subfigure}
    \caption{Mean squared error on the \texttt{synthetic noisy} (first line) and on the \texttt{synthetic heterogeneous} dataset (second line), as a function of the number of communications, for $\nlupdates \in \{10, 100\}$. In \Cref{10in,100in,10ih,100ih} (labelled \emph{Iterates}), we plot the MSE for global iterates of the three methods, while in \Cref{10an,100an,10ah,100ah} (labelled \emph{Averaged}), we plot the MSE for first $10$\% of iterates, and then plot the MSE of the averaged iterates for the last $90$\% of the iterates. We plot the average over $10$ runs, with standard deviation.}
    \label{fig:results}
\end{figure*}

\textbf{Averaged Estimator.}
Although the previous estimator reduces both heterogeneity and stochasticity bias, its error is still dominated by the variance of single iterates, requiring to take small step sizes to handle variance.
To overcome this issue, we propose the following averaged Richardson-Romberg estimator
\begin{align*}
\bar{\vartheta}_T^{(\step,\nlupdates)} \eqdef \frac{1}{T} \sum_{t=0}^{T-1} \vartheta_t^{(\step,H)} \eqsp.
\end{align*}
In the following theorem, we show that this estimator converges to a point of reduced bias.
To our knowledge, this is the first procedure that uses raw \FedAvg iterates to obtain a result with reduced stochastic bias.
\begin{restatable}[Richardson-Romberg]{theorem}{richardsonrombergconvergence}
\label{theo:RR-bias}
Assume \Cref{assum:local_functions},  \Cref{assum:heterogeneity} and \Cref{assum:smooth-var}.
Let $\step \le 1/(45\lip)$ such that $\step\strcvx\nlupdates \le 1$, then %
\begin{equation*}
    \lim_{T \to \infty} \expe{\normLigne[2]{\bar{\vartheta}_T^{(\step,\nlupdates)} - \statdistlimv{\step,H}}} = 0 \eqsp,
\end{equation*}
where we recall that $\statdistlimv{\step,H} - \paramlim
=
O(\step^2 \nlupdates^2 \!+\! \step^{3/2} \nlupdates)$.
\end{restatable}
We prove this Theorem in \Cref{sec:proof-cv-averaged-rr}.
This implies that, when $\step \nlupdates$ is small, the averaged iterates of \FedAvg with Richardson-Romberg extrapolation have a smaller bias than vanilla \FedAvg. 

Note that, in contrast to \cite{dieuleveut2020bridging}, we do not deal with the variance of \FedAvg and its averaged federated Richardson-Romberg approximation counterpart, \ie, we do not quantify the rate of convergence to $0$ of  $\PE[\Vert \bar{\vartheta}_T^{(\step,\nlupdates)} - \statdistlim{\step,H} \Vert^2] $. Solving this question is an interesting direction for future work.

\begin{remark}
\label{rem:richardson-romberg-local-steps}
When $\nlupdates > 1$, one could define a Richardson-Romberg estimator by varying the number of local steps, defining 
$\omega_t^{(\gamma,H)}
 \eqdef
(2 \nlupdates \!-\! 1)/(\nlupdates \!-\! 1) \param[t]^{(\step, \nlupdates)} - \param[t]^{(2 \step, \nlupdates)}$ and $\bar{\omega}_T^{(\step,\nlupdates)} \eqdef \frac{1}{\nrounds}\sum_{t=0}^{T-1} {\omega}_t^{(\step,H)} $.
The sequence $\{\bar{{\omega}}_T^{(\step,\nlupdates)}\}_{T\geq 1}$ converges to $ (2 \nlupdates - 1)/({\nlupdates - 1}) \statdistlim{\step, \nlupdates} - \statdistlim{\step, 2 \nlupdates}
=
\step \biasSto/(2N)
+ O(\step^2 \nlupdates^2 \!+\! \step^{3/2} \nlupdates^{1/2})$, 
removing heterogeneity bias but not stochasticity bias.
The iterates obtained through this procedure therefore have a bias close to the one of the homogeneous setting.
\end{remark}

%% file: 2024-AISTATS/src/aistats-expe.tex
This section illustrates our theoretical findings using regularized logistic regression problems.
This problem can be formulated as~\eqref{pb:smooth-fl}, using $z = (x, y)$ where $x$ and $y$ are respectively the data features and label, and $\lambda > 0$ is a regularization parameter, and
$\smash{\nf[c]{\paramw} \eqdef \PE\left[ \log(1 + \exp(1 - y_c x_c^\top \paramw)) + \lambda/2 \norm{ \paramw }^2 \right]}$, %
and for each $c \in \iint{1}{\nagent}$, the sample $z_c = (x_c, y_c)$ is drawn from client $c$'s local distribution.

We evaluate our approach on two synthetic datasets with  $\nagent=10$ clients.
The first dataset, coined \texttt{synthetic noisy}, is made of two blobs with large variance, split uniformly among clients. It is thus homogeneous, but contains very noisy data. On the opposite, the second dataset, coined \texttt{synthetic heterogeneous}, is made of $2$ blobs with small variance. Half of the clients receive part of the observations directly, while the other half receive perturbed records with shuffled labels. In this second dataset, data is very heterogeneous but has little noise.

We evaluate three algorithms on these datasets: (i) vanilla \FedAvg, (ii) \FedAvg with Richardson-Romberg extrapolation, as described in  \Cref{sec:richardson-romberg}, and (iii) \Scaffold \citep{karimireddy2020scaffold}. For all experiments, we use $\nagent = 10$ and run the algorithm for a total of $\nrounds \nlupdates = 10,000$ estimation of the full gradient, using batch size one and step size $\step = 0.01$.

We plot the results in \Cref{fig:results}, showing that on the two problems that we consider, \FedAvg with Richardson-Romberg extrapolation consistently outperforms vanilla \FedAvg.
However, in non-noisy, stochastic settings (second line of \Cref{fig:results}), it only partly removes heterogeneity bias.
On the opposite, \Scaffold, which uses control variates to handle heterogeneity, successfully suppresses this bias.
More remarkably, when clients are homogeneous, but have noisy data (first line of \Cref{fig:results}), \FedAvg with Richardson-Romberg can reduce the bias, while \Scaffold fails.
This further confirms our theory, highlighting that \FedAvg with Richardson-Romberg extrapolation effectively reduces stochasticity bias.

%% file: 2024-AISTATS/src/aistats-conclusion.tex
In this paper, we introduced a novel perspective on \FedAvg, centered on the idea that the global iterates of the algorithm converge to a stationary distribution. We conducted a detailed analysis of this distribution, deriving an exact first-order expression for both the bias and variance of \FedAvg's iterates. Notably, our results demonstrate that, as long as the number of local steps is not excessively large, the bias of \FedAvg decreases at a rate of $1/\nagent$. Moreover, we established that \FedAvg's bias consists of two distinct components: one arising purely from data heterogeneity and the other from the stochastic nature of the gradients. Crucially, this proves that \FedAvg remains biased even in perfectly homogeneous settings. Building on this key insight, we applied the Richardson-Romberg extrapolation technique to introduce a new method for mitigating \FedAvg's bias. Unlike existing approaches, our method can reduce \emph{both sources of bias}—heterogeneity bias and gradient stochasticity bias—offering a more comprehensive solution.
This opens novel perspectives for the design of federated learning methods with local training.

%% file: 2024-AISTATS/src/aistats-app-analysis-fedavg-det.tex
\subsection{Convergence and Bias -- Proof of \Cref{prop:convergence-fedavg-to-point,prop:bias-det-fedavg} and \Cref{cor:convergence-det-fedavg}}
\label{sec:proof-thm-expand-deterministic-prelim-prop}

To study the convergence of \FedAvgwDG, we first recall the notations introduced in \Cref{sec:deterministic-fedavg}. Namely, we recall that the local updates of \FedAvgwDG for $\paramw \in \rset^d$ and $0 \le h \le \nlupdates - 1$ are denoted as
\begin{align*}
\dfedavgop[\gamma, 0]_c(\paramw) 
& \eqdef \paramw
\eqsp,
\\
\dfedavgop[\gamma, h+1]_c(\paramw)
& \eqdef   
\dfedavgop[\gamma, h]_c(\paramw) - \step \gnf[c]{\dfedavgop[\gamma, h]_c(\paramw)}
\eqsp.
\end{align*}
Additionally, we recall that $\dfedavgop[\gamma,\nlupdates] = \frac{1}{\nagent} \sum_{c=1}^\nagent \dfedavgop[\gamma, \nlupdates]_c$.
First, we show that the local operators are contractions.
\begin{lemma}[Contraction of \FedAvgwDG's Local Iterates]
\label{lem:contract-op-fedavg-loc}
Assume \Cref{assum:local_functions}. Then, for any $\step \le 1/\lip$, $\paramw, \varparam \in \rset^d$, and $c \in \iint{1}{\nagent}$, it holds that 
\begin{align*}
\norm{ (\paramw -  \step \gnf[c]{\paramw}) - ( \varparam - \step \gnf[c]{\varparam} ) }^2 \le
(1 - \step \mu) \norm{ \param - \varparam }^2
\eqsp.
\end{align*}
\end{lemma}
\begin{proof}
Using strong convexity and co-coercivity, we have for any $c \in \iint{1}{\nagent}$, that
\begin{align*}
\norm{ (\paramw -  \step \gnf[c]{\paramw}) - ( \varparam - \step \gnf[c]{\varparam} ) }^2
& = 
\norm{ \paramw - \varparam }^2
+ \step^2 \norm{ \gnf[c]{\param}
- \gnf[c]{\varparam} }^2 
- 2 \step \pscal{\param
- \varparam }{\gnf[c]{\param}
- \gnf[c]{\varparam} }
\\
& \leq
\norm{ \paramw - \varparam }^2
- 2 \step (1 - \step \lip / 2) \pscal{\param
- \varparam }{\gnf[c]{\param}
- \gnf[c]{\varparam } } \\
& \leq
\norm{ \paramw - \varparam }^2
- 2 \step \mu (1 - \step \lip/2 ) \norm{ \paramw - \varparam }^2 \eqsp.
\end{align*}
To conclude, it remains to note that $\step \le 1/\lip$.
\end{proof}

\begin{lemma}[Contraction of \FedAvgwDG's Global Iterates]
\label{lem:contract-op-fedavg}
Assume \Cref{assum:local_functions}. Then for any $\nlupdates > 0$, $\step \le 1/\lip$, and $\param, \varparam \in \rset^d$, the operator $\dfedavgop[\gamma, \nlupdates]$ satisfies 
\begin{align*}
\norm{ \dfedavgop[\gamma,\nlupdates] (\param) - \dfedavgop[\gamma,\nlupdates] (\varparam) }^2
 \le
(1 - \step \mu)^\nlupdates \norm{ \param - \varparam }^2
\eqsp.
\end{align*}
\end{lemma}
\begin{proof}
First, we show that $\dfedavgop[\gamma,h]_c$ is a strict contraction for any $h \in \iint{1}{\nlupdates}$. Note that for any $\param, \varparam \in \rset^d$,
\begin{align*}
\dfedavgop[\gamma,h+1]_c (\param)
- \dfedavgop[\gamma,h+1]_c (\varparam)
= 
(\dfedavgop[\gamma,h]_c (\param)
- \step \gnf[c]{\dfedavgop[\gamma,h]_c (\param)})
- (\dfedavgop[\gamma,h]_c (\varparam)
- \step \gnf[c]{\dfedavgop[\gamma,h]_c (\varparam)}) 
\eqsp.
\end{align*}
Thus, it follows from \Cref{lem:contract-op-fedavg-loc} that
\begin{align}
\label{eq:one-step-contraction}
\norm{ \dfedavgop[\gamma,h+1]_c (\param)
- \dfedavgop[\gamma,h+1]_c (\varparam) }^2
& \le
(1 - \step \strcvx) \norm{ \dfedavgop[\gamma,h]_c (\param)
- \dfedavgop[\gamma,h]_c (\varparam) }^2
\eqsp.
\end{align}
Using Jensen's inequality and applying \eqref{eq:one-step-contraction} recursively, we obtain
\begin{align*}
\norm{ \dfedavgop[\gamma,\nlupdates] (\param) - \dfedavgop[\gamma,\nlupdates] (\varparam) }^2
& \le
\frac{1}{\nagent} \sum_{c=1}^\nagent \norm{ \dfedavgop[\gamma,\nlupdates]_c (\param) - \dfedavgop[\gamma,\nlupdates]_c (\varparam) }^2
 \le
(1 - \step \mu)^\nlupdates \norm{ \param - \varparam }^2
\eqsp,
\end{align*}
which concludes the proof.
\end{proof}

We now have all the tools required to prove \Cref{prop:convergence-fedavg-to-point}, that we restate here for readability.
\statpointfedavgdet*
\begin{proof}
By \Cref{lem:contract-op-fedavg}, $\dfedavgop[\gamma,\nlupdates]$ is a contraction mapping. Thus, by Banach fixed point theorem, there exists a unique stationary point $\detlim{\step, \nlupdates}$ to which \FedAvgwDG converges, and this point satisfies the fixed-point equation $\dfedavgop[ \gamma,\nlupdates](\detlim{\step, \nlupdates}) = \detlim{\step, \nlupdates}$, or, equivalently, $\pseudogradw[\step, \nlupdates](\detlim{\step, \nlupdates}) = 0$.

Then, we study the convergence rate of the algorithm.
Let $t > 0$, and $\paramw_{t+1}$ be the $(t+1)$-th global iterate of \FedAvg.
Since $\dfedavgop[\gamma,\nlupdates] (\detlim{\step, \nlupdates}) = \detlim{\step, \nlupdates}$, we write
\begin{align*}
\param[t+1] - \detlim{\step, \nlupdates}
& =
\dfedavgop[\gamma,\nlupdates] (\param[t])
- \dfedavgop[\gamma,\nlupdates] (\detlim{\step, \nlupdates})
\eqsp.
\end{align*}
Thus, by \Cref{lem:contract-op-fedavg}, we have
\begin{align*}
\norm{\param[t+1] - \detlim{\step, \nlupdates} }^2
& =
\norm{ \dfedavgop[\gamma,\nlupdates] (\param[t])
- \dfedavgop[\gamma,\nlupdates] (\detlim{\step, \nlupdates}) }^2
\le
(1 - \step \strcvx)^\nlupdates  \norm{ \param[t] - \detlim{\step, \nlupdates} }^2
\eqsp,
\end{align*}
and the result follows by induction.
\end{proof}

\biasdetfedavgwdg*
\begin{proof}
Starting from $\detlim{\step,\nlupdates}$, we write
\begin{align*}
\dfedavgop[\gamma,h+1]_c(\detlim{\step,\nlupdates}) 
& = 
\dfedavgop[\gamma,h]_c(\detlim{\step,\nlupdates}) 
- \step \gnf[c]{\dfedavgop[\gamma,h]_c(\detlim{\step,\nlupdates}) }
\\
& =
\dfedavgop[\gamma,h]_c(\detlim{\step,\nlupdates}) 
- \step (\gnf[c]{\dfedavgop[\gamma,h]_c(\detlim{\step,\nlupdates}) } - \gnf[c]{\paramlim})
- \step \gnf[c]{\paramlim}
\eqsp.
\end{align*}
Using the hessian matrix of $\nfw[c]$, we write the previous identity as
\begin{align}
\label{eq:proof-det-rec-one-step-matrix-inthess}
\dfedavgop[\gamma,h+1]_c(\detlim{\step,\nlupdates}) 
& =
\dfedavgop[\gamma,h]_c(\detlim{\step,\nlupdates})
- \step \inthess[c]{\step,h} (\dfedavgop[\gamma,h]_c(\detlim{\step,\nlupdates}) - \paramlim)
- \step \gnf[c]{\paramlim}
\eqsp,
\end{align}
where $\inthess[c]{\step,h} = \int_{0}^1 \hnf[c]{t \dfedavgop[\gamma,h]_c(\detlim{\step,\nlupdates}) + (1 - t) \paramlim} \rmd t$.
Applying \eqref{eq:proof-det-rec-one-step-matrix-inthess} recursively, we have
\begin{align*}
\dfedavgop[\gamma,\nlupdates]_c(\detlim{\step,\nlupdates}) - \paramlim
& =
\detcontract[c]{\star, 1:\nlupdates} (\detlim{\step,\nlupdates} - \paramlim)
- \step \sum_{h=1}^\nlupdates \detcontract[c]{\star, h+1:\nlupdates} \gnf[c]{\paramlim}
\eqsp,
\end{align*}
where we set, for $h \in \{1,\ldots,\nlupdates\}$, the quantity
\begin{align*}
\detcontract[c]{\star, h:\nlupdates}
= 
\prod_{\ell=h}^{\nlupdates-1} \left(\Id - \step \inthess[c]{\paramlim_{c,\ell}, \paramlim}\right)
\eqsp.
\end{align*}

Averaging over all clients, we obtain
\begin{align*}
\dfedavgop[\gamma,\nlupdates](\detlim{\step,\nlupdates}) - \paramlim 
& = 
\detfullcontract{\star} (\detlim{\step,\nlupdates} - \paramlim)
- \frac{\step}{\nagent} \sum_{c=1}^\nagent
 \sum_{h=1}^\nlupdates \detcontract[c]{\star,h+1:\nlupdates} \gnf[c]{\paramlim}
\eqsp.
\end{align*}
We now use the fact that $\detlim{\step,\nlupdates}$ is the fixed point of $\dfedavgop[\gamma,\nlupdates]$, \ie, $\dfedavgop[\gamma,\nlupdates](\detlim{\step,\nlupdates}) = \detlim{\step,\nlupdates}$, and subtract $\detfullcontract{\star} (\detlim{\step, \nlupdates} - \paramlim)$ on both sides to obtain
\begin{align*}
(\Id - \detfullcontract{\star}) (\detlim{\step, \nlupdates} - \paramlim) 
& = 
- \frac{\step}{\nagent} \sum_{c=1}^\nagent
 \sum_{h=1}^\nlupdates \detcontract[c]{\star,h+1:\nlupdates} \gnf[c]{\paramlim}
 \eqsp,
\end{align*}
which gives the first part of the result after multiplying by $(\Id - \detfullcontract{\star})^{-1}$ and introducing $\Upsilon^{(\step, h)}_c = (\Id - \detfullcontract{\star})^{-1} \detcontract[c]{\star,h+1:\nlupdates}$.
Now we introduce an additional notation for 
\begin{align}
\label{eq:F_avg_star_def}
\detcontract[\text{avg}]{\star, h:\nlupdates} = \prod_{\ell=h}^{\nlupdates-1} \left(\Id - \frac{\step}{\nagent} \sum_{c=1}^\nagent \inthess[c]{\paramlim_{c,\ell}, \paramlim}\right)
\eqsp.
\end{align}
With $\detcontract[\text{avg}]{\star, h:\nlupdates}$ defined in \eqref{eq:F_avg_star_def}, we get the following identity:
\begin{align}
\detlim{\step, \nlupdates} - \paramlim \nonumber
& = 
- \frac{\step}{\nagent} (\Id - \detfullcontract{\star})^{-1} \sum_{c=1}^\nagent
 \sum_{h=1}^\nlupdates \detcontract[c]{\star,h+1:\nlupdates} \gnf[c]{\paramlim}
\\ 
& \overset{(a)}{=} 
\frac{\step}{\nagent} \sum_{c=1}^\nagent
 \sum_{h=1}^\nlupdates
 (\Id - \detfullcontract{\star})^{-1}
 (\detcontract[\text{avg}]{\star,h+1:\nlupdates} - \detcontract[c]{\star,h+1:\nlupdates} )
 \gnf[c]{\paramlim}
 \label{eq:app-exp-bias-interm}
\\
& \overset{(b)}{=} 
\frac{\step}{\nagent} \sum_{c=1}^\nagent
 \sum_{h=1}^\nlupdates
 \sum_{k=0}^\infty (\detfullcontract{\star})^{k}
 ( \detcontract[\text{avg}]{\star,h+1:\nlupdates} - \detcontract[c]{\star,h+1:\nlupdates} )
 \gnf[c]{\paramlim}
 \eqsp, \nonumber
\end{align}
where (a) comes from $\sum_{c=1}^\nagent \gnf[c]{\paramlim} = 0$, and (b) is the Neumann series. Note that 
\begin{align*}
\bnorm{ \detcontract[\text{avg}]{\star,h+1:\nlupdates} - \detcontract[c]{\star,h+1:\nlupdates} }
& =
\bnorm{ \sum_{\ell = h+1}^\nlupdates  \detcontract[\text{avg}]{\star,h+1:\ell-1} (\step \inthess[c]{\paramlim_{c,\ell}, \paramlim} - \tfrac{\step}{\nagent} \textstyle{\sum_{c'=1}^\nagent} \inthess[c']{\paramlim_{c',\ell}, \paramlim}) \detcontract[\text{avg}]{\star,\ell+1:\nlupdates}  }
\\
& \le \step 
\sum_{\ell = h+1}^\nlupdates  \bnorm{ \inthess[c]{\paramlim_{c,\ell}, \paramlim} - \tfrac{1}{\nagent} \textstyle{\sum_{c'=1}^\nagent} \inthess[c']{\paramlim_{c',\ell}, \paramlim} }
\eqsp.
\end{align*}
Thus, we have $\bnorm{ \detcontract[\text{avg}]{\star,h+1:\nlupdates} - \detcontract[c]{\star,h+1:\nlupdates} } \leq 2 \step (\nlupdates-h) \lip$.
This gives
\begin{align*}
\norm{ \detlim{\step, \nlupdates} - \paramlim }
& \le
\frac{\step}{\nagent}
 \sum_{k=0}^\infty \sum_{c=1}^\nagent
 \sum_{h=1}^\nlupdates \norm{ (\detfullcontract{\star})^{k} }
 \bnorm{ \detcontract[\text{avg}]{\star,h+1:\nlupdates} - \detcontract[c]{\star,h+1:\nlupdates} }
 \norm{ \gnf[c]{\paramlim} }
\\
& \le
\frac{\step}{\nagent}
 \sum_{k=0}^\infty
 \sum_{c=1}^\nagent
 \sum_{h=1}^\nlupdates 
 2 (1 - \step \strcvx)^{\nlupdates k}
 \step (\nlupdates-h) \lip
 \norm{ \gnf[c]{\paramlim} }
 \eqsp,
\end{align*}
where we also used that $ \norm{ \detfullcontract{\star} } \le (1 - \step \strcvx)^\nlupdates$. Consequently, when $\step \strcvx \nlupdates \le 1$, we obtain
\begin{align}
\label{eq:upper-bound-det-proof}
\norm{ \detlim{\step, \nlupdates} - \paramlim }
& \le
\frac{\step^2 \lip \nlupdates(\nlupdates-1)}{1 - (1 - \step \mu)^\nlupdates } \frac{1}{\nagent} \sum_{c=1}^\nagent
\norm{ \gnf[c]{\paramlim} }
  \le 
\frac{\step \lip (\nlupdates-1)}{\mu}
\frac{1}{\nagent} \sum_{c=1}^\nagent
\norm{ \gnf[c]{\paramlim} }
  \le 
\frac{\step \lip (\nlupdates-1)}{\mu}
\heterboundgrad
 \eqsp,
\end{align}
which is the first part of the result.
From \eqref{eq:upper-bound-det-proof}, it holds that $\norm{ \detlim{\step,\nlupdates} - \paramlim } = O(\step \nlupdates)$.
We now prove that the same result holds for the local iterates $\dfedavgop[\gamma,h] (\detlim{\step,\nlupdates})$.
Let $h \in \iint{0}{\nlupdates-1}$.
Then, using the triangle inequality and the fact that $\gf{\paramlim} = 0$, we obtain
\begin{align}
&\norm{ \dfedavgop[\gamma,h+1]_c (\detlim{\step,\nlupdates}) - \paramlim } \nonumber \\
& \qquad \quad = 
\norm{ \dfedavgop[\gamma,h]_c (\detlim{\step,\nlupdates}) - \step \gnf[c]{\dfedavgop[\gamma,h]_c (\detlim{\step,\nlupdates})} - (\paramlim - \step \gnf[c]{\paramlim}) + \step(\gnf[c]{\paramlim} - \gf{\paramlim}) } 
\nonumber \\
& \qquad \quad \le
\norm{ \dfedavgop[\gamma,h]_c (\detlim{\step,\nlupdates}) - \step \gnf[c]{\dfedavgop[\gamma,h]_c (\detlim{\step,\nlupdates})} - (\paramlim - \step \gnf[c]{\paramlim}) } + \step \norm{ \gnf[c]{\paramlim} - \gf{\paramlim} } 
\eqsp.
\label{eq:proof-bias-det-fedavg-crude}
\end{align}
Applying \Cref{lem:contract-op-fedavg-loc} and \eqref{eq:proof-bias-det-fedavg-crude} recursively, then \Cref{assum:heterogeneity}, we obtain
\begin{align*}
\norm{ \dfedavgop[\gamma,h+1]_c (\detlim{\step,\nlupdates}) - \paramlim } 
& \le
\norm{ \dfedavgop[\gamma,h]_c (\detlim{\step,\nlupdates}) - \paramlim } 
+ \step  \norm{ \gnf[c]{\paramlim} - \gf{\paramlim} }
\le
\norm{ \detlim{\step,\nlupdates} - \paramlim } 
+ \step \nlupdates \heterboundgrad
=
O(\step \nlupdates)
\eqsp,
\end{align*}
which proves the second part of the result.
\end{proof}

\corollaryconvergenceratefedavgwdg*
\begin{proof}
We start with the upper bound
\begin{align*}
\norm{ \paramw_t - \paramlim }^2 
& \le
2 \norm{ \paramw_t - \detlim{\step,\nlupdates} }^2
+ 2 \norm{\detlim{\step,\nlupdates} - \paramlim}^2
\eqsp.
\end{align*}
Then, we apply \Cref{prop:convergence-fedavg-to-point} to bound the first term, and \Cref{prop:bias-det-fedavg} to bound the second term.
\end{proof}

\subsection{Expansion of the Bias -- Proof of \Cref{cor:first-order-det}}
\label{sec:proof-thm-expand-deterministic}

\begin{theorem}[Expansion of \FedAvgwDG's Bias, Restated from \Cref{cor:first-order-det}]
\label{prop:app-exp-bias-wdg}
Assume 
\Cref{assum:local_functions}, \Cref{assum:heterogeneity}.
Let $\nlupdates > 0$, $\step \le 1/\lip$ such that $\step \strcvx \nlupdates \le 1$, then the bias of \FedAvgwDG can be expanded as
  \begin{align*}
    \detlim{\step, \nlupdates}
    - \paramlim
    & =
      \frac{\step(\nlupdates-1)}{2\nagent} 
      \hf{\paramlim}^{-1} 
      \sum_{c=1}^\nagent 
      (\hnf[c]{\paramlim} - \hf{\paramlim})
      \gnf[c]{\paramlim}
      + \step \nlupdates \reste[]{}{\detlim{\step, \nlupdates}}
      \eqsp,
  \end{align*}
  where the expression of $\reste[]{}{\detlim{\step, \nlupdates}} = O(\step \nlupdates)$ is given in \eqref{eq:expression-remainder-fedavg-det}.
\end{theorem}
\begin{proof}
Starting from \eqref{eq:app-exp-bias-interm}, we have
  \begin{align}
  \label{eq:app-expansion-diff-bias-diff}
    \detlim{\step, \nlupdates}
    - \paramlim
    & =
    \frac{\step}{\nagent} \sum_{c=1}^\nagent
 \sum_{h=1}^\nlupdates
 (\Id - \detfullcontract{\star})^{-1}
 (\detcontract[\text{avg}]{\star,h+1:\nlupdates} - \detcontract[c]{\star,h+1:\nlupdates} )
 \gnf[c]{\paramlim}
      \eqsp.
  \end{align}
  We start by writing the expansion of $\inthess[c]{\step,h}$. Note that, for $t \in (0,1)$, we can write
  \[
  t \dfedavgop[\gamma,h]_c(\detlim{\step,\nlupdates}) + (1 - t) \paramlim = \paramlim + t (\dfedavgop[\gamma,h]_c(\detlim{\step,\nlupdates}) - \paramlim)\eqsp.
  \]
  Thus, we can expand the Hessian
  \[
  \hnf[c]{t \dfedavgop[\gamma,h]_c(\detlim{\step,\nlupdates}) + (1 - t) \paramlim} = \hnf[c]{\paramlim} + \resteint[c]{1,h,t}{\dfedavgop[\gamma,h]_c(\detlim{\step,\nlupdates}) }\eqsp,
  \]
  where $\resteintw[c]{1,h,t} : \rset^d \rightarrow \rset^d$ is such that $\sup_{\varparam \in \rset^d} \norm{ \resteint[c]{1,h,t}{\varparam} } / \norm{ \varparam - \paramlim } < + \infty$. Hence, combining this bound and the definition of $\inthess[c]{\step,h}$, we obtain
  \begin{align*}
    \inthess[c]{\step,h}
    & = \int_{0}^1 \left\{ \hnf[c]{\paramlim} + \resteint[c]{1,h,t}{\dfedavgop[\gamma,h]_c(\detlim{\step,\nlupdates}) }  \right\} \rmd t
    = \hnf[c]{\paramlim} + \resteint[c]{1,h}{\dfedavgop[\gamma,h]_c(\detlim{\step,\nlupdates})}
    \eqsp,
  \end{align*}
  where $\displaystyle \resteintw[c]{1,h}: \varparam \mapsto \int_{0}^1 \left\{  \resteint[c]{1,h,t}{\varparam - \paramlim}  \right\} \rmd t$ is such that 
  \begin{align}
  \label{eq:proof-exp-det-bound-reste1h}
    \sup_{\varparam \in \rset^d} {\norm{ \resteint[c]{1,h}{\varparam} }} / {\norm{ \varparam - \paramlim }} < + \infty
    \eqsp.
  \end{align}
  Using \eqref{eq:proof-exp-det-bound-reste1h} and \Cref{prop:bias-det-fedavg}, we can expand
  $\detcontract[c]{\star,h+1:\nlupdates} =
    \prod_{\ell=h}^{\nlupdates-1} \left(\Id - \step \inthess[c]{\paramlim_{c,\ell}, \paramlim}\right)$ and $(\Id - \fullcontract{\star})^{-1}$ as
  \begin{align*}
    \detcontract[c]{\star,h+1:\nlupdates} 
    & =
    \Id - \step (\nlupdates - h - 1) \hnf[c]{\paramlim} + \step \nlupdates \reste[c]{1,h}{\detlim{\step,\nlupdates}}
    \eqsp,
    \\
    \detcontract[\text{avg}]{\star,h+1:\nlupdates} 
    & =
    \Id - \step (\nlupdates - h - 1) \hf{\paramlim} + \step \nlupdates \reste[]{1,h}{\detlim{\step,\nlupdates}}
    \eqsp,
    \\
      (\Id - \fullcontract{\star})^{-1}
      & =
        (\step \nlupdates \hf{\paramlim})^{-1} + \reste[]{1}{\dfedavgop[\gamma,h]_c(\detlim{\step,\nlupdates})}
        \eqsp,
  \end{align*}
  where $\restew[c]{1,h}: \rset^d \rightarrow \rset^{d \times d}$, $\restew[]{1,h} = \frac{1}{\nagent} \sum_{c=1}^\nagent \restew[c]{1,h}$, and $\restew[]{1}: \rset^d \rightarrow \rset^{d \times d}$ are such that 
  \begin{align}
  \label{eq:proof-det-last-remainder}
  \sup_{\varparam \in \rset^d} \norm{ \reste[c]{1,h}{\varparam} } / \norm{ \varparam - \paramlim } < + \infty~,~~ \text{and } 
  \sup_{\varparam \in \rset^d} \norm{ \reste[]{1}{\varparam} } / \norm{ \varparam - \paramlim } < + \infty
  \eqsp.
  \end{align}
  Plugging the three above identities in \eqref{eq:app-expansion-diff-bias-diff}, we obtain
  \begin{align*}
    \detlim{\step, \nlupdates}
    - \paramlim
    & =
    \frac{\step}{\nagent} \sum_{c=1}^\nagent
    \sum_{h=1}^\nlupdates
     \left\{ (\step \nlupdates \hf{\paramlim})^{-1} + \reste[]{1}{\detlim{\step,\nlupdates}} \right\}
     \\ \nonumber
    & \qquad \quad \quad 
    \times \left\{ \step (\nlupdates - h - 1) ( \hnf[c]{\paramlim} - \hf{\paramlim}) + \step \nlupdates (\reste[]{1,h}{\detlim{\step,\nlupdates}} - \reste[c]{1,h}{\detlim{\step,\nlupdates}}) ) \right\}
    \gnf[c]{\paramlim}
    \\
    & =
      \frac{\step}{\nagent \nlupdates} 
      \sum_{c=1}^\nagent 
      \sum_{h=1}^\nlupdates (\nlupdates - h - 1)
      \hf{\paramlim}^{-1} 
        (\hnf[c]{\paramlim} - \hf{\paramlim})
      \gnf[c]{\paramlim}
      + \step \nlupdates \reste[]{}{\detlim{\step, \nlupdates}}
      \eqsp,
  \end{align*}
  where
  \begin{equation}
  \label{eq:expression-remainder-fedavg-det}
  \begin{aligned}
      \reste[]{}{\detlim{\step,\nlupdates}}
      & =
      \frac{1}{\nagent \nlupdates}
      \sum_{c=1}^\nagent
      \sum_{h=1}^\nlupdates
      \hf{\paramlim}^{-1}  (\reste[]{1,h}{\detlim{\step,\nlupdates}} - \reste[c]{1,h}{\detlim{\step,\nlupdates}})\gnf[c]{\paramlim}
      \\ 
      & \quad + 
      \frac{1}{\nagent\nlupdates}
      \sum_{c=1}^\nagent
      \sum_{h=1}^\nlupdates
      \step(\nlupdates-h-1) \reste[]{1}{\detlim{\step,\nlupdates}} (\hnf[c]{\paramlim} - \hf{\paramlim})\gnf[c]{\paramlim}
      \\ 
      & \quad
      +
      \frac{1}{\nagent\nlupdates}
      \sum_{c=1}^\nagent
      \sum_{h=1}^\nlupdates
      \step\nlupdates \reste[]{1}{\detlim{\step,\nlupdates}} (\reste[]{1,h}{\detlim{\step,\nlupdates}} - \reste[c]{1,h}{\detlim{\step,\nlupdates}})\gnf[c]{\paramlim}
      \eqsp.
  \end{aligned}
  \end{equation}  
  Since $\sum_{h=1}^\nlupdates h = \frac{\nlupdates(\nlupdates-1)}{2}$, we obtain from above identities that 
  \begin{align*}
    \detlim{\step, \nlupdates}
    - \paramlim
    & =
      \frac{\step(\nlupdates-1)}{2\nagent} 
      \sum_{c=1}^\nagent 
      \hf{\paramlim}^{-1} 
      (\hnf[c]{\paramlim} - \hf{\paramlim})
      \gnf[c]{\paramlim}
      + \step \nlupdates \reste[]{}{\detlim{\step, \nlupdates}}
      \eqsp.
  \end{align*}
  The result follows from \eqref{eq:proof-det-last-remainder}, which ensures that $\sup_{\varparam \in \rset^d} \norm{ \reste[]{}{\varparam} } / \norm{\varparam - \paramlim } < +\infty$, and \Cref{prop:bias-det-fedavg}, which gives $\norm{ \detlim{\step, \nlupdates} - \paramlim } = O(\step \nlupdates)$ and thus the upper bound on the remainder $\step \nlupdates \reste[]{}{\detlim{\step, \nlupdates}} = O (\step^2 \nlupdates^2)$.
\end{proof}

%% file: 2024-AISTATS/src/aistats-app-analysis-fedavg-sto.tex
\subsection{Convergence to a Stationary Distribution -- Proof of \Cref{prop:conv-stat-dist}}
\label{sec:proof-stat-dist}
\input{2024-AISTATS/src/aistats-app-analysis-fedavg-sto-transient.tex}

\subsection{Crude Bounds on \FedAvg's Convergence}
\label{sec:crude-bounds-sto}
\input{2024-AISTATS/src/aistats-app-crude-bounds-sto}

\subsection{Quadratic Setting -- Proof of \Cref{thm:bias-quadratic-sto}}
\label{sec:proof-quadratic}
\input{2024-AISTATS/src/aistats-app-analysis-fedavg-sto-quad.tex}

\subsection{General Functions, with Homogeneous Agents -- Proof of \Cref{thm:exp-homogeneous}}
\label{sec:proof-homogeneous}
\input{2024-AISTATS/src/aistats-app-analysis-fedavg-sto-homogeneous.tex}

\subsection{General Functions, with Heterogeneous Agents -- Proof of \Cref{thm:bias-var-heterogeneous}}
\label{sec:proof-heterogeneous}
\input{2024-AISTATS/src/aistats-app-analysis-fedavg-sto-heterogeneous.tex}

%% file: 2024-AISTATS/src/aistats-app-analysis-fedavg-sto-transient.tex
In the stochastic setting, we recall the following operators that generate the iterates of \FedAvg. That is, for $\paramw \in \rset^d$, we let
\begin{align*}
\sdfedavgop[c]{\step,0}(\paramw) 
& \eqdef \paramw 
\eqsp,
\\
\sdfedavgop[c]{\step,h+1}(\paramw; \randState[c][1:h+1])
& \eqdef
\sdfedavgop[c]{\step,h}(\paramw; \randState[c][1:h])
- \step \gnfs[c]{\sdfedavgop[c]{\step,h}(\paramw; \randState[c][1:h])}{\randState[c][h+1]}
\eqsp,
\end{align*}
and define the global update 
\[
\sdfedavgop{\step,H}\left(\paramw; \randState[1:\nagent][1:\nlupdates]\right) 
  \eqdef  \frac{1}{\nagent} \sum_{c=1}^{\nagent} \sdfedavgop[c]{\step,\nlupdates}(\paramw; \randState[1:\nlupdates][c])\eqsp.
\]
Here $\randState[1:\nagent][1:\nlupdates]= \{ \randState[\tilde{c}][\tilde{h}] : \tilde{c} \in \iint{1}{\nagent}, \tilde{h} \in \iint{1}{\nlupdates}\}$ is a sequence of independent random variable, such that $\randState[\tilde{c}][\tilde{h}]$ has distribution $\xic[\tilde{c}]$.
Additionally, \FedAvg's global updates are of the form $\param[t+1] = \param[t] - \step \pseudograds[\step,\nlupdates]{\param[t]}{\randState[1:\nagent][1:\nlupdates]}$, where
\begin{align*}
  \pseudograds[\step,\nlupdates]{\param}{\randState[1:\nagent][1:\nlupdates]}
  & =
    \frac{1}{\nagent} \sum_{c=1}^\nagent \sum_{h=0}^{\nlupdates-1}  \gnfs[c]{\sdfedavgop[c]{\step,h}(\paramw; \randState[c][1:h])}{\randState[h+1]}
    \eqsp,
\end{align*}
where $\paramw_{c,0}(\globRandState), \paramw_{c,1}(\globRandState), \ldots, \paramw_{c,\nlupdates}(\globRandState)$ is the sequence obtained using the stochastic local update rule, and $\globRandState = (\randState[1], \dots, \randState[\nlupdates])$ is a sequence of i.i.d. random variables.

Contrarily to \FedAvgwDG, the stochastic variant of \FedAvg does not converge to a single point. Thus, we rather study the convergence of its global iterates to a stationary distribution.
To this end, we start with the following two lemma, that are analogous to \Cref{lem:contract-op-fedavg-loc} and \Cref{lem:contract-op-fedavg} in the stochastic setting.
\begin{lemma}[Contraction of \FedAvg's Local Iterates]
\label{lem:contract-op-fedavg-loc-sto}
Assume \Cref{assum:local_functions}. Let $\paramw, \varparam$ be random vectors, $\mcF$ be a $\sigma$-algebra, such that $\paramw, \varparam$ are $\mcF$-measurable. Moreover, let $c \in \iint{1}{\nagent}$ and $\randState[c] \sim \xic$ be independent of $\mcF$. Then for any $\step \le 1/\lip$, it holds that 
\begin{align*}
\PE\left[ \norm{ (\paramw -  \step \gnfs[c]{\paramw}{\randState[c]})
- (\varparam -  \step \gnfs[c]{\varparam}{\randState[c]})}^2 \right]
\le
(1 - \step \mu) \PE\left[ \norm{ \param - \varparam }^2 \right]
\eqsp.
\end{align*}
\begin{proof}
We start by expanding the norm as
  \begin{align*}
  & \norm{ (\paramw -  \step \gnfs[c]{\paramw}{\randState[c]})
- (\varparam -  \step \gnfs[c]{\varparam}{\randState[c]})}^2
  \\
  & \quad
  = 
     \norm{\paramw - \varparam}^2
      + \step^2 \norm{ \gnfs[c]{\paramw}{\globRandState[c]}
      - \gnfs[c]{\varparam}{\globRandState[c]} }^2 
      - 2 \step \pscal{\paramw - \varparam }{\gnfs[c]{\paramw}{\globRandState[c]}
      - \gnfs[c]{\varparam}{\globRandState[c]} }
      \eqsp.
  \end{align*}
  By co-coercivity \Cref{assum:local_functions}-\ref{assum:smoothness}, we have
  \begin{align*}
     \CPE{\step^2 \norm{ \gnfs[c]{\param}{\globRandState[c]}
      - \gnfs[c]{\varparam}{\globRandState[c]} }^2}{\mcF}
     \leq
      \explip \step^2 \pscal{\paramw - \varparam }{\gnf[c]{\paramw} - \gnf[c]{\varparam} } 
      \eqsp.
  \end{align*}
  Then, strong convexity \Cref{assum:local_functions}-\ref{assum:item_strong_convex} gives
  \begin{align*}
    & \CPE{- \step \pscal{\param
      - \varparam }{\gnfs[c]{\param}{\globRandState[c]}
      - \gnfs[c]{\varparam}{\globRandState[c]} }}{\mcF}
     =
      - \step \pscal{\param
      - \varparam }{\gnf[c]{\param}
      - \gnf[c]{\varparam} }
      \le
      - \step \strcvx \norm{ \param
      - \varparam }^2
      \eqsp.
  \end{align*}
  Combining the above inequalities, we obtain
  \begin{align*}
    & \CPE{ \norm{ (\paramw -  \step \gnfs[c]{\paramw}{\randState[c]})
- (\varparam -  \step \gnfs[c]{\varparam}{\randState[c]}) }^2 }{\mcF}
   \le
      (1 - \step \strcvx) \norm{ \param
      - \varparam }^2
      - 2\step (1 - \explip \step/2) \pscal{\param
      - \varparam }{\gnf[c]{\param}
      - \gnf[c]{\varparam} }
      \eqsp,
      \nonumber
  \end{align*}
  and the result follows from $\step \le 1/\explip$ and the tower property of conditional expectations.
\end{proof}
\end{lemma}
\begin{lemma}[Contraction of \FedAvg's Global Updates]
  \label{lem:contract-op-fedavg-sto}
  Assume \Cref{assum:local_functions}. Let $\nlupdates > 0$ and $\randState[1:\nagent][1:\nlupdates] = \{ \randState[\tilde{c}][\tilde{h}] : \tilde{c} \in \iint{1}{\nagent}, \tilde{h} \in \iint{1}{\nlupdates}\}$ be a sequence of independent random variable, such that $\randState[\tilde{c}][\tilde{h}]$ has distribution $\xic[\tilde{c}]$.
  Let $\mcF$ be a sub-$\sigma$-algebra and $\param, \varparam \in \rset^d$ be two $\mcF$-measurable random variables.
  Then for the operator $\sdfedavgop[c]{\step,\nlupdates}(\cdot; \randState[1:\nagent][1:\nlupdates])$ it holds, for $\step \leq 1/\lip$, that 
  \begin{align*}
    \PE\left[ \norm{ \sdfedavgop[c]{\step,\nlupdates}(\paramw; \randState[1:\nagent][1:\nlupdates]) - \sdfedavgop[c]{\step,\nlupdates}(\varparam; \randState[1:\nagent][1:\nlupdates]) }^2 \right]
    \le
    (1 - \step \mu)^\nlupdates \PE\left[ \norm{ \param - \varparam }^2 \right]
    \eqsp.
  \end{align*}
\end{lemma}
\begin{proof}
  First, remark that
  \begin{align*}
    & \sdfedavgop[c]{\step,h+1}(\param ; \randState[c][1:h+1])     - \sdfedavgop[c]{\step,h+1}(\varparam ; \randState[c][1:h+1])
    \\
    & \quad = 
    (\sdfedavgop[c]{\step,h}(\param ; \randState[c][1:h])
    - \step (\gnfs[c]{\sdfedavgop[c]{\step,h}(\param ; \randState[c][1:h])}{\globRandState[h]})
    - (\sdfedavgop[c]{\step,h}(\varparam ; \randState[c][1:h])
    - \step\gnfs[c]{\sdfedavgop[c]{\step,h}(\varparam ; \randState[c][1:h])}{\globRandState[h]})
    \eqsp.
  \end{align*}
  Therefore, by \Cref{lem:contract-op-fedavg-loc-sto}, we have
  \begin{align*}
    \PE\left[ \norm{ \sdfedavgop[c]{\step,h+1}(\param ; \randState[c][1:h+1])     - \sdfedavgop[c]{\step,h+1}(\varparam ; \randState[c][1:h+1]) }^2 \right]
    \le
    (1 - \step \strcvx)
    \PE\left[ \norm{ \sdfedavgop[c]{\step,h}(\param ; \randState[c][1:h])     - \sdfedavgop[c]{\step,h}(\varparam ; \randState[c][1:h]) }^2 \right]\eqsp.
  \end{align*}
  Thus, using this inequality $\nlupdates$ times recursively, together with Jensen's inequality, we obtain
  \begin{align*}
    \PE \left[ \norm{ \sdfedavgop{\step,\nlupdates}(\param ; \randState[1:\nagent][1:\nlupdates])  - \sdfedavgop{\step,\nlupdates}(\varparam ; \randState[1:\nagent][1:\nlupdates])  }^2\right]
    & \le
      \frac{1}{\nagent} \sum_{c=1}^\nagent \PE \left[ \norm{ \sdfedavgop[c]{\step,\nlupdates}(\param ; \randState[c][1:\nlupdates])  - \sdfedavgop[c]{\step,\nlupdates}(\varparam ; \randState[c][1:\nlupdates])  }^2 \right]
      \\
      &
    \le
      (1 - \step \mu)^\nlupdates \PE\left[ \norm{ \param - \varparam }^2 \right]
      \eqsp,
  \end{align*}
  which implies the statement.%
\end{proof}

We now use the above lemma to show that the iterates of \FedAvg converge to a stationary distribution.
\convstatdistfedavg*
\begin{proof}
  The proof is similar to \citet[Proposition 2]{dieuleveut2020bridging}, but we give it for completeness.
  Let $\lambda_1, \lambda_2$ be two probability measures on $\rset^d$. 
  By \citet{villani2009optimal}, Theorem 4.1, there exists two random variables $\paramw_{0}$ and $\varparam_0$ such that
  \begin{align*}
    \wasserstein^2(\lambda_1, \lambda_2)
    & =
      \PE\left[
      \norm{ \paramw_{0} - \varparam_0 }^2
      \right]
      \eqsp.
  \end{align*}
  For $t \ge 0$, let $\randState[1:\nagent,t][1:\nlupdates]= \{ \randState[\tilde{c},t][\tilde{h}] : \tilde{c} \in \iint{1}{\nagent}, \tilde{h} \in \iint{1}{\nlupdates}, \}$ is a sequence of independent random variables, such that $\randState[\tilde{c}, t][\tilde{h}]$ has distribution $\xic[\tilde{c}]$, and define recursively the two sequences for $t \ge 0$, 
  \begin{align*}
    \param[t+1]
    = 
     \sdfedavgop{\step,\nlupdates}(\param[t]; \randState[1:\nagent,t][1:\nlupdates])
    \eqsp,
    \qquad
    \varparam_{t+1} 
    = 
     \sdfedavgop{\step,\nlupdates}(\varparam_{t}; \randState[1:\nagent,t][1:\nlupdates])
    \eqsp,
  \end{align*}
  corresponding to two trajectories of \FedAvg, sampled with the same noise but with different initializations.
  In the following, we use the filtration $\mcF_{t} = \sigma \{ \randState[1:\nagent,s][1:\nlupdates] : s \le t \}$.
  By the definition of the Wasserstein distance, and using \Cref{lem:contract-op-fedavg-sto}, we obtain, for any $k \ge 0$,
  \begin{align*}
    \wasserstein^2(\lambda_1 \markovkernel^{t}, \lambda_2 \markovkernel^{t})
    & \le
      \PE\left[ 
      \norm{ \param[t] - \varparam_{t} }^2
      \right]
    \\
    & =
    \PE\left[
      \CPE{
      \norm{  \sdfedavgop{\step,\nlupdates}(\param[t-1]; \randState[1:\nagent,t][1:\nlupdates]) -  \sdfedavgop{\step,\nlupdates}(\varparam_{t-1}; \randState[1:\nagent,t-1][1:\nlupdates])}^2
      }{\mcF_{t-1}}
      \right]
    \\
    & \le
      (1 - \step \strcvx)^\nlupdates
      \PE\left[ \norm{ \param[t-1] - \varparam_{t-1}}^2 \right]
      \eqsp.
  \end{align*}
  Applying \Cref{lem:contract-op-fedavg-sto} resursively, we obtain 
  \begin{align*}
    \wasserstein^2(\lambda_1 \markovkernel^{t}, \lambda_2 \markovkernel^{t})
    & \le
      (1 - \step \strcvx)^{\nlupdates t}
      \norm{ \param[0] - \varparam_0}^2
      =
      (1 - \step \strcvx)^{\nlupdates t}
      \wasserstein^2(\lambda_1, \lambda_2)
      \eqsp.
  \end{align*}
  Taking $\lambda_2 = \lambda_1 \markovkernel$, this implies that 
  \begin{align*}
    \wasserstein^2(\lambda_1 \markovkernel^{t}, \lambda_1 \markovkernel^{t+1})
    & \le
      (1 - \step \strcvx)^{\nlupdates t}
      \wasserstein^2(\lambda_1, \markovkernel \lambda_1)
      \eqsp,
  \end{align*}
  which guarantees that $(\lambda_1 \markovkernel^t)_{t \ge 0}$ is a Cauchy sequence with values in the space probability distributions on $\rset^d$ that have a second moment. 
  Consequently, this series has a limit $\statdist{\step,\nlupdates}_{\lambda_1}$ that may depend on $\lambda_1$.

  We now show that this distribution is independent from the initial distribution. Indeed, take $\lambda_1$ and $\lambda_2$ with associated limit distributions $\statdist{\step,\nlupdates}_{\lambda_1}$ and $\statdist{\step,\nlupdates}_{\lambda_2}$. Then, by triangle inequality, we have, for any $t \ge 0$,
  \begin{align*}
      \wasserstein^2(\statdist{\step,\nlupdates}_{\lambda_1}, \statdist{\step,\nlupdates}_{\lambda_2})
      & \le
      \wasserstein^2(\statdist{\step,\nlupdates}_{\lambda_1}, \lambda_1 \markovkernel^{t+1})    
      +
      \wasserstein^2(\lambda_1 \markovkernel^{t}, \lambda_2 \markovkernel^{t+1})
      +
      \wasserstein^2(\lambda_2 \markovkernel^{t},\statdist{\step,\nlupdates}_{\lambda_2})
      \eqsp,
  \end{align*}
  which gives $\wasserstein^2(\statdist{\step,\nlupdates}_{\lambda_1}, \statdist{\step,\nlupdates}_{\lambda_2}) = 0$ by taking the limit as $t \rightarrow +\infty$.
  Thus, $\statdist{\step,\nlupdates}_{\lambda_1} = \statdist{\step,\nlupdates}_{\lambda_2}$ and the limit distribution is unique, and we denote it $\statdist{\step, \nlupdates}$.
  Similarly, we remark that for any probability distribution $\lambda$ on $\rset^d$, and for all $t \ge 0$, it holds that
    \begin{align*}
      \wasserstein^2(\statdist{\step,\nlupdates} \markovkernel, \statdist{\step,\nlupdates}
      & \le
      \wasserstein^2(\statdist{\step,\nlupdates} \markovkernel, \statdist{\step,\nlupdates} \markovkernel^t)    
      +
      \wasserstein^2(\statdist{\step,\nlupdates} \markovkernel^t, \statdist{\step,\nlupdates} \markovkernel)
      \eqsp,
  \end{align*}
  and taking the limit as $t \rightarrow + \infty$, we obtain that $\statdist{\step,\nlupdates} \markovkernel = \statdist{\step,\nlupdates}$, which guarantees that it is a stationary distribution.
\end{proof}

%% file: 2024-AISTATS/src/aistats-app-crude-bounds-sto.tex
In this section, we give crude bounds on the moments of \FedAvg's stationary distribution, that will be used to bound higher-order terms in the expansions below.

\subsubsection{Homogeneous Functions}
For homogeneous functions, we can prove that the errors of \FedAvg's global and local iterates at stationarity are of order $O(\step)$.
This is stated in the next lemma, whose proof follows the lines of classical analysis of \SGD, but only uses the fact that gradients $\gnfw[c]$'s at solution have the same expectation.
\begin{lemma}[Crude Bound, Homogeneous Functions]
\label{lem:crude-bound-second-moment-homogeneous}
Assume \Cref{assum:local_functions}, \Cref{assum:smooth-var}, and let \Cref{assum:heterogeneity} holds with $\heterboundgrad=0$.
Let $\step \le 1/(2\explip)$, and $\step \strcvx \nlupdates \le 1$, then
\begin{align*}
\PE[ \norm{ \param[t] - \paramlim}^2 ]
\le
(1 - 2 \step\strcvx(1 - \step \explip))^{\nlupdates t} \PE[ \norm{ \param[0] - \paramlim}^2 ]
+ \frac{\step}{\strcvx (1 - \step \explip)} \Msmoothcstvar^{2}
\eqsp.
\end{align*}
This implies that, for $\paramw \sim \statdist{\step, \nlupdates}$, where $\statdist{\step, \nlupdates}$ is the stationary distribution of \FedAvg with step size $\step$ and $\nlupdates$ local updates, it holds that
\begin{align*}
\int \norm{ \param - \paramlim }^2 \statdist{\step, \nlupdates}(\rmd \paramw) = O(\step)
~,
~~
\text{ and }
\int \norm{ \sdfedavgop[c]{\step,h}(\paramw; \randState[c][1:h]) - \paramlim}^2  \statdist{\step, \nlupdates}(\rmd \paramw)
= O(\step)
\eqsp,
\end{align*}
where $\randState[c][1:\nlupdates] = \{ \randState[{c}][\tilde{h}] : \tilde{h} \in \iint{1}{\nlupdates}\}$ is a sequence of independent random variable, with $\randState[{c}][\tilde{h}] \sim \xic[{c}]$.
\end{lemma}
\begin{remark}
\Cref{lem:crude-bound-second-moment-homogeneous} only assumes that $\gnf[c]{\paramlim} = 0$ for all $c \in \iint{1}{\nagent}$.
This notably holds under \Cref{assum:homogeneity}, but is in fact a stronger result.
\end{remark}
\begin{proof}
First, we rewrite the local updates of \FedAvg, for $c \in \iint{1}{\nagent}$ and $h \in \iint{0}{\nlupdates-1}$,
\begin{align*}
\sdfedavgop[c]{\step,h+1}(\paramw; \randState[c][1:h+1])
& =
\sdfedavgop[c]{\step,h}(\paramw; \randState[c][1:h])
- \step \gnfs[c]{\sdfedavgop[c]{\step,h}(\paramw; \randState[c][1:h])}{\randState[c][h+1]}
\eqsp.
\end{align*}
Thus, we have
\begin{align*}
& \norm{ \sdfedavgop[c]{\step,h+1}(\paramw; \randState[c][1:h+1]) - \paramlim}^2
\\
& \quad = 
\norm{ \sdfedavgop[c]{\step,h}(\paramw; \randState[c][1:h]) - \paramlim}^2
- 2 \step \pscal{ \gnfs[c]{\sdfedavgop[c]{\step,h}(\paramw; \randState[c][1:h])}{\randState[c][h+1]} }{ \sdfedavgop[c]{\step,h}(\paramw; \randState[c][1:h]) - \paramlim }
+ \norm{ \gnfs[c]{\sdfedavgop[c]{\step,h}(\paramw; \randState[c][1:h])}{\randState[c][h+1]} }^2
\eqsp.
\end{align*}
Decomposing the gradient of $\gnfs[c]{\sdfedavgop[c]{\step,h}(\paramw; \randState[c][1:h])}{\randState[c][h+1]}$ using the fact that, since $\heterboundgrad = 0$, the functions $\nfw[c]$'s satisfy $\gnf[c]{\paramlim} = 0$, we obtain
\begin{align*}
& \gnfs[c]{\sdfedavgop[c]{\step,h}(\paramw; \randState[c][1:h])}{\randState[c][h+1]}
=
\gnfs[c]{\sdfedavgop[c]{\step,h}(\paramw; \randState[c][1:h])}{\randState[c][h+1]} - \gnfs[c]{\paramlim}{\randState[c][h+1]}
+ \gnfs[c]{\paramlim}{\randState[c][h+1]} - \gnf[c]{\paramlim}
\eqsp,
\end{align*}
and using Young's inequality, we obtain
\begin{align*}
\norm{ \sdfedavgop[c]{\step,h+1}(\paramw; \randState[c][1:h+1]) - \paramlim}^2
& \le
\norm{ \sdfedavgop[c]{\step,h}(\paramw; \randState[c][1:h]) - \paramlim}^2
- 2 \step \pscal{ \gnfs[c]{\sdfedavgop[c]{\step,h}(\paramw; \randState[c][1:h])}{\randState[c][h+1]} }{ \sdfedavgop[c]{\step,h}(\paramw; \randState[c][1:h]) - \paramlim }
\\
&  
+ 2 \step^2 \norm{ \gnfs[c]{\sdfedavgop[c]{\step,h}(\paramw; \randState[c][1:h])}{\randState[c][h+1]} - \gnfs[c]{\paramlim}{\randState[c][h+1]} }^2
+ 2 \step^2 \norm{ \gnfs[c]{\paramlim}{\randState[c][h+1]} - \gnf[c]{\paramlim} }^2
\eqsp.
\end{align*}
Now, we define the filtration $\mcF_c^h = \sigma( \randState[c][\ell] : \ell \le h )$, and take the conditional expectation to obtain
\begin{align*}
\CPE{ \norm{ \sdfedavgop[c]{\step,h+1}(\paramw; \randState[c][1:h+1]) - \paramlim}^2 }{\mcF_c^h}
& \le
\norm{ \sdfedavgop[c]{\step,h}(\paramw; \randState[c][1:h]) - \paramlim}^2
- 2 \step \pscal{ \gnf[c]{\sdfedavgop[c]{\step,h}(\paramw; \randState[c][1:h])} }{ \sdfedavgop[c]{\step,h}(\paramw; \randState[c][1:h]) - \paramlim }
\\
& \quad 
+ 2  \step^2\CPE{ \norm{ \gnfs[c]{\sdfedavgop[c]{\step,h}(\paramw; \randState[c][1:h])}{\randState[c][h+1]} - \gnfs[c]{\paramlim}{\randState[c][h+1]} }^2 }{\mcF_c^h}
\\
& \quad
+ 2 \step^2 \CPE{ \norm{ \gnfs[c]{\paramlim}{\randState[c][h+1]} - \gnf[c]{\paramlim} }^2 }{\mcF_c^h}\eqsp.
\end{align*}
By \Cref{assum:local_functions}-\ref{assum:item_strong_convex}, \Cref{assum:local_functions}-\ref{assum:smoothness}, and using that $\gnf[c]{\paramlim} = 0$, we have
\begin{align} \nonumber
& \CPE{ \norm{ \sdfedavgop[c]{\step,h+1}(\paramw; \randState[c][1:h+1]) - \paramlim}^2 }{\mcF_c^h}
\\ \nonumber
& \quad \le
\norm{ \sdfedavgop[c]{\step,h}(\paramw; \randState[c][1:h]) - \paramlim}^2
- 2\step(1 - \step \explip) \pscal{ \gnf[c]{\sdfedavgop[c]{\step,h}(\paramw; \randState[c][1:h])} }{ \sdfedavgop[c]{\step,h}(\paramw; \randState[c][1:h]) - \paramlim }
\\ \nonumber
& \qquad
+ 2 \step^2 \CPE{ \norm{ \gnfs[c]{\paramlim}{\randState[c][h+1]} - \gnf[c]{\paramlim} }^2 }{\mcF_c^h}
\\ 
& \quad \le
(1 - 2 \step\strcvx(1 - \step \explip)) \norm{ \sdfedavgop[c]{\step,h}(\paramw; \randState[c][1:h]) - \paramlim}^2
+ 2 \step^2 \CPE{ \norm{ \gnfs[c]{\paramlim}{\randState[c][h+1]} - \gnf[c]{\paramlim} }^2 }{\mcF_c^h}
\label{eq:proof-app-crude-bound-rec-local-homogeneous}
\eqsp.
\end{align}
Using \eqref{eq:def-epsilon} together with the fact that the $\randState[c][h]$'s are i.i.d., taking the expectation and unrolling \eqref{eq:proof-app-crude-bound-rec-local-homogeneous}, we obtain
\begin{align*}
& \PE[ \norm{ \sdfedavgop[c]{\step,\nlupdates}(\paramw; \randState[c][1:\nlupdates]) - \paramlim}^2 ]
 \le
(1 - 2 \step\strcvx(1 - \step \explip))^\nlupdates \PE[ \norm{ \paramw - \paramlim}^2 ]
+ 2 \step^2 \nlupdates \PE[ \norm{ \updatefuncnoise[c]{\randState[c][1]}(\paramlim) }^2 
\eqsp.
\end{align*}
Therefore, using Jensen's inequality, \Cref{assum:heterogeneity} and \Cref{assum:smooth-var}, we obtain the following bound:
\begin{align}
\label{eq:proof-app-crude-bound-rec}
\PE[ \norm{ \sdfedavgop{\step,\nlupdates}(\paramw; \randState[1:\nagent][1:\nlupdates]) - \paramlim}^2 ]
\le
(1 - 2 \step\strcvx(1 - \step \explip))^\nlupdates \PE[ \norm{ \paramw - \paramlim}^2 ]
+ 2 \step^2 \nlupdates \Msmoothcstvar^{2}
\eqsp.
\end{align}
Denoting $\param[t]$ the global iterates of \FedAvg, and using \eqref{eq:proof-app-crude-bound-rec} recursively, we obtain
\begin{align*}
\PE[ \norm{ \param[t] - \paramlim}^2 ]
\le
(1 - 2 \step\strcvx(1 - \step \explip))^{\nlupdates t} \PE[ \norm{ \paramw - \paramlim}^2 ]
+ \frac{2 \step}{\strcvx (1 - \step \explip)} \Msmoothcstvar^{2} 
\eqsp,
\end{align*}
which is the first part of the result. 
Taking $\paramw \sim \statdist{\step, \nlupdates}$ and using the fact that $\statdist{\step,\nlupdates}$ is the stationary distribution of \FedAvg's global iterates, $\param[t]$ and $\paramw$ are identically distributed, then taking the limit as $t \rightarrow +\infty$ gives the second part of the result.
Finally, using \eqref{eq:proof-app-crude-bound-rec-local-homogeneous} we obtain 
\begin{align*}
\PE[ \norm{ \sdfedavgop[c]{\step,h}(\paramw; \randState[c][1:h]) - \paramlim}^2 ]
\le
\PE[ \norm{ \paramw - \paramlim}^2 ]
+ 2 \step^2 h \Msmoothcstvar^{2}
= O(\step + \step^2 h)
= O(\step)
\eqsp,
\end{align*}
since $\step h = O(1)$, which gives the last part of the result.
\end{proof}

\begin{lemma}
\label{lem:bound-moments-homogeneous}
Assume \Cref{assum:local_functions}, \Cref{assum:smooth-var}, and let \Cref{assum:heterogeneity} holds with $\heterboundgrad=0$.
Let $\step \le 1/(9\explip)$, and $\step \strcvx \nlupdates \le 1$ then there exist a universal constant $\beta > 0$ such that
\begin{align*}
\PE^{1/3} \left[ \norm{ \param[t] - \paramlim}^{6} \right]
\le
(1 - \step\strcvx/3)^{\nlupdates t} \PE^{1/3}[ \norm{ \param[0] - \paramlim}^{6} ]
+ \frac{3 \beta \step}{\strcvx} \Msmoothcstvar^{2}
\eqsp.
\end{align*}
Moreover, for $\paramw \sim \statdist{\step, \nlupdates}$, where $\statdist{\step, \nlupdates}$ is the stationary distribution of \FedAvg with step size $\step$ and $\nlupdates$ local updates, it holds that, for $p \in \{2, 3\}$, and $c \in \iint{1}{\nagent}$,
\begin{align*}
\int \norm{ \param - \paramlim }^{2p} \statdist{\step, \nlupdates}(\rmd \paramw) = O(\step^p)
~,
~~
\text{ and }
\int \norm{ \sdfedavgop[c]{\step,h}(\paramw; \randState[c][1:h]) - \paramlim}^{2p}  \statdist{\step, \nlupdates}(\rmd \paramw)
= O(\step^p)
\eqsp,
\end{align*}
where $\randState[c][1:\nlupdates] = \{ \randState[{c}][\tilde{h}] : \tilde{h} \in \iint{1}{\nlupdates}\}$ is a sequence of independent random variable, with $\randState[{c}][\tilde{h}] \sim \xic[{c}]$.
\end{lemma}
\begin{proof}
We now extend the results of \Cref{lem:crude-bound-second-moment-homogeneous} to higher moments of $\norm{ \param - \paramlim }^2$, with $\param \sim \statdist{\step, \nlupdates}$.
First, we prove a bound on the moment of order $6$. To this end, we start by deriving an upper bound for local updates, decomposing the update between a contraction and an additive term due to stochasticity.
Starting from a point $\paramw \in \rset^d$, we first expand the squared norm, as in the proof of \Cref{lem:crude-bound-second-moment-homogeneous}, as
\begin{align*}
& \norm{ \sdfedavgop[c]{\step,h+1}(\paramw; \randState[c][1:h+1]) - \paramlim}^2
\\
& \quad = 
\norm{ \sdfedavgop[c]{\step,h}(\paramw; \randState[c][1:h]) - \paramlim}^2
- 2 \step \pscal{ \gnfs[c]{\sdfedavgop[c]{\step,h}(\paramw; \randState[c][1:h])}{\randState[c][h+1]} }{ \sdfedavgop[c]{\step,h}(\paramw; \randState[c][1:h]) - \paramlim }
+ \norm{ \gnfs[c]{\sdfedavgop[c]{\step,h}(\paramw; \randState[c][1:h])}{\randState[c][h+1]} }^2
\eqsp.
\end{align*}
To reach the sixth power, we take this equation at the power three. We use the fact that, for $u, v, w \in \rset$, it holds that $(u+v+w)^3 = u^3+3u^2v+3uv^2+v^3+3u^2w+6uvw+3v^2w+3uw^2+3vw^2+w^3$. Thus, for $a, b, c \in \rset$,
\begin{align*}
& (a^2-2 \gamma b+ \gamma^2 c^2)^3
\\
& \quad = a^6 
- 6\gamma a^4 b 
+ 3 \step^2 a^4 c^2 
+ 12 \step^2 a^2 b^2
- 12 \step^3 a^2 b c^2 
+ 3 \step^4 a^2 c^4 
- 8 \step^3 b^3 
+ 12 \step^4 b^2 c^2 
- 6 \step^5 b c^4 
+ \step^6 c^6
\eqsp.
\end{align*}
If $a,b,c$ satisfy $ |b| \leq a c$, we have
\begin{align}
& (a^2-2 \gamma b+ \gamma^2 c^2)^3 \nonumber
\\
& \le a^6 
- 6\gamma a^4 b 
+ 3 \step^2 a^4 c^2 
+ 12 \step^2 a^4 c^2
+ 12 \step^3 a^3 c^3 
+ 3 \step^4 a^2 c^4 
+ 8 \step^3 a^3 c^3 
+ 12 \step^4 a^2 c^4 
+ 6 \step^5 a c^5 
+ \step^6 c^6 \nonumber
\\
& = a^6 
- 6\gamma a^4 b 
+ 15 \step^2 a^4 c^2
+ 20 \step^3 a^3 c^3 
+ 15 \step^4 a^2 c^4 
+ 6 \step^5 a c^5 
+ \step^6 c^6
\label{eq:proof-moment-homogeneous-expand-abc}
\eqsp.
\end{align}
Now, we take $a = \norm{ \sdfedavgop[c]{\step,h}(\paramw; \randState[c][1:h]) - \paramlim }$, $b = \pscal{ \gnfs[c]{\sdfedavgop[c]{\step,h}(\paramw; \randState[c][1:h])}{\randState[c][h+1]} }{ \sdfedavgop[c]{\step,h}(\paramw; \randState[c][1:h]) - \paramlim }$, and $c = \norm{ \gnfs[c]{\sdfedavgop[c]{\step,h}(\paramw; \randState[c][1:h])}{\randState[c][h+1]} } $.
Note that we indeed have $b \le a c$ using the Cauchy-Schwarz inequality.

At this point, we have the following bound, for $2 \le k \le 6$,
\begin{align*}
& \CPE{ c^k }{\mcF_c^h} 
=
\CPE{ \norm{ \gnfs[c]{\sdfedavgop[c]{\step,h}(\paramw; \randState[c][1:h])}{\randState[c][h+1]} }^k }{\mcF_c^h} 
\\
& \qquad \le
2^{k-1} \left\{ \CPE{ \norm{ \gnfs[c]{\sdfedavgop[c]{\step,h}(\paramw; \randState[c][1:h])}{\randState[c][h+1]} - \gnfs[c]{\paramlim}{\randState[c][h+1]} }^k}{\mcF_c^h} 
+ \CPE{ \norm{ \gnfs[c]{\sdfedavgop[c]{\step,h}(\paramw; \randState[c][1:h])}{\randState[c][h+1]} }^k}{\mcF_c^h}  \right\}
\\
& \qquad \le
2^{k-1} \left\{ \CPE{ \norm{ \gnfs[c]{\sdfedavgop[c]{\step,h}(\paramw; \randState[c][1:h])}{\randState[c][h+1]} - \gnfs[c]{\paramlim}{\randState[c][h+1]} }^k}{\mcF_c^h} 
+ \Msmoothcstvar^k  \right\}
\eqsp.
\end{align*}
Then, by \Cref{assum:local_functions}, and since $\gnf[c]{\paramlim} = 0$, we have
\begin{align}
& \CPE{ \norm{ \gnfs[c]{\sdfedavgop[c]{\step,h}(\paramw; \randState[c][1:h])}{\randState[c][h+1]} - \gnfs[c]{\paramlim}{\randState[c][h+1]} }^k}{\mcF_c^h} 
\nonumber
\\
\label{eq:proof-crude-high-homogeneous-all-in-condexp}
& \quad \le
\explip^{k-2} \norm{ \sdfedavgop[c]{\step,h}(\paramw; \randState[c][1:h]) - \paramlim }^{k-2}
\CPE{ \norm{ \gnfs[c]{\sdfedavgop[c]{\step,h}(\paramw; \randState[c][1:h])}{\randState[c][h+1]} - \gnfs[c]{\paramlim}{\randState[c][h+1]} }^k}{\mcF_c^h} 
\\
& \quad \le
\explip^{k-1} \norm{ \sdfedavgop[c]{\step,h}(\paramw; \randState[c][1:h]) - \paramlim }^{k-2}
\pscal{ \gnf[c]{\sdfedavgop[c]{\step,h}(\paramw; \randState[c][1:h])} - \gnf[c]{\paramlim} }{ \sdfedavgop[c]{\step,h}(\paramw; \randState[c][1:h]) - \paramlim }
\eqsp. \nonumber
\end{align}
This guarantees that
\begin{align*}
& \CPE{ c^k }{\mcF_c^h} 
\\
& \quad \le
2^{k-1} \explip^{k-1} \norm{ \sdfedavgop[c]{\step,h}(\paramw; \randState[c][1:h]) - \paramlim }^{k-2}
\pscal{ \gnf[c]{\sdfedavgop[c]{\step,h}(\paramw; \randState[c][1:h])} - \gnf[c]{\paramlim} }{ \sdfedavgop[c]{\step,h}(\paramw; \randState[c][1:h]) - \paramlim }
+ 2^{k-1} \Msmoothcstvar^k  
\eqsp.
\end{align*}
Which in turn proves that
\begin{align*}
& \CPE{ \step^{k} a^{6-k} c^k }{\mcF_c^h} 
\\
& \le
2^{k-1} \step^{k} \explip^{k-1} \norm{ \sdfedavgop[c]{\step,h}(\paramw; \randState[c][1:h]) - \paramlim }^{6 - k + k-2}
\pscal{ \gnf[c]{\sdfedavgop[c]{\step,h}(\paramw; \randState[c][1:h])} - \gnf[c]{\paramlim} }{ \sdfedavgop[c]{\step,h}(\paramw; \randState[c][1:h]) - \paramlim }
\\
& \quad 
+ 2^{k-1} \step^{k} \norm{ \sdfedavgop[c]{\step,h}(\paramw; \randState[c][1:h]) - \paramlim }^{6 - k} \Msmoothcstvar^k 
\\
& =
2^{k-1} \step^{k} \explip^{k-1} \norm{ \sdfedavgop[c]{\step,h}(\paramw; \randState[c][1:h]) - \paramlim }^{4}
\pscal{ \gnf[c]{\sdfedavgop[c]{\step,h}(\paramw; \randState[c][1:h])} - \gnf[c]{\paramlim} }{ \sdfedavgop[c]{\step,h}(\paramw; \randState[c][1:h]) - \paramlim }
\\
& \quad 
+ 2^{k-1} \step^{k} \norm{ \sdfedavgop[c]{\step,h}(\paramw; \randState[c][1:h]) - \paramlim }^{6 - k} \Msmoothcstvar^k  
\eqsp.
\end{align*}
Then, we remark that 
\begin{align}
\label{eq:bound-homogeneous-cpe-6}
\CPE{ -6 \step a^4 b }{\mcF_c^h}
& \le - 6 \step  \norm{ \sdfedavgop[c]{\step,h}(\paramw; \randState[c][1:h]) - \paramlim }^{4} \pscal{ \gnf[c]{\sdfedavgop[c]{\step,h}(\paramw; \randState[c][1:h])} - \gnf[c]{\paramlim} }{ \sdfedavgop[c]{\step,h}(\paramw; \randState[c][1:h]) - \paramlim }
\eqsp.
\end{align}
Plugging \eqref{eq:bound-homogeneous-cpe-6} in the conditional expectation of  \eqref{eq:proof-moment-homogeneous-expand-abc}, we obtain
\begin{align*}
\nonumber
(a^2-2 \gamma b+ \gamma^2 c^2)^3
& \le 
a^6 
+ 
\Big(
-6 \step + 2 \cdot 15 \step^2 \explip + 4 \cdot 20 \step^3 \explip^2 + 8 \cdot 15 \step^4 \explip^3 + 16 \cdot 6 \step^5 \explip^4 + 32 \step^6 \explip^5   
\Big)
\\
& \quad \qquad \qquad \qquad \times
\norm{ \sdfedavgop[c]{\step,h}(\paramw; \randState[c][1:h]) - \paramlim }^{4} \pscal{ \gnf[c]{\sdfedavgop[c]{\step,h}(\paramw; \randState[c][1:h])} }{ \sdfedavgop[c]{\step,h}(\paramw; \randState[c][1:h]) - \paramlim }
\\
& \quad
+ 20 \sum_{k=2}^6 2^{k-1}  \step^{k} \norm{ \sdfedavgop[c]{\step,h}(\paramw; \randState[c][1:h]) - \paramlim }^{6 - k} \Msmoothcstvar^k  
\eqsp. \nonumber
\end{align*}
Taking $\step \explip \le 1/9$, we have $2 \cdot 15 \step^2 \explip + 4 \cdot 20 \step^3 \explip^2 + 8 \cdot 15 \step^4 \explip^3 + 16 \cdot 6 \step^5 \explip^4 + 32 \step^6 \explip^5 \le 5 \step$.
Since, by \Cref{assum:local_functions}, we have
\begin{align*}
- \step \pscal{ \gnf[c]{\sdfedavgop[c]{\step,h}(\paramw; \randState[c][1:h])} }{ \sdfedavgop[c]{\step,h}(\paramw; \randState[c][1:h]) - \paramlim }
\le - \step \strcvx \norm{ \sdfedavgop[c]{\step,h}(\paramw; \randState[c][1:h]) - \paramlim }^2
\eqsp,
\end{align*}
we obtain the following bound
\begin{align}
\nonumber
& \CPE{ \norm{ \sdfedavgop[c]{\step,h+1}(\paramw; \randState[c][1:h+1]) - \paramlim }^{6} }{\mcF_c^h}
\\
\label{eq:bound-homogeneous-with-all-terms-sum}
& \quad \le 
(1 - \step \strcvx)
\norm{ \sdfedavgop[c]{\step,h}(\paramw; \randState[c][1:h]) - \paramlim }^{6} 
+ 20 \sum_{k=2}^6 2^{k-1} \norm{ \sdfedavgop[c]{\step,h}(\paramw; \randState[c][1:h]) - \paramlim }^{6 - k} (\step \Msmoothcstvar)^k
\eqsp.
\end{align}
We now express this sum as a third-power of a sum of two terms: one contraction, and one additive term due to stochasticity.
Let $k = 2\ell + 1 \in \iint{2}{6}$ be an odd number, which implies $\ell = 1$ or $\ell=2$.
Since $k \ge 2$, then $\ell \ge 1$, and $k \ge 3$.
Using the fact that for odd values of $k=2\ell+1$, then $k-1=2\ell \ge 2$ is even, we have
\begin{align}
\nonumber
\norm{ \sdfedavgop[c]{\step,h}(\paramw; \randState[c][1:h]) - \paramlim }^{6 - k} (\step \Msmoothcstvar)^k
& =
\norm{ \sdfedavgop[c]{\step,h}(\paramw; \randState[c][1:h]) - \paramlim }^{5 - 2\ell} (\step \Msmoothcstvar)^{2\ell + 1}
\\
\nonumber
& =
\norm{ \sdfedavgop[c]{\step,h}(\paramw; \randState[c][1:h]) - \paramlim }^{4 - 2\ell} (\step \Msmoothcstvar)^{2\ell} \left( \norm{ \sdfedavgop[c]{\step,h}(\paramw; \randState[c][1:h]) - \paramlim } \step \Msmoothcstvar \right)
\\
\label{eq:bound-homogeneous-split-odd-terms}
& \le
\norm{ \sdfedavgop[c]{\step,h}(\paramw; \randState[c][1:h]) - \paramlim }^{4 - 2\ell} (\step \Msmoothcstvar)^{2\ell} \left( 2 \norm{ \sdfedavgop[c]{\step,h}(\paramw; \randState[c][1:h]) - \paramlim }^2 + 2 \step^2 \Msmoothcstvar^2 \right)
\eqsp.
\end{align}
Using \eqref{eq:bound-homogeneous-split-odd-terms} to remove the odd terms from the sum in \eqref{eq:bound-homogeneous-with-all-terms-sum}, as well as Hölder's inequality, and following the lines of proof of \citet{dieuleveut2020bridging}'s Lemma 13, there exists a constant $\beta > 0$ such that
\begin{align}
\PE \left[ \norm{ \sdfedavgop[c]{\step,h+1}(\paramw; \randState[c][1:h+1]) - \paramlim }^{6} \right]
& \le
\left( 
(1 - \step \strcvx/3) \PE\left[\norm{ \sdfedavgop[c]{\step,h}(\paramw; \randState[c][1:h]) - \paramlim }^{6} \right]^{1/3}
+ 
\beta \step^2 \Msmoothcstvar^2 
\right)^3
\eqsp.
\label{eq:proof-crude-high-homogeneous-factorized-bound}
\end{align}
Consequently, we have
\begin{align*}
\PE \left[ \norm{ \sdfedavgop[c]{\step,h+1}(\paramw; \randState[c][1:h+1]) - \paramlim }^{6} \right]^{1/3}
& \le
(1 - \step \strcvx/3) \PE\left[\norm{ \sdfedavgop[c]{\step,h}(\paramw; \randState[c][1:h]) - \paramlim }^{6} \right]^{1/3}
+ 
\beta \step^2 \Msmoothcstvar^2 
\eqsp.
\end{align*}
Iterating this for $\nlupdates$ iterations, we obtain that
\begin{align}
\label{eq:proof-bound-homogeneous-bound-loc-H-updates}
\PE \left[ \norm{ \sdfedavgop[c]{\step,\nlupdates}(\paramw; \randState[c][1:\nlupdates]) - \paramlim }^{6} \right]^{1/3}
& \le
(1 - \step \strcvx/3)^\nlupdates \PE\left[\norm{ \paramw - \paramlim }^{6} \right]^{1/3}
+ 
\beta \nlupdates \step^2 \Msmoothcstvar^2 
\eqsp.
\end{align}
Using Jensen's inequality and \eqref{eq:proof-bound-homogeneous-bound-loc-H-updates}, we obtain, for any $\paramw \in \rset^d$,
\begin{align*}
\PE \left[ \norm{ \sdfedavgop{\step,\nlupdates}(\paramw; \randState[1:\nagent,t][1:\nlupdates]) - \paramlim }^{6} \right]^{1/3}
& \le
\frac{1}{\nagent}
\sum_{c=1}^\nagent
\PE \left[ \norm{ \sdfedavgop[c]{\step,\nlupdates}(\paramw; \randState[1:\nagent,t][1:\nlupdates]) - \paramlim }^{6} \right]^{1/3}
\\
& \le
(1 - \step \strcvx/3)^\nlupdates \PE\left[\norm{ \paramw- \paramlim }^{6} \right]^{1/3}
+ 
\beta \nlupdates \step^2 \Msmoothcstvar^2 
\eqsp,
\end{align*}
and the first part of the result follows from iterating this inequality $T$ times, starting from $\param[T]$.

The second part of the result for $p = 3$ directly follows from the previous inequality. 
To obtain the result for $p = 2$, we use Hölder inequality and remark that 
\begin{align*}
\int \norm{ \param - \paramlim }^{4} \statdist{\step, \nlupdates}(\rmd \paramw) 
& \le
\left( \int \norm{ \param - \paramlim }^6  \statdist{\step, \nlupdates}(\rmd \paramw)  \right)^{2/3}
= O\left( \step^2 \right)
\eqsp,
\end{align*}
where the last equality comes from the first part of this Lemma.
\end{proof}

\subsubsection{Heterogeneous Functions}
\begin{lemma}
\label{lem:crude-bound-second-moment}
Assume \Cref{assum:local_functions}, \Cref{assum:heterogeneity}, \Cref{assum:smooth-var}, let $\step \le 1/(2\explip)$, and $\step \strcvx \nlupdates \le 1$. Then we have
\begin{align*}
& \PE\left[ \norm{ \param[t] - \paramlim }^2 \right]
\le
\left(1 - \frac{\step \strcvx}{2} \right)^{\nlupdates t} \norm{ \param[0] - \paramlim }^2
+ \frac{\nlupdates (\nlupdates-1)}{\strcvx} \left( 4 \step^3 \explip^2 + \frac{2 \step^2 \explip^2}{\strcvx} \right) \heterboundgrad^2
+ \frac{8 \step}{\strcvx} \Msmoothcstvar^{2}
\eqsp.
\end{align*}
This implies that, for $\paramw \sim \statdist{\step, \nlupdates}$, where $\statdist{\step, \nlupdates}$ is the stationary distribution of \FedAvg with step size $\step$ and $\nlupdates$ local updates, it holds that
\begin{align*}
\int \norm{ \param - \paramlim }^{2} \statdist{\step, \nlupdates}(\rmd \paramw) = O(\step + \step^2 \nlupdates^2)
~,
~~
\text{ and }
\int \norm{ \sdfedavgop[c]{\step,h}(\paramw; \randState[c][1:h]) - \paramlim}^{2}  \statdist{\step, \nlupdates}(\rmd \paramw)
= O(\step + \step^2 \nlupdates^2)
\eqsp,
\end{align*}
where $\randState[c][1:\nlupdates] = \{ \randState[{c}][\tilde{h}] : \tilde{h} \in \iint{1}{\nlupdates}\}$ is a sequence of independent random variable, with $\randState[{c}][\tilde{h}] \sim \xic[{c}]$.
\end{lemma}
\begin{proof}
We start from $\param[t+1] =
\param[t] - \step \pseudograds[\step,\nlupdates]{\paramw}{\randState[1:\nagent][1:\nlupdates]}$, with $\pseudograds[\step,\nlupdates]{\paramw}{\randState[1:\nagent][1:\nlupdates]}$ as defined in \Cref{sec:stochastic-fedavg}, and use $\frac{1}{\nagent} \sum_{c=1}^\nagent \gnf[c]{\paramlim} = 0$, to obtain
\begin{align*}
\param[t+1]
& =
\param[t]
- \frac{\step}{\nagent}
\sum_{c=1}^\nagent 
\sum_{h=0}^{\nlupdates-1} 
\left\{ 
\gnfs[c]{\sdfedavgop[c]{\step,h}(\paramw;\randState[c][1:h])}{\randState[c][h+1]}
- \gnf[c]{\paramlim}
\right\}
\eqsp.
\end{align*}
Using Jensen's inequality, we have
\begin{align}
\label{eq:proof-crude-conv-heter-pytha}
\norm{ \param[t+1] - \paramlim }^2
& \le
\frac{1}{\nagent} 
\sum_{c=1}^\nagent 
\bnorm{
\param[t]
- \step
\sum_{h=0}^{\nlupdates-1} 
\left\{
\gnfs[c]{\sdfedavgop[c]{\step,h}(\paramw;\randState[c][1:h])}{\randState[c][h+1]}
- \gnf[c]{\paramlim}
\right\}
}^2
\eqsp.
\end{align}
To derive an upper bound on this value, we study the following sequence of iterates, that correspond to the local parameters with recentered gradients, defined for $h \in \iint{0}{\nlupdates-1}$,
\begin{align}
\label{eq:def-operator-V-local-heterogeneous}
\hsdfedavgop[c]{\step,h}(\paramw; \randState[c][1:h])
& \eqdef
\sdfedavgop[c]{\step,h}(\paramw; \randState[c][1:h]) - \step h \gnf[c]{\paramlim}
\eqsp,
\end{align}
which allows to rewrite \eqref{eq:proof-crude-conv-heter-pytha} as
\begin{align}
\label{eq:proof-crude-conv-heter-pytha-rewritten}
  \norm{ \param[t+1] - \paramlim }^2 \le \frac{1}{\nagent} \sum_{c=1}^\nagent \norm{ \hsdfedavgop[c]{\step,\nlupdates}(\paramw; \randState[c][1:\nlupdates]) - \paramlim }^2
  \eqsp.
\end{align}
Next, we bound each term of this sum independently.
We do so by induction, setting $h \in \iint{0}{\nlupdates-1}$, we may expand 
\begin{align*}
\norm{ \hsdfedavgop[c]{\step,h+1}(\paramw; \randState[c][1:h+1]) - \paramlim }^2
& =
\norm{ 
\hsdfedavgop[c]{\step,h}(\paramw; \randState[c][1:h]) - \paramlim
- \step (\gnfs[c]{\sdfedavgop[c]{\step,h}(\paramw; \randState[c][1:h])}{\randState[c][h+1]} - \gnf[c]{\paramlim})
}^2
\\
& =
\norm{ \hsdfedavgop[c]{\step,h}(\paramw; \randState[c][1:h]) - \paramlim }^2
+
\step^2 \norm{ \gnfs[c]{\sdfedavgop[c]{\step,h}(\paramw; \randState[c][1:h])}{\randState[c][h+1]} - \gnf[c]{\paramlim} }^2
\\
& \qquad  \qquad 
- 2 \step 
\pscal{ 
\hsdfedavgop[c]{\step,h}(\paramw; \randState[c][1:h]) - \paramlim }{ \gnfs[c]{\sdfedavgop[c]{\step,h}(\paramw; \randState[c][1:h])}{\randState[c][h+1]} - \gnf[c]{\paramlim}
}
\eqsp.
\end{align*}
We now take the expectation using the filtration $\mcF_{c}^h = \sigma( \randState[c][\ell] : \ell \le h )$, for $h \in \iint{0}{\nlupdates-1}$,
\begin{align}
\nonumber
\CPE{
\norm{ \hsdfedavgop[c]{\step,h+1}(\paramw; \randState[c][1:h+1]) - \paramlim }^2
}{\mcF_{c}^{h}}
& =
\norm{ \hsdfedavgop[c]{\step,h}(\paramw; \randState[c][1:h]) - \paramlim }^2
\\
\nonumber
& \quad
+
\step^2 \CPE{ \norm{ \gnfs[c]{\sdfedavgop[c]{\step,h}(\paramw; \randState[c][1:h])}{\randState[c][h+1]} - \gnf[c]{\paramlim} }^2 
}{\mcF_{c}^{h}}
\\
\label{eq:proof-crude-bound-heter-expansion-first}
& \quad 
- 2 \step 
\pscal{ 
\hsdfedavgop[c]{\step,h}(\paramw; \randState[c][1:h]) - \paramlim }{ \gnf[c]{\sdfedavgop[c]{\step,h}(\paramw; \randState[c][1:h])} - \gnf[c]{\paramlim}
}
\eqsp.
\end{align}
Now, we remark that
\begin{align*}
\gnfs[c]{\sdfedavgop[c]{\step,h}(\paramw; \randState[c][1:h])}{\randState[c][h+1]} - \gnf[c]{\paramlim}
& =
\gnfs[c]{\hsdfedavgop[c]{\step,h}(\paramw; \randState[c][1:h])}{\randState[c][h+1]} 
- \gnfs[c]{\paramlim}{\randState[c][h+1]} 
\\
& \quad 
+ \gnfs[c]{\sdfedavgop[c]{\step,h}(\paramw; \randState[c][1:h])}{\randState[c][h+1]}
- \gnfs[c]{\hsdfedavgop[c]{\step,h}(\paramw; \randState[c][1:h])}{\randState[c][h+1]}
\\
& \quad 
+ \gnfs[c]{\paramlim}{\randState[c][h+1]} 
- \gnf[c]{\paramlim}
\eqsp,
\end{align*}
which allows to decompose the term $\CPE{ \norm{ \gnfs[c]{\sdfedavgop[c]{\step,h}(\paramw; \randState[c][1:h])}{\randState[c][h+1]} - \gnf[c]{\paramlim} }^2 
}{\mcF_{c}^{h}}$ using Young's inequality twice,
followed by \Cref{assum:local_functions} and $\Cref{assum:smooth-var}$,
\begin{align}
\nonumber
& \CPE{ \norm{ \gnfs[c]{\sdfedavgop[c]{\step,h}(\paramw; \randState[c][1:h])}{\randState[c][h+1]} - \gnf[c]{\paramlim} }^2 
}{\mcF_{c}^{h}}
\\
\nonumber
& \quad \le
2 \CPE{ \norm{ \gnfs[c]{\hsdfedavgop[c]{\step,h}(\paramw; \randState[c][1:h])}{\randState[c][h+1]} 
- \gnfs[c]{\paramlim}{\randState[c][h+1]}  }^2 }{\mcF_{c}^{h}}
\\
\nonumber
& \qquad 
+ 4 \CPE{ \norm{ \gnfs[c]{\sdfedavgop[c]{\step,h}(\paramw; \randState[c][1:h])}{\randState[c][h+1]} 
- \gnfs[c]{\hsdfedavgop[c]{\step,h}(\paramw; \randState[c][1:h])}{\randState[c][h+1]} }^2 }{\mcF_{c}^{h}}
+ 4 \CPE{ \norm{ \gnfs[c]{\paramlim}{\randState[c][h+1]} 
- \gnf[c]{\paramlim} }^2 }{\mcF_{c}^{h}}
\\
\nonumber
& \quad \le
2 \CPE{ \norm{ \gnfs[c]{\hsdfedavgop[c]{\step,h}(\paramw; \randState[c][1:h])}{\randState[c][h+1]} 
- \gnfs[c]{\paramlim}{\randState[c][h+1]}  }^2 }{\mcF_{c}^{h}}
+ 4 \explip^2 \norm{ \sdfedavgop[c]{\step,h}(\paramw; \randState[c][1:h])
- \hsdfedavgop[c]{\step,h}(\paramw; \randState[c][1:h]) }^2
+ 4 \Msmoothcstvar^{2}
\\
\label{eq:proof-crude-bound-heter-expansion-bound-t1}
& \quad =
2 \CPE{ \norm{ \gnfs[c]{\hsdfedavgop[c]{\step,h}(\paramw; \randState[c][1:h])}{\randState[c][h+1]} 
- \gnfs[c]{\paramlim}{\randState[c][h+1]}  }^2 }{\mcF_{c}^{h}}
+ 4 \explip^2 \step^2 h^2 \norm{ \gnf[c]{\paramlim} }^2
+ 4 \Msmoothcstvar^{2}
\eqsp,
\end{align}
where the last equality comes from the definition of $\hsdfedavgop[c]{\step,h}(\paramw; \randState[c][1:h])$.
Furthermore, we have
\begin{align*}
& - 2 \step 
\pscal{ 
\hsdfedavgop[c]{\step,h}(\paramw; \randState[c][1:h]) - \paramlim }{ \gnf[c]{\sdfedavgop[c]{\step,h}(\paramw; \randState[c][1:h])} - \gnf[c]{\paramlim}
}
\\ & \quad
=
- 2 \step 
\pscal{ 
\hsdfedavgop[c]{\step,h}(\paramw; \randState[c][1:h]) - \paramlim }{ \gnf[c]{\hsdfedavgop[c]{\step,h}(\paramw; \randState[c][1:h])} - \gnf[c]{\paramlim}
}
\\
&  \qquad
- 2 \step 
\pscal{ 
\hsdfedavgop[c]{\step,h}(\paramw; \randState[c][1:h]) - \paramlim }{ \gnf[c]{\sdfedavgop[c]{\step,h}(\paramw; \randState[c][1:h])} - \gnf[c]{\hsdfedavgop[c]{\step,h}(\paramw; \randState[c][1:h])} }
\end{align*}
We may bound the second term of this identity using Young's inequality, \Cref{assum:local_functions}, and the definition of ${\hsdfedavgop[c]{\step,h}(\paramw; \randState[c][1:h])}$,
\begin{align}
\nonumber
& - 2 \step 
\pscal{ 
\hsdfedavgop[c]{\step,h}(\paramw; \randState[c][1:h]) - \paramlim }{ \gnf[c]{\sdfedavgop[c]{\step,h}(\paramw; \randState[c][1:h])} - \gnf[c]{\hsdfedavgop[c]{\step,h}(\paramw; \randState[c][1:h])} }
\\
\nonumber
& \quad \le
\frac{\step \strcvx}{2} 
\norm{ 
\hsdfedavgop[c]{\step,h}(\paramw; \randState[c][1:h]) - \paramlim }^2
+ 
\frac{2 \step}{\strcvx} \norm{ \gnf[c]{\sdfedavgop[c]{\step,h}(\paramw; \randState[c][1:h])} - \gnf[c]{\hsdfedavgop[c]{\step,h}(\paramw; \randState[c][1:h])} }^2
\\
\label{eq:proof-crude-bound-heter-expansion-bound-t2}
& \quad \le
\frac{\step \strcvx}{2} 
\norm{ 
\hsdfedavgop[c]{\step,h}(\paramw; \randState[c][1:h]) - \paramlim }^2
+ 
\frac{2 \step^3 h^2 \explip^2}{\strcvx} \norm{ \gnf[c]{\paramlim} }^2
\eqsp.
\end{align}
Finally, notice that whenever $\step \le 1/(2\explip)$, \Cref{assum:local_functions} implies that
\begin{align}
\nonumber
& \norm{ \hsdfedavgop[c]{\step,h}(\paramw; \randState[c][1:h]) - \paramlim }^2
+
2 \step^2 \CPE{ \norm{ \gnfs[c]{\hsdfedavgop[c]{\step,h}(\paramw; \randState[c][1:h])}{\randState[c][h+1]} 
- \gnfs[c]{\paramlim}{\randState[c][h+1]}  }^2 }{\mcF_{c}^{h}}
\\
\nonumber
& \qquad 
- 2 \step 
\pscal{ 
\hsdfedavgop[c]{\step,h}(\paramw; \randState[c][1:h]) - \paramlim }{ \gnf[c]{\hsdfedavgop[c]{\step,h}(\paramw; \randState[c][1:h])} - \gnf[c]{\paramlim}
}
\\
\label{eq:proof-crude-bound-heter-expansion-bound-t3}
& \qquad
\le 
(1 - \step \strcvx) \norm{ \hsdfedavgop[c]{\step,h}(\paramw; \randState[c][1:h]) - \paramlim }^2
\eqsp.
\end{align}
Plugging \eqref{eq:proof-crude-bound-heter-expansion-bound-t1}, \eqref{eq:proof-crude-bound-heter-expansion-bound-t2} and \eqref{eq:proof-crude-bound-heter-expansion-bound-t3} in \eqref{eq:proof-crude-bound-heter-expansion-first}, we obtain
\begin{align}
\nonumber
& \CPE{
\norm{ \hsdfedavgop[c]{\step,h+1}(\paramw; \randState[c][1:h+1]) - \paramlim }^2
}{\mcF_{c}^{h}}
\\
\label{eq:proof-crude-bound-heter-recursion-one-step}
& \quad \le
\left(1 - \frac{\step \strcvx}{2} \right) \norm{ \hsdfedavgop[c]{\step,h}(\paramw; \randState[c][1:h]) - \paramlim }^2
+ \left( 4 \step^4 h^2 \explip^2 + \frac{2 \step^3 h^2 \explip^2}{\strcvx} \right) \norm{ \gnf[c]{\paramlim }}^2
+ 4 \step^2 \Msmoothcstvar^{2}
\eqsp.
\end{align}
Taking the expectation and unrolling the inequality, we obtain
\begin{align*}
& \PE\left[
\norm{ \hsdfedavgop[c]{\step,\nlupdates}(\paramw; \randState[c][1:h+1]) - \paramlim }^2
\right]
\\
& \quad \le
\left(1 - \frac{\step \strcvx}{2} \right)^\nlupdates \norm{ \paramw - \paramlim }^2
+ \frac{\nlupdates^2 (\nlupdates-1)}{2} \left( 4 \step^4 \explip^2 + \frac{2 \step^3 \explip^2}{\strcvx} \right) \norm{ \gnf[c]{\paramlim }}^2
+ 4 \step^2 \nlupdates \Msmoothcstvar^{2}
\eqsp.
\end{align*}
Using this inequality to bound each term of \eqref{eq:proof-crude-conv-heter-pytha-rewritten}, we obtain the following inequality, that links two consecutive global parameters of \FedAvg,
\begin{align*}
& \PE\left[
\norm{ \sdfedavgop{\step,h}(\paramw; \randState[1:\nagent][1:\nlupdates]) - \paramlim }^2
\right]
\le
\left(1 - \frac{\step \strcvx}{2} \right)^\nlupdates \norm{ \paramw - \paramlim }^2
+ \frac{\nlupdates^2 (\nlupdates-1)}{2} \left( 4 \step^4 \explip^2 + \frac{2 \step^3 \explip^2}{\strcvx} \right) \heterboundgrad^2
+ 4 \step^2 \nlupdates \Msmoothcstvar^{2}
\eqsp.
\end{align*}
Unrolling this inequality starting from a point $\param[0] \in \rset^d$, we obtain
\begin{align*}
& \PE\left[ \norm{ \param[t] - \paramlim }^2 \right]
\le
\left(1 - \frac{\step \strcvx}{2} \right)^{\nlupdates t} \norm{ \param[0] - \paramlim }^2
+ \frac{\nlupdates (\nlupdates-1)}{\strcvx} \left( 4 \step^3 \explip^2 + \frac{2 \step^2 \explip^2}{\strcvx} \right) \heterboundgrad^2
+ \frac{8 \step}{\strcvx} \Msmoothcstvar^{2}
\eqsp,
\end{align*}
which gives the first part of the Lemma. The second part follows the same lines as the second part of \Cref{lem:crude-bound-second-moment-homogeneous}.
\end{proof}
\begin{lemma}
\label{lem:bound-moments-heterogeneous}
Assume \Cref{assum:local_functions}, \Cref{assum:heterogeneity} and \Cref{assum:smooth-var}.
Let $\step \le 1/(45\explip)$, and $\step \strcvx \nlupdates \le 1$ then there exist a universal constant $\beta > 0$ such that
\begin{align*}
\PE^{1/3} \left[ \norm{ \param[t] - \paramlim}^{6} \right]
\le
(1 - \step \strcvx/18)^\nlupdates \PE\left[\norm{ \param[0] - \paramlim }^{6} \right]^{1/3}
+ 6 \beta \frac{\step^{2} (\nlupdates-1)\nlupdates \heterboundgrad}{\strcvx^2} 
+ \frac{ 12 \beta \step }{ \strcvx } \Msmoothcstvar^2
\eqsp.
\end{align*}
This implies that, for $\paramw \sim \statdist{\step, \nlupdates}$, where $\statdist{\step, \nlupdates}$ is the stationary distribution of \FedAvg with step size $\step$ and $\nlupdates$ local updates, it holds that, for $p \in \{2, 3\}$, and $c \in \iint{1}{\nagent}$
\begin{align*}
\int \norm{ \param - \paramlim }^{2p} \statdist{\step, \nlupdates}(\rmd \paramw) 
= O\left( \step^p + \step^{2p} \nlupdates^{2p} \right)
~,
~~
\text{ and }
\int \norm{ \sdfedavgop[c]{\step,h}(\paramw; \randState[c][1:h]) - \paramlim}^{2p}  \statdist{\step, \nlupdates}(\rmd \paramw)
= O\left( \step^p + \step^{2p} \nlupdates^{2p} \right)
\eqsp,
\end{align*}
where $\randState[c][1:\nlupdates] = \{ \randState[{c}][\tilde{h}] : \tilde{h} \in \iint{1}{\nlupdates}\}$ is a sequence of independent random variable, with $\randState[{c}][\tilde{h}] \sim \xic[{c}]$.
\end{lemma}
\begin{proof}
The proof follows the same lines as the proof of \Cref{lem:bound-moments-homogeneous}, with an additional heterogeneity term that is $O(\step^2 \nlupdates^2)$ that plays a role similar to the one of $\Msmoothcstvar$.
We start with the expansion of the local updates, recentered by $\step h \gnf[c]{\paramlim}$, as defined in \eqref{eq:def-operator-V-local-heterogeneous}, in the proof of \Cref{lem:crude-bound-second-moment},
\begin{align}
\nonumber
& \norm{ \hsdfedavgop[c]{\step,h+1}(\paramw; \randState[c][1:h+1]) - \paramlim }^2
\\
\nonumber
& \quad =
\norm{ \hsdfedavgop[c]{\step,h}(\paramw; \randState[c][1:h]) - \paramlim }^2
+
\step^2 \norm{ \gnfs[c]{\sdfedavgop[c]{\step,h}(\paramw; \randState[c][1:h])}{\randState[c][h+1]} - \gnf[c]{\paramlim} }^2
\\
\nonumber
& \qquad 
- 2 \step 
\pscal{ 
\hsdfedavgop[c]{\step,h}(\paramw; \randState[c][1:h]) - \paramlim }{ \gnfs[c]{\sdfedavgop[c]{\step,h}(\paramw; \randState[c][1:h])}{\randState[c][h+1]} - \gnf[c]{\paramlim}
}
\\
\nonumber
& \quad =
\norm{ \hsdfedavgop[c]{\step,h}(\paramw; \randState[c][1:h]) - \paramlim }^2
+
\step^2 \norm{ \gnfs[c]{\sdfedavgop[c]{\step,h}(\paramw; \randState[c][1:h])}{\randState[c][h+1]} - \gnf[c]{\paramlim} }^2
\\
\nonumber
& \qquad  
- 2 \step 
\pscal{ 
\hsdfedavgop[c]{\step,h}(\paramw; \randState[c][1:h]) - \paramlim }{ \gnfs[c]{\hsdfedavgop[c]{\step,h}(\paramw; \randState[c][1:h])}{\randState[c][h+1]} - \gnf[c]{\paramlim}
}
\\
\label{eq:proof-bound-heter-expand-first-interm}
& \qquad 
- 2 \step 
\pscal{ 
\hsdfedavgop[c]{\step,h}(\paramw; \randState[c][1:h]) - \paramlim }{ \gnfs[c]{\sdfedavgop[c]{\step,h}(\paramw; \randState[c][1:h])}{\randState[c][h+1]} 
- \gnfs[c]{\hsdfedavgop[c]{\step,h}(\paramw; \randState[c][1:h])}{\randState[c][h+1]} 
}
\eqsp.
\end{align}
We first bound the following squared norm using Young's inequality,
\begin{align}
\nonumber
& \norm{ \gnfs[c]{\sdfedavgop[c]{\step,h}(\paramw; \randState[c][1:h])}{\randState[c][h+1]} - \gnf[c]{\paramlim} }^2
\\
\nonumber
& \quad \le
2 \norm{ \gnfs[c]{\hsdfedavgop[c]{\step,h}(\paramw; \randState[c][1:h])}{\randState[c][h+1]} 
- \gnfs[c]{\paramlim}{\randState[c][h+1]}  }^2
\\
\label{eq:proof-bound-heter-expand-norm-sq-interm}
& \qquad 
+ 4 \norm{ \gnfs[c]{\sdfedavgop[c]{\step,h}(\paramw; \randState[c][1:h])}{\randState[c][h+1]} 
- \gnfs[c]{\hsdfedavgop[c]{\step,h}(\paramw; \randState[c][1:h])}{\randState[c][h+1]} }^2
+ 4 \norm{ \gnfs[c]{\paramlim}{\randState[c][h+1]} 
- \gnf[c]{\paramlim} }^2
\eqsp.
\end{align}
Then, we bound the last term from \eqref{eq:proof-bound-heter-expand-first-interm} using Young's inequality, %
\begin{align}
\nonumber
& - 2 \step 
\pscal{ 
\hsdfedavgop[c]{\step,h}(\paramw; \randState[c][1:h]) - \paramlim }{ \gnfs[c]{\sdfedavgop[c]{\step,h}(\paramw; \randState[c][1:h])}{\randState[c][h+1]} 
- \gnfs[c]{\hsdfedavgop[c]{\step,h}(\paramw; \randState[c][1:h])}{\randState[c][h+1]} 
}
\\
\label{eq:proof-bound-heter-expand-pscal-interm}
& \quad \le
\frac{\step \strcvx}{6} 
\norm{ 
\hsdfedavgop[c]{\step,h}(\paramw; \randState[c][1:h]) - \paramlim }^2
+ 
\frac{6 \step}{\strcvx} \norm{ \gnfs[c]{\sdfedavgop[c]{\step,h}(\paramw; \randState[c][1:h])}{\randState[c][h+1]} 
- \gnfs[c]{\hsdfedavgop[c]{\step,h}(\paramw; \randState[c][1:h])}{\randState[c][h+1]} 
}^2
\eqsp.
\end{align}
Plugging \eqref{eq:proof-bound-heter-expand-norm-sq-interm} and \eqref{eq:proof-bound-heter-expand-pscal-interm} in \eqref{eq:proof-bound-heter-expand-first-interm}, and using derivations similar to \eqref{eq:proof-crude-bound-heter-recursion-one-step} from \Cref{lem:crude-bound-second-moment}'s proof, we obtain
\begin{align*}
& \norm{ \hsdfedavgop[c]{\step,h+1}(\paramw; \randState[c][1:h+1]) - \paramlim }^2
\\
& \le
(1 + \step \strcvx/6) \norm{ \hsdfedavgop[c]{\step,h}(\paramw; \randState[c][1:h]) - \paramlim }^2
- 2 \step 
\pscal{ 
\hsdfedavgop[c]{\step,h}(\paramw; \randState[c][1:h]) - \paramlim }{ \gnfs[c]{\hsdfedavgop[c]{\step,h}(\paramw; \randState[c][1:h])}{\randState[c][h+1]} - \gnf[c]{\paramlim}
}
\\
& \quad
+
2 \step^2 \norm{ \gnfs[c]{\hsdfedavgop[c]{\step,h}(\paramw; \randState[c][1:h])}{\randState[c][h+1]} - \gnfs[c]{\paramlim}{\randState[c][h+1]} }^2
\\
& \quad
+ \frac{10 \step}{\strcvx} \norm{ \gnfs[c]{\sdfedavgop[c]{\step,h}(\paramw; \randState[c][1:h])}{\randState[c][h+1]} 
- \gnfs[c]{\hsdfedavgop[c]{\step,h}(\paramw; \randState[c][1:h])}{\randState[c][h+1]} }^2
+ 4 \step^2 \norm{ \gnfs[c]{\paramlim}{\randState[c][h+1]} 
- \gnf[c]{\paramlim} }^2
\eqsp,
\end{align*}
where we also used $4 \step^2 \le \frac{4 \step}{\explip}  \le \frac{4 \step}{\strcvx} $.
Then, we expand the third moment of this equation, similarly to the proof of \Cref{lem:bound-moments-homogeneous}-\eqref{eq:proof-moment-homogeneous-expand-abc}, with
\begin{align*}
a^2
& = (1 + \step \strcvx/6) \norm{ \hsdfedavgop[c]{\step,h}(\paramw; \randState[c][1:h]) - \paramlim }^2
\eqsp,
\\
-2 \step b 
& = - 2 \step 
\pscal{ 
\hsdfedavgop[c]{\step,h}(\paramw; \randState[c][1:h]) - \paramlim }{ \gnfs[c]{\hsdfedavgop[c]{\step,h}(\paramw; \randState[c][1:h])}{\randState[c][h+1]} - \gnf[c]{\paramlim}
}
\\
\step^2 c^2
& =
2 \step^2 \norm{ \gnfs[c]{\hsdfedavgop[c]{\step,h}(\paramw; \randState[c][1:h])}{\randState[c][h+1]} - \gnfs[c]{\paramlim}{\randState[c][h+1]} }^2
\\
& \quad
+ \frac{10 \step}{\strcvx} \norm{ \gnfs[c]{\sdfedavgop[c]{\step,h}(\paramw; \randState[c][1:h])}{\randState[c][h+1]} 
- \gnfs[c]{\hsdfedavgop[c]{\step,h}(\paramw; \randState[c][1:h])}{\randState[c][h+1]} }^2
+ 4 \step^2 \norm{ \gnfs[c]{\paramlim}{\randState[c][h+1]} 
- \gnf[c]{\paramlim} }^2
\eqsp.
\end{align*}
First, we notice that by \Cref{assum:local_functions} and since $\step \strcvx \le 1$, we have $-\step b \le -\frac{\step \strcvx}{(1+\step \strcvx/6)} a^2 \le -\frac{\step \strcvx}{2} a^2$.
Additionally, we have, as in \Cref{lem:bound-moments-homogeneous}'s proof, that $b \le a c$.

Now, we remark, since the function $x \mapsto x^{1/2}$ is sub-additive, and $(x+y+z)^k \le 3^{k-1}(x^k+y^k+z^k)$ for all $x,y,z \ge 0$, we have that, for $k \ge 2$,
\begin{align*}
c^k 
& \le
3^{k-1} 2^k \norm{ \gnfs[c]{\hsdfedavgop[c]{\step,h}(\paramw; \randState[c][1:h])}{\randState[c][h+1]} - \gnfs[c]{\paramlim}{\randState[c][h+1]} }^k
\\
& \quad
+ \frac{3^{k-1} 10^k}{\step^{k/2}\strcvx^{k/2}} \norm{ \gnfs[c]{\sdfedavgop[c]{\step,h}(\paramw; \randState[c][1:h])}{\randState[c][h+1]}
- \gnfs[c]{\hsdfedavgop[c]{\step,h}(\paramw; \randState[c][1:h])}{\randState[c][h+1]} }^k
+ 3^{k-1} 4^k \norm{ \gnfs[c]{\paramlim}{\randState[c][h+1]} 
- \gnf[c]{\paramlim} }^k
\\
& =
2 \cdot 6^{k-1} \norm{ \gnfs[c]{\hsdfedavgop[c]{\step,h}(\paramw; \randState[c][1:h])}{\randState[c][h+1]} - \gnfs[c]{\paramlim}{\randState[c][h+1]} }^k
\\
& \quad
+ \frac{10 \cdot 30^{k-1}}{\step^{k/2}\strcvx^{k/2}} \norm{ \gnfs[c]{\sdfedavgop[c]{\step,h}(\paramw; \randState[c][1:h])}{\randState[c][h+1]}
- \gnfs[c]{\hsdfedavgop[c]{\step,h}(\paramw; \randState[c][1:h])}{\randState[c][h+1]} }^k
+ 4 \cdot 12^{k-1} \norm{ \gnfs[c]{\paramlim}{\randState[c][h+1]} 
- \gnf[c]{\paramlim} }^k
\eqsp.
\end{align*}
Similarly to the homogeneous case, we use \Cref{assum:local_functions}, \Cref{assum:smooth-var}, as well as the definition of $\hsdfedavgop[c]{\step,h}(\paramw; \randState[c][1:h])$ in \eqref{eq:def-operator-V-local-heterogeneous} to obtain
\begin{align*}
\CPE{ c^k }{\mcF_c^h} 
& \le
2 \cdot 6^{k-1} \explip^{k-2}\norm{ \hsdfedavgop[c]{\step,h}(\paramw; \randState[c][1:h]) - \paramlim }^{k-2}
\CPE{ \norm{ \gnfs[c]{\hsdfedavgop[c]{\step,h}(\paramw; \randState[c][1:h])}{\randState[c][h+1]} - \gnf[c]{\paramlim} }^2 }{\mcF_c^h} 
\\
& \quad
+ \frac{10 \cdot 30^{k-1} \step^{3k/2} \explip^k h^{2k}}{\strcvx^{k/2}} \norm{ \gnf[c]{\paramlim} }^k
+ 4 \cdot 12^{k-1} \Msmoothcstvar^k
\\
& \le
2 \cdot 6^{k-1} \explip^{k-1}\norm{ \hsdfedavgop[c]{\step,h}(\paramw; \randState[c][1:h]) - \paramlim }^{k-2}
\pscal{ \gnf[c]{\hsdfedavgop[c]{\step,h}(\paramw; \randState[c][1:h])} - \gnf[c]{\paramlim} }{\hsdfedavgop[c]{\step,h}(\paramw; \randState[c][1:h]) - \paramlim } 
\\
& \quad
+ \frac{10 \cdot 30^{k-1} \step^{3k/2} \explip^k h^{2k}}{\strcvx^{k/2}} \norm{ \gnf[c]{\paramlim} }^k
+ 4 \cdot 12^{k-1} \Msmoothcstvar^k
\eqsp.
\end{align*}
Which in turn proves that
\begin{align*}
& \CPE{ \step^{k} a^{6-k} c^k }{\mcF_c^h} 
\\
& \le
2 \cdot 6^{k-1} \step^{k} \explip^{k-2}\norm{ \hsdfedavgop[c]{\step,h}(\paramw; \randState[c][1:h]) - \paramlim }^{6-k+k-2}
\pscal{ \gnf[c]{\hsdfedavgop[c]{\step,h}(\paramw; \randState[c][1:h])} - \gnf[c]{\paramlim} }{\hsdfedavgop[c]{\step,h}(\paramw; \randState[c][1:h]) - \paramlim } 
\\
& \quad
+ \frac{10 \cdot 30^{k-1} \step^{5k/2} \explip^k h^{2k}}{\strcvx^{k/2}} \norm{ \hsdfedavgop[c]{\step,h}(\paramw; \randState[c][1:h]) - \paramlim }^{6-k} \norm{ \gnf[c]{\paramlim} }^k
+ 4 \cdot 12^{k-1} \step^k \norm{ \hsdfedavgop[c]{\step,h}(\paramw; \randState[c][1:h]) - \paramlim }^{6-k} \Msmoothcstvar^k
\\
& =
2 \cdot 6^{k-1} \step^{k} \explip^{k-1} \norm{ \sdfedavgop[c]{\step,h}(\paramw; \randState[c][1:h]) - \paramlim }^{4}
\pscal{ \gnf[c]{\sdfedavgop[c]{\step,h}(\paramw; \randState[c][1:h])} - \gnf[c]{\paramlim} }{ \sdfedavgop[c]{\step,h}(\paramw; \randState[c][1:h]) - \paramlim }
\\
& \quad 
+ \frac{10 \cdot 30^{k-1} \step^{5k/2} \explip^k h^{2k}}{\strcvx^{k/2}} \norm{ \hsdfedavgop[c]{\step,h}(\paramw; \randState[c][1:h]) - \paramlim }^{6-k} \norm{ \gnf[c]{\paramlim} }^k
+ 4 \cdot 12^{k-1} \step^k \norm{ \hsdfedavgop[c]{\step,h}(\paramw; \randState[c][1:h]) - \paramlim }^{6-k} \Msmoothcstvar^k
\eqsp.
\end{align*}
Proceeding as in \eqref{eq:proof-crude-high-homogeneous-all-in-condexp}, we plug this bound in the conditional expectation of \eqref{eq:proof-moment-homogeneous-expand-abc}, and take $\step \explip \le 1/45$, which gives
\begin{align*}
(a^2-2 \gamma b+ \gamma^2 c^2)^3
& \le 
a^6 
+ 
\Big(
-6 \step + 2 \cdot 6 \cdot 15 \step^2 \explip + 2 \cdot 6^2 \cdot 20 \step^3 \explip^2 + 2 \cdot 6^3 \cdot 15 \step^4 \explip^3 + 2 \cdot 6^4 \cdot 6 \step^5 \explip^4 + 2 \cdot 6^5 \step^6 \explip^5   
\Big)
\\
& \quad \qquad \qquad \times
\norm{ \sdfedavgop[c]{\step,h}(\paramw; \randState[c][1:h]) - \paramlim }^{4} \pscal{ \gnf[c]{\sdfedavgop[c]{\step,h}(\paramw; \randState[c][1:h])} }{ \sdfedavgop[c]{\step,h}(\paramw; \randState[c][1:h]) - \paramlim }
\\
& 
+ 20 \sum_{k=2}^6 2^{k-1}  \step^{k} \norm{ \sdfedavgop[c]{\step,h}(\paramw; \randState[c][1:h]) - \paramlim }^{6 - k}
\left\{ 
\frac{10 \cdot 30^{k-1} \step^{5k/2} \explip^k h^{k}}{\strcvx^{k/2}} \norm{ \gnf[c]{\paramlim} }^k
+ 4 \cdot 12^{k-1} \step^k \Msmoothcstvar^k
\right\}
\\
& \le 
a^6 
- \step \norm{ \sdfedavgop[c]{\step,h}(\paramw; \randState[c][1:h]) - \paramlim }^{4} \pscal{ \gnf[c]{\sdfedavgop[c]{\step,h}(\paramw; \randState[c][1:h])} }{ \sdfedavgop[c]{\step,h}(\paramw; \randState[c][1:h]) - \paramlim }
\\
& 
+ 2 \cdot 20 \cdot 30 \sum_{k=2}^6 2^{k-1}  \step^{k} \norm{ \sdfedavgop[c]{\step,h}(\paramw; \randState[c][1:h]) - \paramlim }^{6 - k}
\min \left\{ 
\frac{\step^{3/2} h}{\strcvx} \norm{ \gnf[c]{\paramlim} }, \cdot 12 \step \Msmoothcstvar
\right\}^k
\eqsp. 
\end{align*}
We now upper bound this sum by the third-power of a sum of two terms: one contraction, and one additive term due to stochasticity.
Let $k = 2\ell + 1 \in \iint{2}{6}$ be an odd number, which implies $\ell = 1$ or $\ell=2$.
Since $k \ge 2$, then $\ell \ge 1$, and $k \ge 3$.
Using the fact that for odd values of $k=2\ell+1$, then $k-1=2\ell \ge 2$ is even, we have, denoting $\Xi = \min \left\{ 
\frac{\step^{3/2} h}{\strcvx^{1/2}} \norm{ \gnf[c]{\paramlim} }, \cdot 12 \step \Msmoothcstvar
\right\}$,
\begin{align*}
\norm{ \sdfedavgop[c]{\step,h}(\paramw; \randState[c][1:h]) - \paramlim }^{6 - k}
\Xi^k
& =
\norm{ \sdfedavgop[c]{\step,h}(\paramw; \randState[c][1:h]) - \paramlim }^{5 - 2\ell} \Xi^{2 \ell + 1}
\\
& =
\norm{ \sdfedavgop[c]{\step,h}(\paramw; \randState[c][1:h]) - \paramlim }^{4 - 2\ell} \Xi^{2\ell} \left( \norm{ \sdfedavgop[c]{\step,h}(\paramw; \randState[c][1:h]) - \paramlim } \Xi \right)
\\
& \le
\norm{ \sdfedavgop[c]{\step,h}(\paramw; \randState[c][1:h]) - \paramlim }^{4 - 2\ell} \Xi^{2\ell} \left( 2 \norm{ \sdfedavgop[c]{\step,h}(\paramw; \randState[c][1:h]) - \paramlim }^2 + 2 \Xi^2 \right)
\eqsp.
\end{align*}
Following the lines of \eqref{eq:proof-crude-high-homogeneous-factorized-bound}, using the above inequalities, Hölder's inequality, and following \citet{dieuleveut2020bridging}'s Lemma 13, there exists a constant $\beta > 0$ such that
\begin{align*}
\PE \left[ \norm{ \sdfedavgop[c]{\step,h+1}(\paramw; \randState[c][1:h+1]) - \paramlim }^{6} \right]
& \le
\left( 
(1 + \step \strcvx/6)(1 - \step \strcvx/3) \PE\left[\norm{ \sdfedavgop[c]{\step,h}(\paramw; \randState[c][1:h]) - \paramlim }^{6} \right]^{1/3}
+ 
\beta \Xi^2/2
\right)^3
\\
& \le
\left( 
(1 - \step \strcvx/6) \PE\left[\norm{ \sdfedavgop[c]{\step,h}(\paramw; \randState[c][1:h]) - \paramlim }^{6} \right]^{1/3}
+ 
\beta \Xi^2/2
\right)^3
\eqsp.
\end{align*}
Taking the third root, we have
\begin{align*}
\PE \left[ \norm{ \sdfedavgop[c]{\step,h+1}(\paramw; \randState[c][1:h+1]) - \paramlim }^{6} \right]^{1/3}
& \le
(1 - \step \strcvx/18) \PE\left[\norm{ \sdfedavgop[c]{\step,h}(\paramw; \randState[c][1:h]) - \paramlim }^{6} \right]^{1/3}
+ \beta \frac{\step^{3} h^2}{\strcvx} \norm{ \gnf[c]{\paramlim} }^2
+ 12 \beta \step^2 \Msmoothcstvar^2
\eqsp.
\end{align*}
After $\nlupdates$ iterations, we thus have, using Minkowski's inequality, and \Cref{assum:heterogeneity} to bound $\frac{1}{\nagent} \sum_{c=1}^\nagent \norm{\gnf[c]{\paramlim}}^2$,
\begin{align*}
\PE \left[ \norm{ \sdfedavgop[c]{\step,h+1}(\paramw; \randState[c][1:h+1]) - \paramlim }^{6} \right]^{1/3}
& \le
(1 - \step \strcvx/18)^\nlupdates \PE\left[\norm{ \paramw - \paramlim }^{6} \right]^{1/3}
+ \beta \frac{\step^{3} (\nlupdates-1)\nlupdates^2}{\strcvx} \heterboundgrad
+ 12 \beta \step^2 \Msmoothcstvar^2
\eqsp,
\end{align*}
and the first part of the result follows from iterating this inequality $T$ times, starting from $\param[T]$.

The second part of the result for $p=2$ follows from the previous inequality. To obtain the result for $p=2$ we use Hölder inequality and \Cref{lem:crude-bound-second-moment}, and proceed as in \Cref{lem:bound-moments-homogeneous}.
\end{proof}

\subsection{Convergence to a neighboorhood of $\statdistlim{\step,\nlupdates}$ -- Proof of \Cref{prop:conv-neighborhood-stat-dist}}
\label{sec:conv-neighborhood-proof}

\convrateneighborhoodfedavg*
\begin{proof}

\textbf{Decomposition of the error.}
Let $\param[t] \in \rset^d$ be the global iterates of \FedAvg with step size $\step$ and number of local updates $\nlupdates$, obtained by starting at a point $\param[0] \in \rset^d$, with noise sequence $\randState[1:\nagent,1:\nrounds][1:\nlupdates]$.
We define another sequence $\varparam_{t}$, analogous to the $\param[t]$'s, but where the first point $\varparam[0] \sim \statdist{\step, \nlupdates}$ is directly sampled from the stationary distribution, and where the next iterates are generated by \FedAvg with the same noise sequence $\randState[1:\nagent,1:\nrounds][1:\nlupdates]$ as the original sequence of iterates $\param[t]$'s.

Using the identity $\norm{ u + v }^2 \le 2 \norm{u}^2 + 2 \norm{v}^2$ for any vectors $u, v \in \rset^d$, we can split the quadratic error as
\begin{align}
\label{eq:ineq-proof-cv-to-sol-bound-decomp}
  \norm{ \param[t] - \statdistlim{\step, \nlupdates} }^2
  & \le
    2 \norm{ \param[t] - \varparam_{t} }^2    
    + 2 \norm{ \varparam_{t} - \statdistlim{\step, \nlupdates} }^2       
    \eqsp.
\end{align}

\textbf{Bound on forgetting of initial conditions.}
The first term controls forgetting of the initial conditions.
From Lemma 4, it is upper bounded by
\begin{align*}
  \PE[\norm{ \param[t] - \varparam_{t} }^2]
  & \le
    (1 - \step \strcvx)^{\nlupdates t} \norm{ \param[0] - \varparam_{0} }^2
    \eqsp.
\end{align*}
Using Young's inequality to bound $\norm{ \param[0] - \varparam_{0} }^2$, and \Cref{lem:crude-bound-second-moment} to bound the error's second moment in the stationary distribution, we can further decompose
\begin{align*}
\norm{ \param[0] - \varparam_{0} }^2
& \le 
2 \norm{ \param[0] - \paramlim }^2
+ 2 \norm{ \varparam_{0} - \paramlim }^2
\le 
2 \norm{ \param[0] - \paramlim }^2
+ \frac{12 \nlupdates^2 \step^2 \explip^2 \heterboundgrad^2}{\strcvx^2} 
+ \frac{16 \step}{\strcvx} \Msmoothcstvar^{2}
\eqsp.
\end{align*}
This gives the bound
\begin{align}
\label{eq:ineq-proof-cv-to-sol-bound-cv}
  \PE[\norm{ \param[t] - \varparam_{t} }^2]
  & \le
    (1 - \step \strcvx)^{\nlupdates t} \Bigg\{ 
    2 \norm{ \param[0] - \paramlim }^2
+ \frac{12 \nlupdates^2 \step^2 \explip^2 \heterboundgrad^2}{\strcvx^2} 
+ \frac{16 \step}{\strcvx} \Msmoothcstvar^{2}
\Bigg\} \eqsp.
\end{align}

\textbf{Bound on the variance.}
The second term $\PE[\norm{ \varparam_{t} - \statdistlim{\step, \nlupdates} }^2]$ is a variance term.
Since $\varparam_0$ is sampled from the stationary distribution $\statdist{\step, \nlupdates}$, it also holds that $\varparam_t \sim \statdist{\step, \nlupdates}$ for all $t \ge 0$.
Moreover, by definition of $\statdistlim{\step, \nlupdates}$, we have $\statdistlim{\step, \nlupdates} = \PE[ \sdfedavgop{\step,\nlupdates}(\varparam_0; \randState[1:\nagent][1:\nlupdates]) ]$.
Then, by Jensen's inequality, we have
\begin{align}
\label{eq:ineq-proof-cv-to-sol-after-jensen}
\PE\left[ \norm{ \varparam_{t} - \statdistlim{\step, \nlupdates} }^2 \right]
= \PE\left[ \norm{ \varparam_{1} - \statdistlim{\step, \nlupdates} }^2 \right]
& \le
\frac{1}{\nagent} 
\sum_{c=1}^\nagent 
\PE\left[ \norm{ \sdfedavgop[c]{\step,\nlupdates}(\varparam_{0}; \randState[c,t][1:\nlupdates]) - \PE[ \sdfedavgop[c]{\step,\nlupdates}(\varparam_0; \randStatebis[c][1:\nlupdates])] }^2 \right]
\eqsp.
\end{align}
We bound each term of this sum by induction. Let $h \in \iint{1}{\nlupdates}$, and $\mcF_c^h = \sigma( \randState[c][1:h], \randStatebis[c][1:h])$, then we have
\begin{align}
\nonumber
& \CPE{ \bnorm{ \sdfedavgop[c]{\step,h+1}(\varparam_{0}; \randState[c,t][1:h+1]) - \PE[ \sdfedavgop[c]{\step,h+1}(\varparam_0; \randState[c][1:h+1])] }^2 }{\mcF_c^h}
\\
\nonumber
& =
\bnorm{ \sdfedavgop[c]{\step,h}(\varparam_{0}; \randState[c,t][1:h]) - \PE[ \sdfedavgop[c]{\step,h}(\varparam_0; \randState[c][1:h])] }^2
\\
\nonumber
& \quad - 2 \step \bpscal{\sdfedavgop[c]{\step,h}(\varparam_{0}; \randState[c,t][1:h]) - \PE[ \sdfedavgop[c]{\step,h}(\varparam_0; \randState[c][1:h])]} 
{\gnf[c]{\sdfedavgop[c]{\step,h}(\varparam_{0}; \randState[c,t][1:h])} - \PE[ \gnf[c]{\sdfedavgop[c]{\step,h}(\varparam_0; \randState[c][1:h])} ]}
\\
\label{eq:proof-bound-var-cv-sol-expanded}
& \quad 
+ \step^2 \CPE{ \bnorm{\gnfs[c]{\sdfedavgop[c]{\step,h}(\varparam_{0}; \randState[c,t][1:h])}{\randState[c][h+1]} - \PE[ \gnf[c]{ \sdfedavgop[c]{\step,h}(\varparam_0; \randState[c][1:h])} ]}^2 }{\mcF_c^h}
\eqsp.
\end{align}
By \Cref{assum:unif-bound} and using twice the inequality $\norm{ u + v }^2 \le 2 \norm{ u }^2 + 2 \norm{ v }^2$ for any $u, v \in \rset^d$, then using Jensen's inequality, we can bound
\begin{align*}
& \CPE{ \bnorm{\gnfs[c]{\sdfedavgop[c]{\step,h}(\varparam_{0}; \randState[c,t][1:h])}{\randState[c][h+1]} - \PE[ \gnf[c]{ \sdfedavgop[c]{\step,h}(\varparam_0; \randState[c][1:h])} ]}^2 }{\mcF_c^h}
\\ & \quad \le
2 \bnorm{\gnf[c]{\sdfedavgop[c]{\step,h}(\varparam_{0}; \randState[c,t][1:h])} - \PE[ \gnf[c]{ \sdfedavgop[c]{\step,h}(\varparam_0; \randState[c][1:h])} ]}^2 
+
2 \Msmoothcstvarbis
\\ & \quad \le
4 \bnorm{\gnf[c]{\sdfedavgop[c]{\step,h}(\varparam_{0}; \randState[c,t][1:h])} - \gnf[c]{ \PE[ \sdfedavgop[c]{\step,h}(\varparam_0; \randState[c][1:h]) ]} }^2 
\\ & \qquad +
4 \bnorm{\gnf[c]{ \PE[ \sdfedavgop[c]{\step,h}(\varparam_0; \randState[c][1:h]) ]} - \PE[ \gnf[c]{ \sdfedavgop[c]{\step,h}(\varparam_0; \randState[c][1:h])} ]}^2 
+
2 \Msmoothcstvarbis
\\ & \quad \le
4 \bnorm{\gnf[c]{\sdfedavgop[c]{\step,h}(\varparam_{0}; \randState[c,t][1:h])} - \gnf[c]{ \PE[ \sdfedavgop[c]{\step,h}(\varparam_0; \randState[c][1:h]) ]} }^2 
\\
& \qquad +
4 \PE\left[ \bnorm{\gnf[c]{ \PE[ \sdfedavgop[c]{\step,h}(\varparam_0; \randState[c][1:h]) ]} -  \gnf[c]{ \sdfedavgop[c]{\step,h}(\varparam_0; \randState[c][1:h])} }^2 \right]
+
2 \Msmoothcstvarbis
\eqsp.
\end{align*}
Taking the expectation and using \Cref{assum:local_functions}-\ref{assum:smoothness}, this gives
\begin{align}
\nonumber
& \PE\left[  \bnorm{\gnfs[c]{\sdfedavgop[c]{\step,h}(\varparam_{0}; \randState[c,t][1:h])}{\randState[c][h+1]} - \PE[ \gnf[c]{ \sdfedavgop[c]{\step,h}(\varparam_0; \randState[c][1:h])} ]}^2 \right]
\\
\nonumber
& \le
8 \PE\left[ \bnorm{\gnf[c]{ \PE[ \sdfedavgop[c]{\step,h}(\varparam_0; \randState[c][1:h]) ]} -  \gnf[c]{ \sdfedavgop[c]{\step,h}(\varparam_0; \randState[c][1:h])} }^2 \right]
+
2 \Msmoothcstvarbis
\\
\label{eq:proof-bound-var-cv-sol-norm}
& \le
\PE\left[ 8 \explip \bpscal{\sdfedavgop[c]{\step,h}(\varparam_{0}; \randState[c,t][1:h]) - \PE[ \sdfedavgop[c]{\step,h}(\varparam_0; \randState[c][1:h])]} 
{\gnf[c]{\sdfedavgop[c]{\step,h}(\varparam_{0}; \randState[c,t][1:h])} -  \gnf[c]{\PE[\sdfedavgop[c]{\step,h}(\varparam_0; \randState[c][1:h])]} }
\right]
+ 2 \Msmoothcstvarbis
\eqsp.
\end{align}
Since $\PE[ \pscal{ \sdfedavgop[c]{\step,h}(\varparam_{0}; \randState[c,t][1:h]) - \PE[ \sdfedavgop[c]{\step,h}(\varparam_0; \randState[c][1:h])] ] }{ \PE[ \gnf[c]{\sdfedavgop[c]{\step,h}(\varparam_0; \randState[c][1:h])} ]  - \gnf[c]{\PE[\sdfedavgop[c]{\step,h}(\varparam_0; \randState[c][1:h])]} } = 0$, it holds that 
\begin{align}
\nonumber
& \PE\left[ - 2 \step \bpscal{\sdfedavgop[c]{\step,h}(\varparam_{0}; \randState[c,t][1:h]) - \PE[ \sdfedavgop[c]{\step,h}(\varparam_0; \randState[c][1:h])]} 
{\gnf[c]{\sdfedavgop[c]{\step,h}(\varparam_{0}; \randState[c,t][1:h])} - \PE[ \gnf[c]{\sdfedavgop[c]{\step,h}(\varparam_0; \randState[c][1:h])} ]}
\right]
\\
\label{eq:proof-bound-var-cv-sol-pscal}
& =
\PE\left[ - 2 \step \bpscal{\sdfedavgop[c]{\step,h}(\varparam_{0}; \randState[c,t][1:h]) - \PE[ \sdfedavgop[c]{\step,h}(\varparam_0; \randState[c][1:h])]} 
{\gnf[c]{\sdfedavgop[c]{\step,h}(\varparam_{0}; \randState[c,t][1:h])} -  \gnf[c]{\PE[\sdfedavgop[c]{\step,h}(\varparam_0; \randState[c][1:h])]} }
\right]
\eqsp.
\end{align}
Taking the expectation of \eqref{eq:proof-bound-var-cv-sol-expanded} and plugging \eqref{eq:proof-bound-var-cv-sol-norm} and \eqref{eq:proof-bound-var-cv-sol-pscal} in, then using \Cref{assum:local_functions}-\ref{assum:item_strong_convex} and the fact that $\step \le 1 / (8\explip)$, we obtain
\begin{align*}
& \CPE{ \bnorm{ \sdfedavgop[c]{\step,h+1}(\varparam_{t}; \randState[c,t][1:h+1]) \!-\! \PE[ \sdfedavgop[c]{\step,h+1}(\param; \randStatebis[c][1:h+1])] }^2\! }{\mcF_c^h}
\\
& \le
\PE\left[ \bnorm{ \sdfedavgop[c]{\step,h}(\varparam_{0}; \randState[c,t][1:h]) - \PE[ \sdfedavgop[c]{\step,h}(\varparam_0; \randState[c][1:h])] }^2 \right]
+ 2 \Msmoothcstvarbis
\\
& \quad +
\left( 8 \step^2 \explip - 2 \step \right) \PE\left[ \step \bpscal{\sdfedavgop[c]{\step,h}(\varparam_{0}; \randState[c,t][1:h]) - \PE[ \sdfedavgop[c]{\step,h}(\varparam_0; \randState[c][1:h])]} 
{\gnf[c]{\sdfedavgop[c]{\step,h}(\varparam_{0}; \randState[c,t][1:h])} -  \gnf[c]{\PE[\sdfedavgop[c]{\step,h}(\varparam_0; \randState[c][1:h])]} }
\right]
\\
& \le
\left( 1 - \step \strcvx \right)
\bnorm{ \sdfedavgop[c]{\step,h}(\varparam_{t}; \randState[c,t][1:h]) \!-\! \PE[ \sdfedavgop[c]{\step,h}(\param; \randStatebis[c][1:h])] }^2
\!\!+ 2 \Msmoothcstvarbis^2
\eqsp.
\end{align*}
Unrolling the recursion and plugging the result in \eqref{eq:ineq-proof-cv-to-sol-after-jensen}, we obtain
\begin{align}
\label{eq:ineq-proof-cv-to-sol-bound-before-integration}
&
\PE\left[ \norm{ \varparam_{1} - \statdistlim{\step, \nlupdates} }^2 \right]
\le
\left( 1 - \step \strcvx \right)^\nlupdates
\PE\left[ \bnorm{ \varparam_{0} - \statdistlim{\step, \nlupdates} }^2 \right]
+ 2 \step^2 \nlupdates \Msmoothcstvarbis^2
\eqsp.
\end{align}
And \eqref{eq:ineq-proof-cv-to-sol-bound-before-integration} can be rewritten
\begin{align*}
&
\int 
\norm{ \varparam - \statdistlim{\step, \nlupdates} }^2
\statdist{\step, \nlupdates}(\rmd \varparam)
\le
\left( 1 - \step \strcvx \right)^\nlupdates
\int 
\norm{ \varparam - \statdistlim{\step, \nlupdates} }^2
\statdist{\step, \nlupdates}(\rmd \varparam)
+ 2 \step^2 \nlupdates \Msmoothcstvarbis^2
\eqsp.
\end{align*}
Thus, we obtain that
\begin{align}
\label{eq:ineq-proof-cv-to-sol-bound-variance}
&
\int 
\norm{ \varparam - \statdistlim{\step, \nlupdates} }^2
\statdist{\step, \nlupdates}(\rmd \varparam)
\le
\frac{2 \step \Msmoothcstvarbis^2}{\strcvx}
\eqsp.
\end{align}

\textbf{Final result.}
The result of the lemma follows from plugging \eqref{eq:ineq-proof-cv-to-sol-bound-cv} and \eqref{eq:ineq-proof-cv-to-sol-bound-variance} in \eqref{eq:ineq-proof-cv-to-sol-bound-decomp} and integrating the result over the stationary distribution $\statdist{\step, \nlupdates}(\rmd \varparam)$.
\end{proof}

%% file: 2024-AISTATS/src/aistats-app-analysis-fedavg-sto-quad.tex
\subsubsection{Study of the Bias}

In this section, we study the particular case where the functions $\nfw[c]$'s are quadratic.
Specifically, we assume that there exist symmetric matrices $\nbarA[c]$'s and vectors $\paramlim_c$'s such that
\begin{align*}
  \nf[c]{\param} = \frac{1}{2}\bnorm{ (\nbarA[c]{})^{1/2} (\param - \paramlim_c ) }^2
  \eqsp.
\end{align*}
This implies that $\nfw[c]$'s gradients are linear, and satisfy $\gnf[c]{\param} = \nbarA[c] (\param - \paramlim_c) $.
Consequently, for all $h \le \nlupdates$, 
$\PE[ \sdfedavgop[c]{\step,H}\left(\paramw; \randState[c][1:\nlupdates]\right) ]
  - \paramlim_c =  (\Id - \step \nbarA[c])^h ( \param - \paramlim_c )$.
For further analysis, we recall the matrices introduced in \eqref{eq:def-contract-sto-mat} and introduce the intermediate matrices $\contract[c]{\star,h+1:\nlupdates}$,
\begin{align}
\label{eq:matrices-quadratic-expansions}
  \contract[c]{\star,h+1:\nlupdates}
  = 
  (\Id - \step \nbarA[c])^{\nlupdates-h}
  \eqsp,
  \quad
  \contract[c]{\star}
  = 
  (\Id - \step \nbarA[c])^\nlupdates
  \eqsp,
  \quad
  \fullcontract{\star}
  =
  \frac{1}{\nagent}
  \sum_{c=1}^\nagent \contract[c]{\star}
  \eqsp.
\end{align}
Refined
Now, we give a proof of \Cref{thm:bias-quadratic-sto}, that we restate here for readability. 
\biasfedavgquadratic*
We prove the explicit expression of the bias and the upper bound from \Cref{thm:bias-quadratic-sto} in \Cref{prop:bias-quadratic-sto}, and give the first-order expansion of the bias in \Cref{cor:exp-quadratic-sto}.
\begin{proposition}[Bias of \FedAvg for Quadratics]
\label{prop:bias-quadratic-sto}
  Assume \Cref{assum:local_functions}, \Cref{assum:heterogeneity}, \Cref{assum:smooth-var}, \Cref{ass:quadratic}, and $\step \le 1/L$, then the bias of \FedAvg with quadratic functions is
  \begin{align*}
    \statdistlim{\step, \nlupdates}
    & = 
      \paramlim
      +
      (\Id - \fullcontract{\star})^{-1}
      \cdot
      \frac{1}{\nagent} \sum_{c=1}^\nagent 
      (\Id - \contract[c]{\star}) ( \paramlim - \paramlim_c )
      \eqsp.
  \end{align*}
  Furthermore, when $\step \strcvx \nlupdates \le 1$, it holds that
  \begin{align*}
    \bnorm{ \statdistlim{\step, \nlupdates} - \paramlim }
    \le
    \frac{\step (\nlupdates-1) \heterbound \heterboundgrad}{2 \strcvx}
    \eqsp.
  \end{align*}
\end{proposition}
\begin{proof}
  Using derivations similar to the proof of \Cref{prop:bias-det-fedavg}, or following the decomposition derived in the Section 3 of \citet{mangold2024scafflsa}, we have, for any point $\param \in \rset^d$, it holds, for $c \in \iint{1}{\nagent}$, that
  \begin{align}
  \nonumber
    \sdfedavgop[c]{\step,H}\left(\paramw; \randState[c][1:\nlupdates]\right) - \paramlim
    & =
    \sdfedavgop[c]{\step,H}\left(\paramw; \randState[c][1:\nlupdates]\right) - \paramlim_c + \paramlim_c - \paramlim
    \\ \nonumber
    & =
    \contract[c]{\star} (\paramw - \paramlim_c) 
    + \step \sum_{h=1}^\nlupdates \contract[c]{\star,h+1:\nlupdates}  \updatefuncnoise[c]{\randState[c][1:h]}\sdfedavgop{\step,h}\left(\paramw; \randState[c][1:h]\right)
     + \paramlim_c - \paramlim
    \\
  \label{eq:expansion-quadratic-local-updates-H}
    & =
    \contract[c]{\star} (\paramw - \paramlim)
    + (\contract[c]{\star} - \Id) (\paramlim - \paramlim_c) 
    + \step \sum_{h=1}^\nlupdates \contract[c]{\star,h+1:\nlupdates}  \updatefuncnoise[c]{\randState[c][1:h]}\sdfedavgop{\step,h}\left(\paramw; \randState[c][1:h]\right)
    \eqsp,
  \end{align}
  where $\updatefuncnoise[c]{z}$ is defined in \eqref{eq:def-epsilon}.
  Taking the average of \eqref{eq:expansion-quadratic-local-updates-H} for $c = 1 \cdots \nagent$ and taking the expectation, we obtain
  \begin{align*}
    \PE[\sdfedavgop{\step,H}\left(\paramw; \randState[1:\nagent][1:\nlupdates]\right) - \paramlim ]
    & =
      \frac{1}{\nagent} \sum_{c=1}^\nagent \contract[c]{\star} (\param - \paramlim) + (\contract[c]{\star} - \Id) (\paramlim - \paramlim_c)
      \eqsp.
  \end{align*}
  When $\param \sim \statdist{\step}$ is sampled from the stationary distribution of \FedAvg's iterates, we have $\statdistlim{\step,\nlupdates} = \PE[ \param ] = \PE[\sdfedavgop[\nlupdates]{\randState} \param ]$.
  This gives the equation
  \begin{align*}
    \statdistlim{\step,\nlupdates} - \paramlim
    & =
      \fullcontract{\star} (\statdistlim{\step, \nlupdates} - \paramlim) + \frac{1}{\nagent} \sum_{c=1}^\nagent (\contract[c]{\star} - \Id) (\paramlim - \paramlim_c)
      \eqsp.
  \end{align*}
  Subtracting $\fullcontract{\star} (\statdistlim{\step, \nlupdates} - \paramlim_c)$ on both side, and multiplying by $(\Id - \fullcontract{\star})^{-1}$, we obtain the following expression for $\statdistlim{\step, \nlupdates}$ as a function of $\paramlim$,
  \begin{align*}
    \statdistlim{\step, \nlupdates}
    & = 
      \paramlim
      +
      (\Id - \fullcontract{\star})^{-1}
      \cdot
      \frac{1}{\nagent} \sum_{c=1}^\nagent 
      (\Id - \contract[c]{\star}) ( \paramlim_c - \paramlim )
      \eqsp,
  \end{align*}
  which gives the first part of the result. Then, using the Neumann series together with \Cref{lem:product_coupling_lemma}, we obtain
  \begin{align*}
    \statdistlim{\step, \nlupdates}
    & = 
      \paramlim
      +
      \sum_{t=0}^\infty
      (\fullcontract{\star})^t
      \cdot
      \frac{1}{\nagent} 
      \sum_{c=1}^\nagent 
      \sum_{h=0}^\nlupdates 
      \step 
      \contract[c]{\star,h+1:\nlupdates} \nbarA[c] ( \paramlim - \paramlim_c )
    \\
    & = 
      \paramlim
      +
      \sum_{t=0}^\infty
      (\fullcontract{\star})^t
      \cdot
      \frac{1}{\nagent} 
      \sum_{c=1}^\nagent 
      \sum_{h=0}^\nlupdates 
      \step 
      \left( \contract[c]{\star,h+1:\nlupdates} - \fullcontract{\star,h+1:\nlupdates}_{\text{avg}} \right) \nbarA[c] ( \paramlim - \paramlim_c )
      \eqsp,
  \end{align*}
  where we defined the notation $\fullcontract{\star,h+1:\nlupdates}_{\text{avg}} = \prod_{h+1}^\nlupdates (\Id - \step \nbarA)$, and the second inequality comes from the fact that $\fullcontract{\star,h+1:\nlupdates}_{\text{avg}} \sum_{c=1}^\nagent \nbarA[c] (\paramlim - \paramlim_c) = 0$.
  Now, we note that 
  \begin{align*}
    \contract[c]{\star,h+1:\nlupdates} - \fullcontract{\star,h+1:\nlupdates}_{\text{avg}} 
    & =
     \sum_{\ell=h+1}^\nlupdates \contract[c]{\star,h+1:\ell-1} (\step \nbarA[c] - \step \nbarA) \fullcontract{\star,\ell+1:\nlupdates}_{\text{avg}} 
      \eqsp.
  \end{align*}
  Therefore, we have
  \begin{align}
      \frac{1}{\nagent} \sum_{c=1}^\nagent 
      (\Id - \contract[c]{\star}) ( \paramlim_c - \paramlim )
      & =
      \frac{1}{\nagent} 
      \sum_{c=1}^\nagent 
      \sum_{h=0}^\nlupdates 
      \step 
      \left( \contract[c]{\star,h+1:\nlupdates} - \fullcontract{\star,h+1:\nlupdates}_{\text{avg}} \right) \nbarA[c] ( \paramlim - \paramlim_c )
      \nonumber
      \\
      & =     
      \frac{\step^2}{\nagent} 
      \sum_{c=1}^\nagent 
      \sum_{h=0}^\nlupdates 
      \sum_{\ell=h+1}^\nlupdates \contract[c]{\star,h+1:\ell-1} ( \nbarA[c] - \nbarA) \fullcontract{\star,\ell+1:\nlupdates}_{\text{avg}} 
      \label{eq:expansion-bias-one-step-sum-het}
      \eqsp.
  \end{align}
  This yields, using the triangle inequality,
  \begin{align*}
    \bnorm{ \statdistlim{\step, \nlupdates} - \paramlim }
    & \le
     \sum_{t=0}^\infty
      (1- \step\strcvx)^{\nlupdates t}
      \cdot
      \sum_{h=0}^\nlupdates 
      \bnorm{
      \frac{1}{\nagent} 
      \sum_{c=1}^\nagent 
      \step 
      \left( \contract[c]{\star,h+1:\nlupdates} - \fullcontract{\star,h+1:\nlupdates}_{\text{avg}}\right) \nbarA[c] ( \paramlim - \paramlim_c )
      }
    \\
    & =
     \sum_{t=0}^\infty
      (1- \step\strcvx)^{\nlupdates t}
      \cdot
      \sum_{h=0}^\nlupdates 
      \bnorm{
      \frac{1}{\nagent} 
      \sum_{c=1}^\nagent 
      \step 
      \sum_{\ell=h+1}^\nlupdates \contract[c]{\star,h+1:\ell-1} (\step \nbarA[c] - \step \nbarA) \fullcontract{\star,\ell+1:\nlupdates}_{\text{avg}} 
      \nbarA[c] ( \paramlim - \paramlim_c )
      }
    \\
    & \le
     \sum_{t=0}^\infty
      (1- \step\strcvx)^{\nlupdates t}
      \cdot
      \step ^2
      \sum_{h=0}^\nlupdates 
      \sum_{\ell=h+1}^\nlupdates
      \bnorm{
      \frac{1}{\nagent} 
      \sum_{c=1}^\nagent 
       \contract[c]{\star,h+1:\ell-1} ( \nbarA[c] - \nbarA) \fullcontract{\star,\ell+1:\nlupdates}_{\text{avg}} 
      \nbarA[c] ( \paramlim - \paramlim_c )
      }
    \eqsp.
    \end{align*}
    And we obtain
    \begin{align*}
    & \bnorm{ \statdistlim{\step, \nlupdates} - \paramlim } 
    \\
    & \le
     \sum_{t=0}^\infty
      (1- \step\strcvx)^{\nlupdates t}
      \cdot
      \step ^2
      \sum_{h=0}^\nlupdates 
      \sum_{\ell=h+1}^\nlupdates
      \left(
      \frac{1}{\nagent} 
      \sum_{c=1}^\nagent 
      \bnorm{
       \contract[c]{\star,h+1:\ell-1} ( \nbarA[c] - \nbarA) \fullcontract{\star,\ell+1:\nlupdates}_{\text{avg}} 
       }^2\right)^{1/2}
       \left( \frac{1}{\nagent} \sum_{c=1}^\nagent  \norm{ \nbarA[c](\paramlim - \paramlim_c) }  \right)^{1/2}
      \\
    & \le
      \sum_{t=0}^\infty (1 - \step \strcvx)^{\nlupdates t} \step^2 \frac{\nlupdates(\nlupdates-1)}{2}
      \heterbound
      \heterboundgrad
      =
      \frac{\step (\nlupdates-1) \heterbound \heterboundgrad}{2\strcvx}
      \eqsp,
  \end{align*}
  which is the second part of the result.
\end{proof}

\begin{proposition}[Expansion of \FedAvg's Bias and Variance for Quadratics]
\label{cor:exp-quadratic-sto}  
  Assume \Cref{assum:local_functions}, \Cref{assum:heterogeneity}, \Cref{assum:smooth-var}, \Cref{ass:quadratic}, $\step \le 1/L$ and $\step \nlupdates \le 1$, then we can express $\statdistlim{\step, \nlupdates}$ as
  \begin{align*}
    \statdistlim{\step, \nlupdates}
    - \paramlim
    & =
      \frac{\step(\nlupdates-1)}{2\nagent} 
      \hf{\paramlim}^{-1} 
      \sum_{c=1}^\nagent 
      (\hnf[c]{\paramlim} - \hf{\paramlim})
      \gnf[c]{\paramlim}
      + O(\step^2 \nlupdates^2)
      \eqsp,
    \\ 
     \int \left( \param - \paramlim \right)^{\otimes 2} \statdist{\step,\nlupdates}(\rmd \param)
    &  =
    \frac{\step}{\nagent}   
    \invopcov \covfunc\left(\paramlim\right)  
    + O(\step^2 \nlupdates^2 + \step^2 \nlupdates)
    \eqsp.
  \end{align*}
\end{proposition}
\begin{proof}
\textbf{Expansion of the Bias (Quadratic Case).}
We start from the expression in \Cref{prop:bias-quadratic-sto}.
As in \Cref{prop:bias-quadratic-sto}, we use \Cref{lem:product_coupling_lemma} and the fact that $\fullcontract{\star,h+1:\nlupdates}_{\text{avg}} \sum_{c=1}^\nagent \nbarA[c] (\paramlim - \paramlim_c) = 0$ to obtain
\begin{align*}
\statdistlim{\step, \nlupdates}
& = 
  \paramlim
  +
  (\Id - \fullcontract{\star})^{-1}
  \cdot
  \frac{1}{\nagent} 
  \sum_{c=1}^\nagent 
  \sum_{h=0}^\nlupdates 
  \step 
  \left( \contract[c]{\star,h+1:\nlupdates} - \fullcontract{\star,h+1:\nlupdates}_{\text{avg}} \right) \nbarA[c] ( \paramlim - \paramlim_c )
  \eqsp.
\end{align*}
Then, following the proof of \Cref{prop:app-exp-bias-wdg}, we expand
\begin{align*}
\contract[c]{\star,h+1:\nlupdates} - \fullcontract{\star,h+1:\nlupdates}_{\text{avg}}
& =
(\Id - \step (\nlupdates-h-1) \nbarA[c] + O(\step^2 \nlupdates^2) - (\Id - \step \nbarA + O(\step^2 \nlupdates^2)
\\
&
=
\step (\nlupdates-h-1) (\nbarA - \nbarA[c]) + O(\step^2 \nlupdates^2)
\eqsp,
\\[0.5em]
(\Id - \fullcontract{\star})^{-1}
& =
(\Id - (\Id - \step \nlupdates \nbarA + O(\step^2 \nlupdates^2) ) )^{-1}
=
(\step \nlupdates \nbarA)^{-1} + O(\step \nlupdates)
\eqsp.
\end{align*}
Therefore, we obtain
\begin{align*}
\statdistlim{\step, \nlupdates}
& = 
  \paramlim
  +
  \left( (\step \nlupdates \nbarA)^{-1} + O(\step \nlupdates) \right)
  \cdot
  \frac{1}{\nagent} 
  \sum_{c=1}^\nagent 
  \sum_{h=0}^{\nlupdates-1} 
  \step 
  \left( \step (\nlupdates-h-1) (\nbarA - \nbarA[c]) + O(\step^2 \nlupdates^2) \right) \nbarA[c] ( \paramlim - \paramlim_c )
\\
& = 
  \paramlim
  +
  (\step \nlupdates \nbarA)^{-1}
  \frac{1}{\nagent} 
  \sum_{c=1}^\nagent  
  \left\{ 
  \step^2 \frac{\nlupdates (\nlupdates-1)}{2} (\nbarA - \nbarA[c]) \nbarA[c] ( \paramlim - \paramlim_c )
  \right\}
  + O(\step^2 \nlupdates^2) 
\\
& = 
  \paramlim
  -
  \frac{\step (\nlupdates-1)}{2 \nagent}
  \nbarA^{-1}
  \sum_{c=1}^\nagent  
  \left\{ (\nbarA[c] - \nbarA) \nbarA[c] ( \paramlim - \paramlim_c )
  \right\}
  + O(\step^2 \nlupdates^2) 
  \eqsp.
\end{align*}
Then, the result follows from $\hnf[c]{\paramlim} = \nbarA[c]$, $\hf{\paramlim} = \nbarA$ and $\gnf[c]{\paramlim} = \nbarA[c](\paramlim - \paramlim_c)$.

\textbf{Expansion of the Variance (Quadratic Case).}
Starting from \eqref{eq:expansion-quadratic-local-updates-H}, and summing for $c = 1$ to $\nagent$, we have
  \begin{align*}
  \nonumber
    \sdfedavgop[c]{\step,H}\left(\paramw; \randState[1:\nagent][1:\nlupdates]\right) - \paramlim
    & =
    \fullcontract{\star} (\paramw - \paramlim)
    + \frac{1}{\nagent} \sum_{c=1}^\nagent (\contract[c]{\star} - \Id) (\paramlim - \paramlim_c) 
    + \frac{\step}{\nagent} \sum_{h=1}^\nlupdates \contract[c]{\star,h+1:\nlupdates}  \updatefuncnoise[c]{\randState[c][1:h]}\sdfedavgop{\step,h}\left(\paramw; \randState[c][1:h]\right)
    \eqsp.
  \end{align*}
Taking the square and expectation of this equation, and using the fact that agents' local random variables $\randState[c][1:\nlupdates]$ are independent from one agent to another, we have
\begin{align*}
\int \left( \param - \paramlim \right)^{\otimes 2} \statdist{\step,\nlupdates}(\rmd \param)
    & =
    \int
    \left( \fullcontract{\star} (\paramw - \paramlim)
    + \frac{1}{\nagent} \sum_{c=1}^\nagent (\contract[c]{\star} - \Id) (\paramlim - \paramlim_c) 
    \right)^{\otimes 2}  \statdist{\step,\nlupdates}(\rmd \param)
    \\
    & \quad
    + \frac{\step^2}{\nagent} \sum_{c=1}^\nagent \sum_{h=1}^\nlupdates 
    \contract[c]{\star,h+1:\nlupdates}  
    \covfunc\left(\sdfedavgop{\step,h}(\paramw; \randState[c][1:h])\right)
    \contract[c]{\star,h+1:\nlupdates}  
    \eqsp,
  \end{align*}
  where $\covfunc(\param) = \PE\Big[ \frac{1}{\nagent} \sum_{c=1}^{\nagent} \updatefuncnoise[1]{1}(\param)^{\otimes 2} \Big]$.
  Then, since $(\contract[c]{\star} - \Id) (\paramlim - \paramlim_c) $ does not depend on $\param$, and by \eqref{eq:expansion-bias-one-step-sum-het} we have
  \begin{align*}
    \frac{1}{\nagent} \sum_{c=1}^\nagent (\contract[c]{\star} - \Id) (\paramlim - \paramlim_c) 
    & = O(\step^2 \nlupdates^2)
    \eqsp,
  \end{align*}
  and using the bound from \Cref{prop:bias-quadratic-sto} which guarantees that $\int (\param - \paramlim) \statdist{\step,\nlupdates}(\rmd \param) = O(\step \nlupdates)$, we obtain
  \begin{align*}
    \int \left( \fullcontract{\star} (\paramw - \paramlim)
    + (\fullcontract{\star} - \Id) (\paramlim - \paramlim_c) 
    \right)^{\otimes 2}  \statdist{\step,\nlupdates}(\rmd \param)
    & =
    \fullcontract{\star} \int \left( \paramw - \paramlim
    \right)^{\otimes 2}  \statdist{\step,\nlupdates}(\rmd \param) \fullcontract{\star} 
    + O(\step^3 \nlupdates^3)
    \eqsp.
  \end{align*}
    Expanding $\fullcontract{\star} = \Id - \step \nlupdates \nbarA$ and using \Cref{assum:smooth-var} together with \Cref{lem:crude-bound-second-moment}, we have
  \begin{align*}
\int \left( \param - \paramlim \right)^{\otimes 2} \statdist{\step,\nlupdates}(\rmd \param)
    & =
    (\Id - \step \nlupdates \nbarA)
    \int
     (\paramw - \paramlim)^{\otimes 2}  \statdist{\step,\nlupdates}(\rmd \param)
    (\Id - \step \nlupdates \nbarA)
    + \frac{\step^2 \nlupdates}{\nagent}   
    \covfunc\left(\paramlim\right)  
    + O(\step^3 \nlupdates^3 + \step^3 \nlupdates^2)
    \eqsp.
  \end{align*}
  Simplifying this equation, and using \Cref{lem:crude-bound-second-moment} again, we obtain
  \begin{align*}
    (\Id \otimes \nbarA + \nbarA \otimes \Id) 
    \int \left( \param - \paramlim \right)^{\otimes 2} \statdist{\step,\nlupdates}(\rmd \param)
    & =
    \frac{\step}{\nagent}   
    \covfunc\left(\paramlim\right)  
    + O(\step^2 \nlupdates^2 + \step^2 \nlupdates)
    \eqsp,
  \end{align*}
  and the result follows from $\invopcov = (\Id \otimes \hf{\paramlim} + \hf{\paramlim} \otimes \Id)^{-1}$ with $\hf{\paramlim} = \nbarA$, as defined in \eqref{eq:def-a-C}.
\end{proof}

%% file: 2024-AISTATS/src/aistats-app-analysis-fedavg-sto-homogeneous.tex
When functions are not quadratic and gradients are stochastic, local iterates are inherently biased.
We start in the simpler case where agents are homogeneous, which will serve as a skeleton for the general heterogeneous case.
In this setting, the functions $\nfw[c]$ are all identical, therefore we simply denote them $\fw$.

To study this case, we define the following matrices, for $h = 0$ to $\nlupdates$, that are the counterparts of the matrices defined in \eqref{eq:matrices-quadratic-expansions} in the quadratic setting, using the Hessian at the solution $\paramlim$,
\begin{align*}
\fullcontract{\star,h} = \left( \Id - \step \hf{\paramlim} \right)^h
\eqsp,
\quad
\fullcontract{\star}
= 
(\Id - \step \nbarA[c])^\nlupdates
\eqsp.
\end{align*}
Crucially, in the homogeneous setting, all agents have the same local matrices. Note that this will not be the case anymore in the next section, where agents will be heterogeneous.
We now prove \Cref{thm:exp-homogeneous}, that we restate here for readability.
\biasfedavghomogeneous*
\begin{proof}
\textbf{Expansion of Local Updates (Homogeneous Case).}
We start by studying the local iterates of the algorithm, when starting from a point $\param$ drawn from the local distribution of \FedAvg.
Using a second-order Taylor expansion of the gradient of $\gfww$ at $\paramlim$, we have
\begin{align*}
& \gf{\sdfedavgop[c]{\step,h}(\paramw; \randState[c][1:h])} 
\\ & \quad =
\gf{\paramlim}
+ \hf{\paramlim} (\sdfedavgop[c]{\step,h}(\paramw; \randState[c][1:h]) - \paramlim)
+ \frac{1}{2} \hhf{\paramlim} (\sdfedavgop[c]{\step,h}(\paramw; \randState[c][1:h]) - \paramlim)^{\otimes 2}
+ \reste[c]{3,h}{\sdfedavgop[c]{\step,h}(\paramw; \randState[c][1:h])}
\\
& \quad =
\hf{\paramlim} (\sdfedavgop[c]{\step,h}(\paramw; \randState[c][1:h]) - \paramlim)
+ \frac{1}{2} \hhf{\paramlim} (\sdfedavgop[c]{\step,h}(\paramw; \randState[c][1:h]) - \paramlim)^{\otimes 2}
+ \reste[c]{3,h}{\sdfedavgop[c]{\step,h}(\paramw; \randState[c][1:h])}
\eqsp,
\end{align*}
where we used $\gnf[c]{\paramlim} = 0$ due to homogeneity, and $\restew[c]{3,h}$ is a function that satisfies
\begin{align*}
\sup_{\param \in \rset^d} 
{\norm{ \reste[c]{3,h}{\param}}} / {\norm{ \param - \paramlim }^3}
< + \infty
\eqsp.
\end{align*}
We stress here that, although the local functions are all the same, the noise variables drawn by each agent are different from each other. 
Consequently, local iterates are different from each other.

We can use the above expression to expand \FedAvg's recursion as
\begin{align*}
& \sdfedavgop[c]{\step,h+1}(\paramw; \randState[c][1:h+1]) - \paramlim
\\
& \quad =
\sdfedavgop[c]{\step,h}(\paramw; \randState[c][1:h]) - \paramlim - \step \gf{\sdfedavgop[c]{\step,h}(\paramw; \randState[c][1:h])} - \step \updatefuncnoise[c]{\randState[c][h+1]}(\sdfedavgop[c]{\step,h}(\paramw; \randState[c][1:h]))
\\
& \quad =
\sdfedavgop[c]{\step,h}(\paramw; \randState[c][1:h]) - \paramlim
\\
& \qquad - \step \left(
\hf{\paramlim} (\sdfedavgop[c]{\step,h}(\paramw; \randState[c][1:h]) - \paramlim)
+ \frac{1}{2} \hhf{\paramlim} (\sdfedavgop[c]{\step,h}(\paramw; \randState[c][1:h]) - \paramlim)^{\otimes 2}
+ \reste[c]{3,h}{\sdfedavgop[c]{\step,h}(\paramw; \randState[c][1:h])}
\right)
\\
& \qquad - \step \updatefuncnoise[c]{\randState[c][h+1]}(\sdfedavgop[c]{\step,h}(\paramw; \randState[c][1:h]))
\\
& \quad =
\left( \Id - \step \hf{\paramlim} \right) (\sdfedavgop[c]{\step,h}(\paramw; \randState[c][1:h]) - \paramlim)
\\
& \qquad 
- \frac{\step}{2} \hhf{\paramlim} (\sdfedavgop[c]{\step,h}(\paramw; \randState[c][1:h]) - \paramlim)^{\otimes 2}
- \step \reste[c]{3,h}{\sdfedavgop[c]{\step,h}(\paramw; \randState[c][1:h])}
- \step \updatefuncnoise[c]{\randState[c][h+1]}(\sdfedavgop[c]{\step,h}(\paramw; \randState[c][1:h]))
\eqsp.
\end{align*}
Unrolling this recursion, we obtain
\begin{align*}
& \sdfedavgop[c]{\step,\nlupdates}(\paramw; \randState[c][1:\nlupdates]) - \paramlim
 =
\contract{\star,\nlupdates} (\param - \paramlim)
\\
& \quad
- \step \sum_{h=0}^{\nlupdates-1} \fullcontract{\star,\nlupdates-h-1}
\left( 
\frac{1}{2} \hhf{\paramlim} (\sdfedavgop[c]{\step,h}(\paramw; \randState[c][1:h]) - \paramlim)^{\otimes 2}
+ \reste[c]{3,h}{\sdfedavgop[c]{\step,h}(\paramw; \randState[c][1:h])}
+ \updatefuncnoise[c]{\randState[c][h+1]}(\sdfedavgop[c]{\step,h}(\paramw; \randState[c][1:h]))
\right)
\eqsp. \nonumber
\end{align*}

\textbf{Expansion of $\PE\left[ (\sdfedavgop[c]{\step,h}(\paramw; \randState[c][1:h]) - \paramlim)^{\otimes 2} \right]$ (Homogeneous Case).}
We start with the expression
\begin{align*}
\sdfedavgop[c]{\step,h}(\paramw; \randState[c][1:h]) - \paramlim
& =
\param - \paramlim
- \step \sum_{\ell=0}^{h-1} \gnf[c]{\sdfedavgop[c]{\step,\ell}(\paramw; \randState[c][1:\ell])} + \updatefuncnoise[c]{\randState[c][\ell+1]}(\sdfedavgop[c]{\step,\ell}(\paramw; \randState[c][1:\ell]))
\eqsp.
\end{align*}
We use second-order Taylor expansion of the gradient to obtain
\begin{align*}
\sdfedavgop[c]{\step,h}(\paramw; \randState[c][1:h]) \!-\! \paramlim
& =
\param \!-\! \paramlim
- \step \sum_{\ell=0}^{h-1} 
\hnf[c]{\paramlim} (\sdfedavgop[c]{\step,\ell}(\paramw; \randState[c][1:\ell]) \!-\! \paramlim)
+ \reste[c]{2,h}{\sdfedavgop[c]{\step,\ell}(\paramw; \randState[c][1:\ell])}
+ \updatefuncnoise[c]{\randState[c][\ell+1]}(\sdfedavgop[c]{\step,\ell}(\paramw; \randState[c][1:\ell]))
\eqsp,
\end{align*}
where $\restew[c]{2,h}$ is such that $\sup_{\varparam \in \rset^d} \norm{ \reste[c]{2,h}{\varparam} } / \norm{ \varparam - \paramlim }^2 < + \infty$.
Expanding the square of this equation, and taking the expectation, we get
\begin{align*}
& \int \PE \left( \sdfedavgop[c]{\step,h}(\paramw; \randState[c][1:h]) - \paramlim \right)^{\otimes 2} \statdist{\step,\nlupdates} (\rmd \paramw)
=
\int \left( \param - \paramlim \right)^{\otimes 2} \statdist{\step,\nlupdates} (\rmd \paramw)
\\ \nonumber
& \quad 
- \step
\int  
\left( \param - \paramlim \right) \otimes \left( \sum_{\ell=0}^{h-1} 
\hnf[c]{\paramlim} (\PE \sdfedavgop[c]{\step,\ell}(\paramw; \randState[c][1:\ell]) - \paramlim)
+ \PE \reste[c]{2,\ell}{\sdfedavgop[c]{\step,\ell}(\paramw; \randState[c][1:\ell])}
\right) \statdist{\step,\nlupdates} (\rmd \paramw)
\\ \nonumber
& \quad 
- \step \int
\left( \sum_{\ell=0}^{h-1} 
\hnf[c]{\paramlim} ( \PE \sdfedavgop[c]{\step,\ell}(\paramw; \randState[c][1:\ell]) - \paramlim)
+ \PE \reste[c]{2,\ell}{\sdfedavgop[c]{\step,\ell}(\paramw; \randState[c][1:\ell])}
\right) 
\otimes 
(\param - \paramlim)
\statdist{\step,\nlupdates} (\rmd \paramw)
\\ \nonumber
& \quad 
+ \step^2 \int \PE \left( \sum_{\ell=0}^{h-1} 
\hnf[c]{\paramlim} (\sdfedavgop[c]{\step,\ell}(\paramw; \randState[c][1:\ell]) - \paramlim)
+ \reste[c]{2,\ell}{\sdfedavgop[c]{\step,\ell}(\paramw; \randState[c][1:\ell])}
+ \updatefuncnoise[c]{\randState[c][\ell+1]}(\sdfedavgop[c]{\step,\ell}(\paramw; \randState[c][1:\ell]))
\right)^{\otimes 2}
\statdist{\step,\nlupdates} (\rmd \paramw)
\eqsp.
\end{align*}
From this expansion, Hölder inequality, the definition of $\restew[c]{2,\ell}$, the fact that $\step \nlupdates = O(1)$, \Cref{assum:smooth-var}, \Cref{lem:bound-moments-homogeneous}, and the fact that the $\randState[c][1:\nlupdates]$ are independent from an agent to another, we obtain
\begin{align}    
\label{eq:expansion-paramh}
\int \PE \left( \sdfedavgop[c]{\step,h}(\paramw; \randState[c][1:h]) - \paramlim \right)^{\otimes 2} 
\statdist{\step,\nlupdates} (\rmd \paramw)
& =
\int \left( \param - \paramlim \right)^{\otimes 2} 
\statdist{\step,\nlupdates} (\rmd \paramw)
+ O( \step^2 h )
\eqsp.
\end{align}

\textbf{Expression of the Global Update (Homogeneous Case).}
After averaging the expression obtained for the local updates, we get an expression of the global update,
\begin{align*}
& \sdfedavgop{\step,\nlupdates}(\paramw; \randState[1:\nagent][1:\nlupdates]) - \paramlim
=
\fullcontract{\star,\nlupdates} (\param - \paramlim)
\\
& \quad
- \frac{\step}{\nagent} 
\sum_{c=1}^\nagent 
\sum_{h=0}^{\nlupdates-1} 
\fullcontract{\star,\nlupdates-h-1}
\left( 
\frac{1}{2} \hhf{\paramlim} (\sdfedavgop[c]{\step,h}(\paramw; \randState[c][1:h] - \paramlim)^{\otimes 2}
+ \reste[c]{3,h}{\sdfedavgop[c]{\step,h}(\paramw; \randState[c][1:h]}
+ \updatefuncnoise[c]{\randState[c][h+1]}(\sdfedavgop[c]{\step,h}(\paramw; \randState[c][1:h])
\right)
\eqsp.
\end{align*}
Integrating over $\statdist{\step,\nlupdates}$ and taking the expectation, we obtain
\begin{align*}
& \statdistlim{\step, \nlupdates} - \paramlim
=
\fullcontract{\star,\nlupdates} (\statdistlim{\step, \nlupdates} - \paramlim)
\\
& \quad
- \frac{\step}{\nagent} 
\sum_{c=1}^\nagent 
\sum_{h=0}^{\nlupdates-1} 
\fullcontract{\star,\nlupdates-h-1}
\int \left\{
\frac{1}{2} \hhf{\paramlim} \PE(\sdfedavgop[c]{\step,h}(\paramw; \randState[c][1:h]) - \paramlim)^{\otimes 2}
+ \PE\reste[c]{3,h}{\sdfedavgop[c]{\step,h}(\paramw; \randState[c][1:h]}
\right\} \statdist{\step, \nlupdates}(\rmd \paramw)
\eqsp.
\end{align*}
Using the expression \eqref{eq:expansion-paramh}, Hölder inequality, \Cref{lem:bound-moments-homogeneous}, and the definition of $\restew[c]{3,h}$, we can simplify this expression as
\begin{align*}
(\Id - \fullcontract{\star,\nlupdates}) \left( \statdistlim{\step, \nlupdates} - \paramlim \right)
& =
- \frac{\step}{2}
\sum_{h=0}^{\nlupdates-1} 
\fullcontract{\star,\nlupdates-h-1}
\hhf{\paramlim} 
\int
(\paramw - \paramlim)^{\otimes 2}
\statdist{\step,\nlupdates}(\rmd \paramw)
+ O(\step^2 h)
+ O(\step^{3/2})
\eqsp,
\end{align*}
To give a simpler expression, we remark that  \Cref{lem:product_coupling_lemma} gives the following equality
\begin{align*}
- \frac{\step}{2} \sum_{h=0}^{\nlupdates-1} \fullcontract{\star,\nlupdates-h-1}
 & =
- \frac{1}{2} \left( \Id - \fullcontract{\star,\nlupdates} \right)
\hf{\paramlim}^{-1}
\eqsp.
\end{align*}
Therefore, starting from the previous equation, reorganizing the terms and using this equality, we obtain
\begin{align*}
(\Id \!-\! \fullcontract{\star, \nlupdates}) \left( \statdistlim{\step, \nlupdates} - \paramlim \right)
& =
- \frac{1}{2} (\Id \!-\! \fullcontract{\star, \nlupdates})  
\left\{
\hf{\paramlim}^{-1}
\hhf{\paramlim}  \int \!\!\left( \param - \paramlim\right)^{\otimes 2} \statdist{\step, \nlupdates} (\rmd \param)
+ O(\step^2 h)
+ O(\step^{3/2})
\right\}
\eqsp.
\end{align*}
Multiplying by $(\Id - \fullcontract{\star, \nlupdates})^{-1}$, we obtain
\begin{align}
\label{eq:expansion-with-statvar}
\statdistlim{\step, \nlupdates} - \paramlim
& =
- \frac{1}{2} \hf{\paramlim}^{-1}
\hhf{\paramlim} \int \left( \param - \paramlim\right)^{\otimes 2} \statdist{\step, \nlupdates} (\rmd \param)
+ O(\step^2 \nlupdates)
+ O(\step^{3/2})
\eqsp.
\end{align}

\textbf{Bound the Variance (Homogeneous Case).}
To bound $ \int \left( \param - \paramlim\right)^{\otimes 2} \statdist{\step, \nlupdates} (\rmd \param)$, we proceed as above but with one less term in the expansion, and study the square.
We get
\begin{align*}
& \sdfedavgop[c]{\step,h+1}(\paramw; \randState[c][1:h+1]) - \paramlim
\\
& \quad =
\sdfedavgop[c]{\step,h}(\paramw; \randState[c][1:h]) - \paramlim
- \step \left(
\hf{\paramlim} (\sdfedavgop[c]{\step,h}(\paramw; \randState[c][1:h]) - \paramlim)
+ \reste[c]{2}{\sdfedavgop[c]{\step,h}(\paramw; \randState[c][1:h])}
\right)
- \step \updatefuncnoise[c]{\randState[c][h+1]}(\sdfedavgop[c]{\step,h}(\paramw; \randState[c][1:h]))
\\
& \quad =
\left( \Id - \step \hf{\paramlim} \right) (\sdfedavgop[c]{\step,h}(\paramw; \randState[c][1:h]) - \paramlim)
- \step \reste[c]{2,h}{\sdfedavgop[c]{\step,h}(\paramw; \randState[c][1:h])}
- \step \updatefuncnoise[c]{\randState[c][h+1]}(\sdfedavgop[c]{\step,h}(\paramw; \randState[c][1:h]))
\eqsp.
\end{align*}
Unrolling this recursion and averaging over all agents, we get
\begin{align*}
\sdfedavgop{\step,\nlupdates}(\paramw; \randState[1:\nagent][1:\nlupdates]) - \paramlim
& =
\fullcontract{\star, \nlupdates} (\param - \paramlim)
- \frac{\step}{\nagent} \sum_{c=1}^\nagent \sum_{h=0}^{\nlupdates-1} \fullcontract{\star, \nlupdates    - h- 1} \left\{ \reste[c]{2,h}{\sdfedavgop[c]{\step,h}(\paramw; \randState[c][1:h])} + \updatefuncnoise[c]{\randState[c][h+1]}(\sdfedavgop[c]{\step,h}(\paramw; \randState[c][1:h])) \right\} 
\eqsp.
\end{align*}
Taking the second order moment of this equation, and using the fact that $\sdfedavgop[c]{\step,h+1}(\paramw; \randState[c][1:h+1])$ follows the same distribution as $\param$, we obtain
\begin{align*}
& \int \left( \param - \paramlim \right)^{\otimes 2}
\statdist{\step, \nlupdates} (\rmd \param)
\\
& \quad =
\int \left(
\fullcontract{\star, \nlupdates} (\param - \paramlim)
- \frac{\step}{\nagent} \sum_{c=1}^\nagent  \sum_{h=0}^{\nlupdates-1} \fullcontract{\star, \nlupdates    - h- 1} \left\{ \reste[c]{2,h}{\sdfedavgop[c]{\step,h}(\paramw; \randState[c][1:h])} + \updatefuncnoise[c]{\randState[c][h+1]}(\sdfedavgop[c]{\step,h}(\paramw; \randState[c][1:h])) \right\} 
\right)^{\otimes 2}
\statdist{\step, \nlupdates} (\rmd \param)
\\
& \quad =
\int \left(
\fullcontract{\star, \nlupdates} (\param - \paramlim)
\right)^{\otimes 2} 
\statdist{\step, \nlupdates} (\rmd \param)
\\
& \qquad -
\frac{\step}{\nagent} \sum_{c=1}^\nagent  
\int 
\left(
\fullcontract{\star, \nlupdates} (\param - \paramlim)
\right)
\otimes
\left(\sum_{h=0}^{\nlupdates-1} \fullcontract{\star, \nlupdates    - h- 1} \left\{ \reste[c]{2,h}{\sdfedavgop[c]{\step,h}(\paramw; \randState[c][1:h])} + \updatefuncnoise[c]{\randState[c][h+1]}(\sdfedavgop[c]{\step,h}(\paramw; \randState[c][1:h])) \right\} 
\right) 
\statdist{\step, \nlupdates} (\rmd \param)
\\
& \qquad -
\frac{\step}{\nagent} \sum_{c=1}^\nagent  
\int 
\left(\sum_{h=0}^{\nlupdates-1} \fullcontract{\star, \nlupdates    - h- 1} \left\{ \reste[c]{2,h}{\sdfedavgop[c]{\step,h}(\paramw; \randState[c][1:h])} + \updatefuncnoise[c]{\randState[c][h+1]}(\sdfedavgop[c]{\step,h}(\paramw; \randState[c][1:h])) \right\} 
\right)
\otimes 
\left(
\fullcontract{\star, \nlupdates} (\param - \paramlim)
\right)
\statdist{\step, \nlupdates} (\rmd \param)
\\ 
& \qquad +
\frac{\step^2}{\nagent^2}
\int \left(  \sum_{c=1}^\nagent \sum_{h=0}^{\nlupdates-1} \fullcontract{\star, \nlupdates   - h- 1} \left\{ \reste[c]{2,h}{\sdfedavgop[c]{\step,h}(\paramw; \randState[c][1:h])} + \updatefuncnoise[c]{\randState[c][h+1]}(\sdfedavgop[c]{\step,h}(\paramw; \randState[c][1:h])) \right\} 
\right)^{\otimes 2} 
\statdist{\step, \nlupdates} (\rmd \param)
\eqsp.
\end{align*}
Which gives, using Hölder inequality, \Cref{lem:bound-moments-homogeneous}, \Cref{assum:smooth-var}, the definition of $\restew[c]{2,h}$, the definition of $\covfunc$, and after taking the expectation,
\begin{align*}
\int  \left( \param - \paramlim \right)^{\otimes 2} 
\statdist{\step, \nlupdates} (\rmd \param)
& =
\fullcontract{\star, \nlupdates}
\int \left( \param - \paramlim \right)^{\otimes 2} 
\statdist{\step, \nlupdates} (\rmd \param)
\fullcontract{\star, \nlupdates} 
+
\frac{\step^2}{\nagent} \sum_{h=0}^{\nlupdates-1} 
\PE\covfunc(\sdfedavgop[c]{\step,h}(\paramw; \randState[c][1:h]))
+ O(\step^{5/2} \nlupdates)
\eqsp.
\end{align*}
Now, using \Cref{assum:smooth-var} and \Cref{lem:bound-moments-homogeneous}, we have $\PE\covfunc(\sdfedavgop[c]{\step,h}(\paramw; \randState[c][1:h])) = \covfunc(\paramlim) + O(\step)$, which results in the identity
\begin{align*}
\int \left( \param - \paramlim \right)^{\otimes 2} \statdist{\step,\nlupdates}(\rmd \param)
& =
\fullcontract{\star, \nlupdates}
\int \left( \param - \paramlim \right)^{\otimes 2} \statdist{\step,\nlupdates}(\rmd \param)
\fullcontract{\star, \nlupdates} 
+
\frac{\step^2 \nlupdates}{\nagent}
\covfunc(\paramlim)
+ O(\step^{5/2} \nlupdates)
\eqsp.
\end{align*}
We now use the fact that $\fullcontract{\star, \nlupdates} = \Id - \step \nlupdates \hnf{\paramlim} + O(\step^2 \nlupdates^2)$, which allows to rewrite
\begin{align*}
\int \left( \param - \paramlim \right)^{\otimes 2} \statdist{\step,\nlupdates}(\rmd \param)
& =
\left( \Id - \step \nlupdates \hnf{\paramlim} \right) 
\int \left( \param - \paramlim \right)^{\otimes 2} \statdist{\step,\nlupdates}(\rmd \param)
\left( \Id - \step \nlupdates \hnf{\paramlim} \right) 
\\
& \qquad 
+ \frac{\step^2 \nlupdates}{\nagent}
\covfunc(\paramlim)
+ O(\step^{5/2} \nlupdates)
+ O(\step^{3} \nlupdates^2)
\eqsp.
\end{align*}
Simplifying this expression, we obtain
\begin{align*}
\int \left( \param - \paramlim \right)^{\otimes 2} \statdist{\step,\nlupdates}(\rmd \param)
& =
\frac{\step}{\nagent}
\invopcov
\covfunc(\paramlim)
+ O(\step^{3/2})
+ O(\step^{2} \nlupdates)
\eqsp,
\end{align*}
where we recall that %
\begin{align*}
\invopcov = \left( 
\Id \otimes \hf{\paramlim}
+ \hf{\paramlim} \otimes \Id
\right)^{-1}
\eqsp,
\end{align*}
Plugging this expression in \eqref{eq:expansion-with-statvar}
\begin{align*}
\statdistlim{\step, \nlupdates} - \paramlim
& =
- \frac{\step}{2 \nagent} \hf{\paramlim}^{-1}
\hhf{\paramlim} 
\invopcov 
\covfunc(\paramlim)
+ O(\step^2 \nlupdates)
+ O(\step^{3/2})
\eqsp,
\end{align*}
which is the result
\end{proof}

%% file: 2024-AISTATS/src/aistats-app-analysis-fedavg-sto-heterogeneous.tex
When functions are not quadratic nor homogeneous, local iterates are inherently biased.
There are thus two sources of bias: heterogeneity, as in the quadratic case, and "iterate bias", that is due to stochasticity of gradients and the fact that derivatives of order greater than two are non zero.

To study this case, we define the following matrices, for $h = 0$ to $\nlupdates$, that will be central in the analysis
\begin{align*}
\contract[c]{\star,h} = \left( \Id - \step \hnf[c]{\paramlim} \right)^h
\eqsp.
\end{align*}
Note that, contrarily to the homogeneous setting, the $\contract[c]{\star, h}$'s differ from an agent to another. This will result in additional bias due to heterogeneity.
We now prove \Cref{thm:bias-var-heterogeneous}, that we restate here for readability. 
\biasfedavgheterogeneous*
\begin{proof}
\textbf{Expansion of Local Updates (Heterogeneous Case).}
We start by studying the local iterates of the algorithm.
Using a second-order Taylor expansion of the gradient of $\gnfw[c]$ at $\paramlim$, we have
\begin{align*}
\sdfedavgop[c]{\step,h+1}(\paramw; \randState[c][1:h+1])
& =
\gnf[c]{\paramlim}
+ \hnf[c]{\paramlim} (\sdfedavgop[c]{\step,h}(\paramw; \randState[c][1:h])) - \paramlim)
\\
& \quad
+ \frac{1}{2} \hhnf[c]{\paramlim} (\sdfedavgop[c]{\step,h}(\paramw; \randState[c][1:h])) - \paramlim)^{\otimes 2}
+ \reste[c]{3,h}{\sdfedavgop[c]{\step,h}(\paramw; \randState[c][1:h])}
\eqsp,
\end{align*}
where $\restew[c]{3}$ is a function that satisfies $\sup_{\param \in \rset^d} \left\{
\frac{\norm{ \reste[c]{3,h}{\param}}}{\norm{ \param - \paramlim }^3}
\right\}
< + \infty$.
We can use this expression to expand \FedAvg's recursion as
\begin{align*}
& \sdfedavgop[c]{\step,h+1}(\paramw; \randState[c][1:h+1]) - \paramlim
=
\sdfedavgop[c]{\step,h}(\paramw; \randState[c][1:h]) - \paramlim - \step \gnf[c]{\sdfedavgop[c]{\step,h}(\paramw; \randState[c][1:h])} - \step \updatefuncnoise[c]{\randState[c][h+1]}(\sdfedavgop[c]{\step,h}(\paramw; \randState[c][1:h]))
\\
& \quad =
\sdfedavgop[c]{\step,h}(\paramw; \randState[c][1:h]) - \paramlim
- \step \left(
\gnf[c]{\paramlim}
+ \hnf[c]{\paramlim} (\sdfedavgop[c]{\step,h}(\paramw; \randState[c][1:h])) - \paramlim) \right.
\\
& \qquad \qquad \qquad\qquad \quad
\left. + \frac{1}{2} \hhnf[c]{\paramlim} (\sdfedavgop[c]{\step,h}(\paramw; \randState[c][1:h])) - \paramlim)^{\otimes 2}
+ \reste[c]{3,h}{\sdfedavgop[c]{\step,h}(\paramw; \randState[c][1:h])}
\right)
- \step \updatefuncnoise[c]{\randState[c][h+1]}(\sdfedavgop[c]{\step,h}(\paramw; \randState[c][1:h])
\\
& \quad =
\left( \Id - \step \hnf[c]{\paramlim} \right) (\sdfedavgop[c]{\step,h}(\paramw; \randState[c][1:h])) - \paramlim)
- \step \gnf[c]{\paramlim}
\\
& \qquad \qquad \qquad \qquad \quad 
- \frac{\step}{2}\hhnf[c]{\paramlim} (\sdfedavgop[c]{\step,h}(\paramw; \randState[c][1:h])) - \paramlim)^{\otimes 2}
- \step \reste[c]{3,h}{\sdfedavgop[c]{\step,h}(\paramw; \randState[c][1:h])}
- \step \updatefuncnoise[c]{\randState[c][h+1]}(\sdfedavgop[c]{\step,h}(\paramw; \randState[c][1:h]))
\eqsp.
\end{align*}
Unrolling this recursion, we obtain
\begin{align}
\label{eq:het-local-unrol}
& \sdfedavgop[c]{\step,\nlupdates}(\paramw; \randState[c][1:\nlupdates]) - \paramlim
 =
\contract[c]{\star,\nlupdates} (\param - \paramlim)
- \step \sum_{h=0}^{\nlupdates-1} \contract[c]{\star,\nlupdates-h-1}
\left( 
\gnf[c]{\paramlim}
+ \frac{1}{2} \hhnf[c]{\paramlim} (\sdfedavgop[c]{\step,h}(\paramw; \randState[c][1:h])) - \paramlim)^{\otimes 2}
\right.
\\ \nonumber
& \qquad \qquad \qquad \qquad \qquad\qquad \qquad \qquad\qquad \qquad \qquad
\left.
+ \reste[c]{3,h}{\sdfedavgop[c]{\step,h}(\paramw; \randState[c][1:h])}
+ \updatefuncnoise[c]{\randState[c][h+1]}(\sdfedavgop[c]{\step,h}(\paramw; \randState[c][1:h])
\right)
\eqsp.
\end{align}

\textbf{Expansion of Global Updates (Heterogeneous Case).}
We start by summing \eqref{eq:het-local-unrol} over all agents
\begin{align*}
& \frac{1}{\nagent} \sum_{c=1}^\nagent
\paramw_{\nlupdates}^c - \paramlim
=
\fullcontract{\star,\nlupdates} (\param - \paramlim)
- \frac{\step}{\nagent} \sum_{c=1}^\nagent
\sum_{h=0}^{\nlupdates-1} \contract[c]{\star,\nlupdates-h-1}
\left( 
\gnf[c]{\paramlim}
+ \frac{1}{2} \hhnf[c]{\paramlim} \left( \sdfedavgop[c]{\step,h}
(\paramw; \randState[c][1:h]) - \paramlim \right)^{\otimes 2} 
\right.
\\ \nonumber
& \qquad \qquad \qquad \qquad \qquad\qquad \qquad \qquad\qquad \qquad\qquad \qquad\qquad \qquad\qquad \qquad \qquad
\left.
+ \reste[c]{3,h}{\sdfedavgop[c]{\step,h}(\paramw; \randState[c][1:h])}
\right)
\eqsp.
\end{align*}
Similarly to the homogeneous setting, we integrate over $\statdist{\step,\nlupdates}$, take the expectation and use the fact that $\frac{1}{\nagent} \sum_{c=1}^{\nagent} \paramw_\nlupdates^c$ follows the same distribution as $\param$, to obtain
\begin{align}
\label{eq:het-expect-before-simplif}
& (\Id - \fullcontract{\star,\nlupdates}) (\statdistlim{\step, \nlupdates} - \paramlim)
= 
- \frac{\step}{\nagent} 
 \sum_{c=1}^\nagent
\sum_{h=0}^{\nlupdates-1}
\contract[c]{\star,\nlupdates-h-1}
\gnf[c]{\paramlim}
\\ \nonumber
& \quad 
- \frac{\step}{2\nagent} 
 \sum_{c=1}^\nagent
\sum_{h=0}^{\nlupdates-1}
\contract[c]{\star,\nlupdates-h-1}
\hhnf[c]{\paramlim} \int  \left\{ \PE \left( \sdfedavgop[c]{\step,h}(\paramw; \randState[c][1:h]) - \paramlim \right)^{\otimes 2}
+ \PE \reste[c]{3,h}{\sdfedavgop[c]{\step,h}(\paramw; \randState[c][1:h])}
\right\} \statdist{\step, \nlupdates}(\rmd \paramw)
\eqsp.
\end{align}
Now we use \Cref{lem:product_coupling_lemma} to write
$- \step \sum_{h=0}^{\nlupdates-1} \contract[c]{\star,\nlupdates-h-1} = \left( \Id - \contract[c]{\star,\nlupdates} \right)
\hnf[c]{\paramlim}^{-1}$, and plug it in \eqref{eq:het-expect-before-simplif} to obtain
\begin{align}
\label{eq:het-decomp-bias}
& (\Id - \fullcontract{\star,\nlupdates}) \left( \statdistlim{\step, \nlupdates} - \paramlim \right)
=
\frac{1}{\nagent} 
\sum_{c=1}^\nagent
(\Id - \contract[c]{\star,\nlupdates})
\hnf[c]{\paramlim}^{-1}
\gnf[c]{\paramlim}
\\ \nonumber
& 
- \frac{\step}{2\nagent} 
\sum_{c=1}^\nagent
\sum_{h=0}^{\nlupdates-1}
\contract[c]{\star,\nlupdates-h-1}
\hhnf[c]{\paramlim} \int\!\!  \left(\! \PE \left( \sdfedavgop[c]{\step,h}(\paramw; \randState[c][1:h]) - \paramlim \right)^{\otimes 2}
+ \PE \reste[c]{3,h}{\sdfedavgop[c]{\step,h}(\paramw; \randState[c][1:h])}
\!\right) \statdist{\step, \nlupdates}(\rmd \paramw)
~.
\end{align}
Interestingly, \Cref{eq:het-decomp-bias} is composed of two terms.
The first term is due to heterogeneity, and is the same as in the quadratic setting. From \Cref{prop:bias-quadratic-sto}, we thus know that this term is of order $O(\step \nlupdates)$.
The second one reflects the bias of \FedAvg that is due to stochasticity of the gradients.

\textbf{Expansion of $\mathbf{\int  \left( \sdfedavgop[c]{\step,h}(\paramw; \randState[c][1:h]) - \paramlim \right)^{\otimes 2}  \statdist{\step, \nlupdates}(\rmd \paramw)}$ (Heterogeneous Case).}
We start with the following explicit expression of one round of the local updates
\begin{align*}
\sdfedavgop[c]{\step,h}(\paramw; \randState[c][1:h]) - \paramlim
& =
\param - \paramlim
- \step \sum_{\ell=0}^{h-1} \gnf[c]{\sdfedavgop[c]{\step,\ell}(\paramw; \randState[c][1:\ell])} + \updatefuncnoise[c]{\randState[c][\ell+1]}(\sdfedavgop[c]{\step,\ell}(\paramw; \randState[c][1:\ell]))
\eqsp.
\end{align*}
We use the first-order Taylor expansion of the gradient at $\paramlim$ to obtain
\begin{align*}
& \sdfedavgop[c]{\step,h}(\paramw; \randState[c][1:h]) - \paramlim
\\ 
& \quad =
\param - \paramlim
- \step \sum_{\ell=0}^{h-1}
\gnf[c]{\paramlim}
+ \hnf[c]{\paramlim} (\sdfedavgop[c]{\step,\ell}(\paramw; \randState[c][1:\ell]) - \paramlim)
+ \reste[c]{2,\ell}{\sdfedavgop[c]{\step,\ell}(\paramw; \randState[c][1:\ell])}
+ \updatefuncnoise[c]{\randState[c][\ell+1]}(\sdfedavgop[c]{\step,\ell}(\paramw; \randState[c][1:\ell]))
\eqsp,
\end{align*}
where $\restew[c]{2,\ell}: \rset^d \rightarrow \rset^d$ is a function such that $\sup_{\varparam \in \rset^d} \norm{ \reste[c]{2,\ell}(\varparam) } / \norm{ \varparam - \paramlim }^2 < + \infty$.
Expanding the square of this equation, integrating over $\statdist{\step,\nlupdates}$ and taking the expectation, we get
\begin{align*}
& \int \PE \left( \sdfedavgop[c]{\step,h}(\paramw; \randState[c][1:h]) - \paramlim \right)^{\otimes 2} \statdist{\step, \nlupdates}(\rmd \paramw)
=
\int \left( \param - \paramlim \right)^{\otimes 2}  \statdist{\step, \nlupdates}(\rmd \paramw)
\\ \nonumber
& 
- \step
\int
 \left( \param - \paramlim \right) \otimes \left( \sum_{\ell=0}^{h-1} 
\gnf[c]{\paramlim} 
+ \hnf[c]{\paramlim} (\PE\sdfedavgop[c]{\step,\ell}(\paramw; \randState[c][1:\ell]) - \paramlim)
+ \PE\reste[c]{2,\ell}{\sdfedavgop[c]{\step,\ell}(\paramw; \randState[c][1:\ell])}
\right)  \statdist{\step, \nlupdates}(\rmd \paramw)
\\ \nonumber
& 
- \step 
\int
\left( \sum_{\ell=0}^{h-1} 
\gnf[c]{\paramlim} 
+ \hnf[c]{\paramlim} (\PE\sdfedavgop[c]{\step,\ell}(\paramw; \randState[c][1:\ell]) - \paramlim)
+ \PE\reste[c]{2,\ell}{\sdfedavgop[c]{\step,\ell}(\paramw; \randState[c][1:\ell])}
\right) 
\otimes 
(\param - \paramlim)
 \statdist{\step, \nlupdates}(\rmd \paramw)
\\ \nonumber
& 
\!+\! \step^2\! \!\int \!\!\PE\left( \sum_{\ell=0}^{h-1} 
\gnf[c]{\paramlim} 
\!+\! \hnf[c]{\paramlim} (\sdfedavgop[c]{\step,\ell}(\paramw; \randState[c][1:\ell]) - \paramlim)
\!+\! \reste[c]{2,\ell}{\sdfedavgop[c]{\step,\ell}(\paramw; \randState[c][1:\ell])}
\!+\! \updatefuncnoise[c]{\randState[c][\ell+1]}\!\!(\sdfedavgop[c]{\step,\ell}\!(\paramw; \randState[c][1:\ell]))
\!\right)^{\otimes 2} \!\!\!\!\!\! \statdist{\step, \nlupdates}(\rmd \paramw)
\,.
\end{align*}
From this expansion, Hölder inequality, the definition of $\restew[c]{2,\ell}$, \Cref{assum:smooth-var} and \Cref{lem:bound-moments-heterogeneous}, we obtain
\begin{align}    
\label{eq:expansion-paramh-heter}
\int \PE\left( \sdfedavgop[c]{\step,h}(\paramw; \randState[c][1:h]) - \paramlim \right)^{\otimes 2} \statdist{\step, \nlupdates}(\rmd \paramw)
& =
\int \left( \param - \paramlim \right)^{\otimes 2} \statdist{\step, \nlupdates}(\rmd \paramw)
+ O( \step^{3/2} \nlupdates + \step^2 \nlupdates^2 )
\eqsp.
\end{align}

\textbf{Expression of the Global Update (Heterogeneous Case).}
Plugging \eqref{eq:expansion-paramh-heter} in \eqref{eq:het-decomp-bias}, using \Cref{lem:bound-moments-heterogeneous} to bound $\int \reste[c]{3,h}{\sdfedavgop[c]{\step,h}(\paramw; \randState[c][1:h])} \statdist{\step, \nlupdates} (\rmd \paramw) = O(\step^{3/2} h^{3/2})$, and expanding the first term of \eqref{eq:het-decomp-bias} as in the quadratic setting (see \Cref{cor:exp-quadratic-sto}), we now obtain
\begin{align*}
\statdistlim{\step, \nlupdates} - \paramlim
& =
\frac{\step(\nlupdates-1)}{2\nagent} 
\hf{\paramlim}^{-1} 
\sum_{c=1}^\nagent 
(\hnf[c]{\paramlim} - \hf{\paramlim})
\gnf[c]{\paramlim}
+ O(\step^2 \nlupdates^2)
\\ \nonumber
& 
- \frac{\step}{2 \nagent} 
(\Id - \fullcontract{\star,\nlupdates})^{-1} \sum_{c=1}^\nagent
\sum_{h=0}^{\nlupdates-1}
\contract[c]{\star,\nlupdates-h-1}
\hhnf[c]{\paramlim} \int \left( \param - \paramlim \right)^{\otimes 2} \statdist{\step, \nlupdates} (\rmd \paramw)
+ O( \step^{3/2} \nlupdates + \step^2 \nlupdates^2 ) 
\eqsp.
\end{align*}
Use \Cref{lem:product_coupling_lemma}, that is,
$- \step \sum_{h=0}^{\nlupdates-1} \contract[c]{\star,\nlupdates-h-1} = - \left( \Id - \contract[c]{\star,\nlupdates} \right)
\hnf[c]{\paramlim}^{-1}$, again, we obtain
\begin{align}
\label{eq:het-decomp-bias-same-square-fact}
\statdistlim{\step, \nlupdates} - \paramlim
& =
\frac{\step(\nlupdates-1)}{2\nagent} 
\hf{\paramlim}^{-1} 
\sum_{c=1}^\nagent 
(\hnf[c]{\paramlim} - \hf{\paramlim})
\gnf[c]{\paramlim}
\\ \nonumber
& \quad
- \frac{1}{2 \nagent} 
\hf{\paramlim}^{-1}
\hhf{\paramlim}  \int \left( \param - \paramlim \right)^{\otimes 2} \statdist{\step, \nlupdates} (\rmd \paramw)
+ O( \step^{3/2} \nlupdates + \step^2 \nlupdates^2 )
\eqsp.
\end{align}

\textbf{Expansion of the Variance (Heterogeneous Case).}
To bound $ \int \left( \param - \paramlim \right)^{\otimes 2} \statdist{\step, \nlupdates} (\rmd \paramw)$, we proceed as above but with one less term in the expansion, and study the square.
We get
\begin{align*}
& \sdfedavgop[c]{\step,h+1}(\paramw; \randState[c][1:h+1]) - \paramlim
\\
& =
\left( \Id - \step \hnf[c]{\paramlim} \right) (\sdfedavgop[c]{\step,h}(\paramw; \randState[c][1:h])) - \paramlim)
- \step \gnf[c]{\paramlim}
- \step \reste[c]{2,h}{\sdfedavgop[c]{\step,h}(\paramw; \randState[c][1:h])}
- \step \updatefuncnoise[c]{\randState[c][h+1]}(\sdfedavgop[c]{\step,h}(\paramw; \randState[c][1:h])
\eqsp.
\end{align*}
Unrolling this recursion and averaging over all agents, we get
\begin{align*}
\sdfedavgop[c]{\step,\nlupdates}(\paramw; \randState[1:\nagent][1:\nlupdates]) - \paramlim
& =
\fullcontract{\star, \nlupdates} (\param - \paramlim)
\\
& \qquad
- \frac{\step}{\nagent} \sum_{c=1}^\nagent \sum_{h=0}^{\nlupdates-1} \contract[c]{\star, \nlupdates    - h- 1} \left\{ \gnf[c]{\paramlim} + \reste[c]{2,h}{\sdfedavgop[c]{\step,h}(\paramw; \randState[c][1:h])} + \updatefuncnoise[c]{\randState[c][h+1]}(\sdfedavgop[c]{\step,h}(\paramw; \randState[c][1:h])) \right\} 
\eqsp.
\end{align*}
Taking the second order moment of this equation, using the fact that $\frac{1}{\nagent} \sum_{c=1}^\nagent 
\paramw_{\nlupdates}^c$ follows the same distribution as $\param$, and integrating over $\statdist{\step,\nlupdates}$, we obtain
\begin{align*}
& \int \left( \param - \paramlim \right)^{\otimes 2} \statdist{\step, \nlupdates} (\rmd \paramw)
\\
& =\!\!
\int\! \!\left(\!
\fullcontract{\star, \nlupdates} (\param \!-\! \paramlim)
\!-\! \frac{\step}{\nagent} \sum_{c=1}^\nagent  \sum_{h=0}^{\nlupdates-1} \contract[c]{\star, \nlupdates    - h- 1} \left\{
\gnf[c]{\paramlim} 
+ \reste[c]{2,h}{\sdfedavgop[c]{\step,h}(\paramw; \randState[c][1:h])} 
\!+\! \updatefuncnoise[c]{\randState[c][h+1]}(\sdfedavgop[c]{\step,h}(\paramw; \randState[c][1:h])) \right\} 
\right)^{\otimes 2} \!\!\!\!\statdist{\step, \nlupdates} (\rmd \paramw)
\\
& =
\fullcontract{\star, \nlupdates} \int \left( \param - \paramlim
\right)^{\otimes 2} \statdist{\step, \nlupdates} (\rmd \paramw)
\fullcontract{\star, \nlupdates}
\\
& 
- \!\step
\!\int\!\! \left(
\fullcontract{\star, \nlupdates} (\param \!-\! \paramlim)
\right)
\!\otimes\!
\left(
\!\frac{1}{\nagent} \sum_{c=1}^\nagent  \sum_{h=0}^{\nlupdates-1} \contract[c]{\star, \nlupdates    - h- 1} \!\left\{ \gnf[c]{\paramlim} \!+\! \reste[c]{2,h}{\sdfedavgop[c]{\step,h}(\paramw; \randState[c][1:h])}\! +\! \updatefuncnoise[c]{\randState[c][h+1]}\!\!(\sdfedavgop[c]{\step,h}(\paramw; \randState[c][1:h])) \!\right\} 
\!\!\right) \statdist{\step, \nlupdates} (\rmd \paramw)
\\
& 
-\! \step
\!\int\!\!
\left(
\frac{1}{\nagent}\! \sum_{c=1}^\nagent \! \sum_{h=0}^{\nlupdates-1} \contract[c]{\star, \nlupdates    - h- 1} \!\left\{ \gnf[c]{\paramlim} \!+\! \reste[c]{2,h}{\sdfedavgop[c]{\step,h}(\paramw; \randState[c][1:h])} \!+\! \updatefuncnoise[c]{\randState[c][h+1]}\!(\sdfedavgop[c]{\step,h}(\paramw; \randState[c][1:h])) \right\} \!\!
\right)
\!\otimes\! 
\left(
\fullcontract{\star, \nlupdates} (\param \!-\! \paramlim)
\right)
\statdist{\step, \nlupdates} (\rmd \paramw)
\\ 
&
+
\step^2
\int \left( 
\frac{1}{\nagent}  \sum_{c=1}^\nagent \sum_{h=0}^{\nlupdates-1} \contract[c]{\star, \nlupdates - h- 1} \left\{ \gnf[c]{\paramlim} + \reste[c]{2,h}{\sdfedavgop[c]{\step,h}(\paramw; \randState[c][1:h])} + \updatefuncnoise[c]{\randState[c][h+1]}(\sdfedavgop[c]{\step,h}(\paramw; \randState[c][1:h])) \right\} 
\right)^{\otimes 2} \statdist{\step, \nlupdates} (\rmd \paramw)
\eqsp.
\end{align*}
Now, we expand $\contract[c]{\star, \nlupdates - h- 1}$ and use the fact that $\frac{1}{\nagent} \sum_{c=1}^\nagent \gnf[c]{\paramlim} = 0$, which gives
\begin{align*}
\frac{1}{\nagent}  \sum_{c=1}^\nagent \sum_{h=0}^{\nlupdates-1} \contract[c]{\star, \nlupdates - h- 1} \gnf[c]{\paramlim}
& =
\frac{1}{\nagent}  \sum_{c=1}^\nagent \sum_{h=0}^{\nlupdates-1} \gnf[c]{\paramlim}
- \step \nlupdates \hnf[c]{\paramlim} \gnf[c]{\paramlim}
+ O(\step^2 \nlupdates^2)
\\
& =
\frac{1}{\nagent}  \sum_{c=1}^\nagent \sum_{h=0}^{\nlupdates-1} 
- \step \nlupdates \hnf[c]{\paramlim} \gnf[c]{\paramlim}
+ O(\step^2 \nlupdates^2)
\eqsp,
\end{align*}
which, since $\step \nlupdates = O(1)$, implies that 
\begin{align*}
& \frac{1}{\nagent}  \sum_{c=1}^\nagent \sum_{h=0}^{\nlupdates-1} \contract[c]{\star, \nlupdates - h- 1} \gnf[c]{\paramlim}
= O(\step \nlupdates^2)%
\eqsp,
\eqsp
\text{ and }
\eqsp
\left( \frac{1}{\nagent}  \sum_{c=1}^\nagent \sum_{h=0}^{\nlupdates-1} \contract[c]{\star, \nlupdates - h- 1} \gnf[c]{\paramlim} \right)^{\otimes 2} 
= O(\step^2 \nlupdates^4)
\eqsp.
\end{align*}
Combining the expansions above with Hölder inequality, the definition of $\restew[c]{2,\ell}$, \Cref{assum:smooth-var} and \Cref{lem:bound-moments-heterogeneous}, we obtain
\begin{align*}
& \int \left( \param - \paramlim \right)^{\otimes 2}  \statdist{\step, \nlupdates} (\rmd \paramw)
=
\fullcontract{\star, \nlupdates}
\int \left( \param - \paramlim \right)^{\otimes 2}  \statdist{\step, \nlupdates} (\rmd \paramw)
\fullcontract{\star, \nlupdates} 
\\
& \qquad\qquad\qquad
+
\frac{\step^2}{\nagent} \sum_{h=0}^{\nlupdates-1} 
\int \PE\left[ \frac{1}{\nagent} \sum_{c=1}^\nagent \updatefuncnoise[c]{\randState[c][h+1]}(\sdfedavgop[c]{\step,h}(\paramw; \randState[c][1:h]))^{\otimes 2} \right] \statdist{\step, \nlupdates} (\rmd \paramw)
+ O(\step^{3} \nlupdates^3)
+ O(\step^{5/2} \nlupdates^{2})
\\
& =
\fullcontract{\star, \nlupdates}
\int \left( \param - \paramlim \right)^{\otimes 2}  \statdist{\step, \nlupdates} (\rmd \paramw)
\fullcontract{\star, \nlupdates} 
+
\frac{\step^2}{\nagent}\sum_{h=0}^{\nlupdates-1}  \int \covfunc(\sdfedavgop[c]{\step,h}(\paramw; \randState[c][1:h])) \statdist{\step, \nlupdates} (\rmd \paramw)
+ O(\step^{3} \nlupdates^3)
+ O(\step^{5/2} \nlupdates^{2})
\eqsp.
\end{align*}
Now, using \Cref{assum:smooth-var} and \Cref{lem:bound-moments-heterogeneous} we have $ \int \covfunc(\sdfedavgop[c]{\step,h}(\paramw; \randState[c][1:h])) \statdist{\step, \nlupdates} (\rmd \paramw) = \covfunc(\paramlim) + O(\step\nlupdates)$, which results in the identity
\begin{align*}
\int \left( \param - \paramlim \right)^{\otimes 2} \statdist{\step, \nlupdates} (\rmd \paramw)
& =
\fullcontract{\star, \nlupdates}
\int \left( \param - \paramlim \right)^{\otimes 2} \statdist{\step, \nlupdates} (\rmd \paramw)
\fullcontract{\star, \nlupdates} 
+
\frac{\step^2 \nlupdates}{\nagent}
\covfunc(\paramlim)
+ O(\step^{3} \nlupdates^3)
+ O(\step^{5/2} \nlupdates^{2})
\eqsp.
\end{align*}
We now use the fact that $\fullcontract{\star, \nlupdates} = \Id - \step \nlupdates \hf{\paramlim} + O(\step^2 \nlupdates^2)$, which allows to rewrite
\begin{align*}
\int \left( \param - \paramlim \right)^{\otimes 2} \statdist{\step, \nlupdates} (\rmd \paramw)
& =
\left( \Id - \step \nlupdates \hf{\paramlim} \right)
\int \left( \param - \paramlim \right)^{\otimes 2} \statdist{\step, \nlupdates} (\rmd \paramw)
\left( \Id - \step \nlupdates \hf{\paramlim} \right) 
\\
& \qquad 
+ \frac{\step^2 \nlupdates}{\nagent}
\covfunc(\paramlim)
+ O(\step^{3} \nlupdates^3)
+ O(\step^{5/2} \nlupdates^{2})
\eqsp.
\end{align*}
Developing this expression and using \Cref{lem:bound-moments-heterogeneous}, we get
\begin{align*}
\int \left( \param - \paramlim \right)^{\otimes 2} \statdist{\step, \nlupdates} (\rmd \paramw)
& =
\int \left( \param - \paramlim \right)^{\otimes 2} \statdist{\step, \nlupdates} (\rmd \paramw)
\\
& \qquad 
- \step \nlupdates \hf{\paramlim} \int \left( \param - \paramlim \right)^{\otimes 2} \statdist{\step, \nlupdates} (\rmd \paramw)
- \step \nlupdates \int \left( \param - \paramlim \right)^{\otimes 2} \statdist{\step, \nlupdates} (\rmd \paramw) \hf{\paramlim} 
\\
& \qquad 
+ \frac{\step^2 \nlupdates}{\nagent}
\covfunc(\paramlim)
+ O(\step^{3} \nlupdates^3)
+ O(\step^{5/2} \nlupdates^{2})
\eqsp.
\end{align*}
Simplifying this expression, we obtain
\begin{align*}
\int \left( \param - \paramlim \right)^{\otimes 2} \statdist{\step, \nlupdates} (\rmd \paramw)
& =
\frac{\step}{\nagent}
\invopcov
\covfunc(\paramlim)
+ O(\step^{2} \nlupdates^2)
+ O(\step^{3/2} \nlupdates)
\eqsp,
\end{align*}
where we recall that
\begin{align*}
\invopcov = \left( 
\Id \otimes \hf{\paramlim}
+ \hf{\paramlim} \otimes \Id
\right)^{-1}
\eqsp,
\end{align*}
Plugging this expression in \eqref{eq:het-decomp-bias-same-square-fact}, we obtain
\begin{align*}
\statdistlim{\step, \nlupdates} - \paramlim
& =
\frac{\step(\nlupdates-1)}{2\nagent} 
\hf{\paramlim}^{-1} 
\sum_{c=1}^\nagent 
(\hnf[c]{\paramlim} - \hf{\paramlim})
\gnf[c]{\paramlim}
\\ \nonumber
& \quad
- \frac{\step}{2 \nagent} 
\hf{\paramlim}^{-1}
\hhf{\paramlim}  
\invopcov \covfunc(\paramlim)
+ O(\step^{2} \nlupdates^2)
+ O(\step^{3/2} \nlupdates)
\eqsp,
\end{align*}
which is the result of the theorem.
\end{proof}

%% file: 2024-AISTATS/src/aistats-app-richardson-romberg.tex
\subsection{Convergence of Richardson-Romberg Iterates -- Proof of \Cref{thm:RR-non-avg}}
\label{sec:proof-cv-rr}

\richardsonrombergconvergenceiterates*

\begin{proof}

\textbf{Bound on the bias.}
Recall that the iterates of \FedAvg with Richardson-Romberg extrapolation are
\begin{align*}
\vartheta_t^{(\step,H)}
& =
2 \param[t]^{(\step,\nlupdates)} - \param[t]^{(2\step,\nlupdates)}
\eqsp,
\end{align*}
where $ \param[t]^{(\step)}$ are \FedAvg's iterates with step size $\step$ and $\param[t]^{(2 \step)}$ are \FedAvg's iterates with step $2 \step$.
By \Cref{thm:bias-var-heterogeneous}, we have that
\begin{align}
\label{eq:proof-bias-rr-one}
\statdistlim{\step, \nlupdates} \!-\! \paramlim
& = \frac{\step}{2 \nagent}
\biasSto + \frac{\step(\nlupdates-1)}{2}  
\biasHetero
+ O(\step^{2} \nlupdates^2 + \step^{{3}/{2}} \nlupdates)
\eqsp,
\\
\label{eq:proof-bias-rr-two}
\statdistlim{2\step, \nlupdates} \!-\! \paramlim
& = \frac{2\step}{2 \nagent}
\biasSto + \frac{2\step(\nlupdates-1)}{2}  
\biasHetero
+ O(\step^{2} \nlupdates^2 + \step^{{3}/{2}} \nlupdates)
\eqsp.
\end{align}
Multiplying \eqref{eq:proof-bias-rr-one} by two and subtracting \eqref{eq:proof-bias-rr-two}, we obtain the first part of the theorem.

\textbf{Communication complexity.}
To bound the number of required communications, we decompose the error as
\begin{align*}
\vartheta_t^{(\step,H)} - \paramlim 
& = 
2 \param[t]^{(\step,\nlupdates)} - \param[t]^{(2 \step, \nlupdates)} - \paramlim
\\
& = 
2 \param[t]^{(\step,\nlupdates)} - 2 \statdistlim{\step, \nlupdates}
- \param[t]^{(2 \step, \nlupdates)} + \statdistlim{2 \step, \nlupdates}
- \paramlim + 2 \statdistlim{\step, \nlupdates} - \statdistlim{2 \step, \nlupdates}
\\
& = 
2 \param[t]^{(\step,\nlupdates)} - 2 \statdistlim{\step, \nlupdates}
- \param[t]^{(2 \step, \nlupdates)} + \statdistlim{2 \step, \nlupdates}
- \paramlim + \statdistlimv{\step,H}
\eqsp.
\end{align*}
Using Jensen's inequality, we thus obtain the following bound on the squared error,
\begin{align}
\norm{ \vartheta_t^{(\step,H)} - \paramlim }^2 
& \le
3 \norm{ 2 \param[t]^{(\step)} - 2 \statdistlim{\step, \nlupdates} }^2
+ 3 \norm{ \param[t]^{(2\step)} - \statdistlim{2\step, \nlupdates} }^2
+ 3 \norm{\statdistlimv{\step,\nlupdates} - \paramlim}^2
\eqsp.
\end{align}
By \Cref{prop:conv-neighborhood-stat-dist}, we can bound the first two terms as
\begin{align}
\PE\left[ \norm{ \param[t]^{(2 \step)} - \statdistlim{2 \step, \nlupdates} }^2 \right]
& \le
(1 - 2 \step \strcvx)^{\nlupdates t} \Big\{ 4 \norm{ \param[0] - \paramlim }^2
+ \frac{24 \nlupdates^2 \step^2 \explip^2 \heterboundgrad^2}{\strcvx^2} 
+ \frac{32 \step}{\strcvx} \Msmoothcstvar^{2} \Big\}
    + \frac{8 \step}{\strcvx} \Msmoothcstvarbis^2
\eqsp,
\\
\PE\left[ \norm{ 2 \param[t]^{(\step)} - 2 \statdistlim{\step, \nlupdates} }^2 \right]
& \le
(1 - \step \strcvx)^{\nlupdates t} \Big\{ 16 \norm{ \param[0] - \paramlim }^2
+ \frac{96 \nlupdates^2 \step^2 \explip^2 \heterboundgrad^2}{\strcvx^2} 
+ \frac{128 \step}{\strcvx} \Msmoothcstvar^{2} \Big\}
    + \frac{32 \step}{\strcvx} \Msmoothcstvarbis^2
\eqsp.
\end{align}
By \Cref{thm:bias-var-heterogeneous}, we have
\begin{align}
\norm{\statdistlimv{\step,\nlupdates} - \paramlim}^2
= O(\step^4 \nlupdates^4 \!+\! \step^{3} \nlupdates^2)
\eqsp.
\end{align}
Thus, the iterates of \FedAvg with Richardson-Romberg extrapolation (without averaging) satisfy
\begin{align}
\PE[ \norm{ \vartheta_t^{(\step,H)} - \paramlim }^2 ]
=
O\left( 
(1 - \step \strcvx)^{\nlupdates t}  \Big\{ \norm{ \param[0]  \!-\! \paramlim }^2 + \frac{\nlupdates^2 \step^2 \explip^2 \heterboundgrad^2}{\strcvx^2} 
+ \frac{\step}{\strcvx} \Msmoothcstvar^{2} \Big\} 
+ \step^4 \nlupdates^4 + \step^{3} \nlupdates^2
+ \frac{\step}{\mu} \Msmoothcstvarbis^{1/2}
\right)
\eqsp.
\end{align}
To obtain $\PE[\norm{\param[t] - \paramlim}^2] = O(\epsilon^2)$, we require
\begin{align}
\step = O(\epsilon^2)
\eqsp,\qquad
\step^4 \nlupdates^4 = O(\epsilon^2)
\eqsp,\qquad
\step^3 \nlupdates^2 = O(\epsilon^2)
\eqsp,\qquad
T = O\left(\frac{1}{\step \strcvx \nlupdates} \log\left(\frac{1}{\epsilon}\right)\right)
\eqsp.
\end{align}
Thus, we require $\nlupdates^4 = O(1/\epsilon^{6})$ and $\nlupdates^3 = O(1/\epsilon^4)$, which necessitates $\nlupdates = O(1/\epsilon^{4/3})$, which yields $\step \nlupdates = O(\epsilon^{2/3})$.
As a result, the required number of communication to reach mean squared error of order $O(\epsilon^2)$ is
\begin{align}
    T = O\left(\frac{1}{\epsilon^{2/3}} \log\left(\frac{1}{\epsilon}\right)\right)
    \eqsp,
\end{align}
which gives the second part of the result.
\end{proof}

\subsection{Averaged Richardson-Romberg Iterates -- Proof of \Cref{theo:RR-bias}}
\label{sec:proof-cv-averaged-rr}
Finally, we prove the following theorem.
\richardsonrombergconvergence*
\begin{proof}
The only statement to show is that under our assumptions, the iterates $\smash{\{ \bar{\theta}_T^{(\step,\nlupdates)} \}_{T \ge 1}}$ defined as
\begin{align*}
 \bar{\theta}_T^{(\step,\nlupdates)} = \frac{1}{T} \sum_{t=0}^{T-1} \theta_t^{(\step,H)} \eqsp,
\end{align*}
converge in $\mrl^2$ to $\statdistlim{\step, \nlupdates}$.
This is a consequence of \cite[Theorem 8]{durmus2024probability} whose assumptions are satisfied by \Cref{lem:bound-moments-homogeneous} and \Cref{prop:conv-stat-dist}.

Then, the identity $\statdistlimv{\step,H} - \paramlim
=
O(\step^2 \nlupdates^2 + \step^{3/2} \nlupdates)$ follows from \Cref{thm:RR-non-avg}.
\end{proof}

%% file: 2024-AISTATS/src/aistats-technical-lemmas.tex
\begin{lemma}
\label{lem:product_coupling_lemma}
For any matrix-valued sequences $(M_k)_{k \in \nset}$, $(M'_k)_{k\in \nset}$ and for any $K \in \nset$, it holds that:
\begin{equation*}
  \prod_{k=1}^K M_k - \prod_{k=1}^K M'_k
  = \sum_{k=1}^K \left\{\prod_{\ell=1}^{k-1} M_\ell \right\} \big(M_k - M'_k\big) \left\{\prod_{\ell=k+1}^M M'_\ell \right\} \eqsp .
\end{equation*}
\end{lemma}